%% file: main.tex
\documentclass{article}
\usepackage{enumitem}
\usepackage[utf8]{inputenc}
\usepackage{fullpage}
\usepackage[english]{babel}
\usepackage{bbold}
\usepackage{charter}
\usepackage{pifont}
\usepackage{color,soul}
\usepackage{empheq}
\usepackage{authblk}

\newcommand{\T}[1]{\texttt{#1}}

\input{header}

\usepackage{xspace}
\newcommand{\sys}{\texttt{CorgiPile}\xspace}

\title{\bf Stochastic Gradient Descent without Full Data Shuffle\footnote{This technical report is an extension of our SIGMOD 2022 paper titled \textit{In-Database Machine Learning with CorgiPile: Stochastic Gradient Descent without Full Data Shuffle}. \url{https://doi.org/10.1145/3514221.3526150}}}

\author[$\dagger\ddagger$]{Lijie Xu}
\author[$\mathsection$]{Shuang Qiu}
\author[$\dagger$]{Binhang Yuan}
\author[$\dagger$]{Jiawei Jiang}
\author[$\dagger$]{Cedric Renggli}
\author[$\dagger$]{Shaoduo Gan}
\author[$\dagger$]{Kaan Kara}
\author[$\mathparagraph$]{Guoliang Li}
\author[$\sharp$]{Ji Liu}
\author[$\flat$]{Wentao Wu}
\author[$\natural$]{Jieping Ye}
\author[$\dagger$]{Ce Zhang}

\affil[$\dagger$]{ETH Z\"urich \quad$^{\ddagger}$Institute of Software, CAS \quad$^{\mathsection}$University of Chicago} 
\affil[$\mathparagraph$]{Tsinghua University \quad$^{\sharp}$Kwai Inc. \quad$^{\flat}$Microsoft Research \quad$^{\natural}$University of Michigan}

\date{\small $^{\dagger}$\{firstname.lastname\}@inf.ethz.ch  \\
$^{\mathparagraph}$liguoliang@tsinghua.edu.cn
$^{\sharp}$ji.liu.uwisc@gmail.com
$^{\flat}$wentao.wu@microsoft.com
$^{\mathsection \natural}$\{qiush, jpye\}@umich.edu}

\begin{document}
\maketitle
\begin{abstract} 
Stochastic gradient descent (SGD) is
the cornerstone of modern machine learning (ML) systems.
Despite its computational 
efficiency, SGD requires 
random data access that is inherently inefficient when implemented 
in systems that rely on \emph{block-addressable secondary storage} such as 
HDD and SSD, e.g., TensorFlow/PyTorch and in-DB ML systems over large files.
To address this impedance mismatch, 
various data shuffling strategies have been proposed
to balance the convergence rate of SGD (which favors randomness) 
and its I/O performance (which favors sequential access).

In this paper, we first conduct a systematic empirical study on 
existing data shuffling strategies, which reveals that 
all existing strategies have room for improvement---they all suffer in terms of I/O performance \emph{or} convergence rate. 
With this in mind, we propose a simple but novel \emph{hierarchical} data shuffling strategy, \sys.
Compared with existing strategies, \sys \emph{avoids} a full data shuffle while maintaining \emph{comparable convergence rate} of SGD as if a full shuffle were performed. 
We provide a non-trivial theoretical 
analysis of \sys on its convergence behavior. We further integrate \sys into PyTorch by designing new parallel/distributed shuffle operators inside a new \texttt{CorgiPileDataSet} API. We also integrate \sys into PostgreSQL by introducing 
three new \textit{physical} operators with optimizations.
Our experimental results show that \sys
can achieve comparable convergence rate with the full shuffle based SGD for both deep learning and generalized linear models. For deep learning models on ImageNet dataset, \sys is 1.5$\times$ faster than PyTorch with full data shuffle. For in-DB ML with linear models, \sys is 1.6$\times$-12.8$\times$ faster than two state-of-the-art in-DB ML systems, Apache MADlib and Bismarck, on both HDD and SSD.
\end{abstract}

\input{s1-introduction}



\input{s2-empirical-study}


\input{s3-our-approach-design}

\input{s4-implementation}


\input{s5-evaluation}


\input{s6-related-work}


\input{s7-conclusion}

\clearpage
\balance

\bibliographystyle{abbrv}
\bibliography{main}

\clearpage

\appendix

\input{s8-appendix}

\clearpage

\input{proof}

\end{document}

%% file: header.tex
\usepackage{microtype}
\usepackage{graphicx}
\usepackage{subfigure}
\usepackage{booktabs} 

\usepackage{hyperref}
\usepackage{balance}

\usepackage{nicefrac}       
\usepackage{algorithm}
\usepackage{algorithmic}

\usepackage{multirow}
\usepackage{bbm}

\usepackage{amsmath}
\usepackage{amsthm}
\usepackage{amssymb}
\usepackage{epstopdf} 
\usepackage{enumitem}

\usepackage{caption}

\usepackage{xcolor}

\usepackage{listings}
\definecolor{mygreen}{rgb}{0,0.6,0}
\definecolor{mygray}{rgb}{0.5,0.5,0.5}
\definecolor{mymauve}{rgb}{0.58,0,0.82}

\definecolor{codegreen}{rgb}{0,0.6,0}
\definecolor{codegray}{rgb}{0.5,0.5,0.5}
\definecolor{codepurple}{rgb}{0.58,0,0.82}
\definecolor{backcolour}{rgb}{0.95,0.95,0.95}

\newtheorem{theorem}{Theorem}

\newtheorem{lemma}{Lemma}
\newtheorem{example}{Example}

\newtheorem{assumption}{Assumption}
\newtheorem{fact}{Fact}
\allowdisplaybreaks

\def \t {\mathbf{t}}

\def \x {\mathbf{x}}
\def \y {\mathbf{y}}
\def \z {\mathbf{z}}

\def \R {\mathbb{R}}
\def \E {\mathbb{E}}

\def \calB {\mathcal{B}}

\def \calG {\mathcal{G}}

\def \calI {\mathcal{I}}
\def \calJ {\mathcal{J}}

%% file: s1-introduction.tex
\section{Introduction}
\label{sec:intro}
Stochastic gradient descent (SGD) is the cornerstone of modern ML systems. With ever-growing 
data volume, inevitably, SGD algorithms have to 
access data stored in the \textit{secondary storage} instead of accessing the DRAM directly. 
This can happen in two prominent applications: (1) in \textit{deep learning systems} such as TensorFlow~\cite{Tensorflow-OSDI},
one needs to
support out-of-memory access via a specialized 
scanner over large files;
(2) for many \textit{in-database machine learning} (in-DB ML)
scenarios, one has to assume that
the data is stored on the secondary storage, managed 
by the buffer manager~\cite{Buffer-manager}.

\paragraph*{\underline{Deep Learning Systems and In-database Machine Learning}} 

Both deep learning and in-DB ML systems are popular research areas for years~\cite{madlib-paper, Feng:2012:TUA:2213836.2213874,DBLP:conf/sigmod/SchleichOC16, learning-over-join, F-system, linear-algebra, Scale-factorization-ml, DBLP:conf/pods/Khamis0NOS18, scalable-linear-algebra}. The state-of-the-art deep learning systems such as PyTorch and TensorFlow provide users with simple \texttt{Dataset/DataLoader} APIs to load data from secondary storage into memory and further to GPUs, as shown in the following lines of code. The deep learning systems can automatically perform model training in \texttt{train()}, using a number of GPUs.
\begin{verbatim} 
          train_dataset = Dataset(dataset_path, other_args)
          train_loader = DataLoader(train_dataset, shuffle_args, other_args)
          train(train_loader, model, other_args)
\end{verbatim}

For in-DB ML, its major benefit is that users do not need to move the data out of DB to another specialized ML platform, given that data movement is often time-consuming, error-prone, or even impossible (e.g., due to privacy and security compliance concerns).
Instead, users can define their ML training jobs using SQL, e.g., training an SVM model 
with MADlib \cite{madlib, madlib-paper} and Bismarck \cite{Feng:2012:TUA:2213836.2213874} can be done via a single SQL statement:
\begin{verbatim}
                        SELECT svm_train(table_name, parameters).
\end{verbatim}

\paragraph*{\underline{A Fundamental Discrepancy}}  
As identified
by previous work~\cite{Feng:2012:TUA:2213836.2213874,zhang2014dimmwitted,kara2018columnml}, one unique challenge of deep learning and in-DB ML 
is that data can be \textit{clustered} 
while shuffling is not always feasible. For example, the data is clustered by the label, where
data with negative labels might always come before data with positive labels~\cite{Feng:2012:TUA:2213836.2213874}. Another example is that the data is ordered by one of the features. These are common
cases when there is a clustered B-tree index, or the data is naturally
grouped/ordered by, e.g., timestamps.
As SGD requires a random data order 
(shuffle over \emph{all} data) to converge~\cite{Feng:2012:TUA:2213836.2213874,gurbuzbalaban2015random,yun2021can,de2020random,haochen2018random,mishchenko2020random,rajput2020closing,safran2020good},
directly running sequential scans over 
such a clustered dataset can slowdown its
convergence.

Meanwhile, when data are stored on block-addressable secondary storage such as HDD and SSD, it can be incredibly expensive to either shuffle the data \emph{on-the-fly} when running SGD, or shuffle the data \emph{once} with copy
and run SGD over the shuffled copy, due to the amount of \emph{random} I/O's. 
Sometimes, data shuffling might not be applicable in database systems---in-place shuffling might 
have an impact on other indices,
whereas shuffling over a data copy introduces  $2\times$ storage
overhead. 
\textit{How to design efficient SGD
algorithms without requiring 
even a single pass of full data shuffle?}
Understanding this question 
can have a profound impact to 
the system design of both deep 
learning systems and in-DB ML systems.

\begin{figure*}[t!]
\centering
\includegraphics[width=0.6\textwidth]{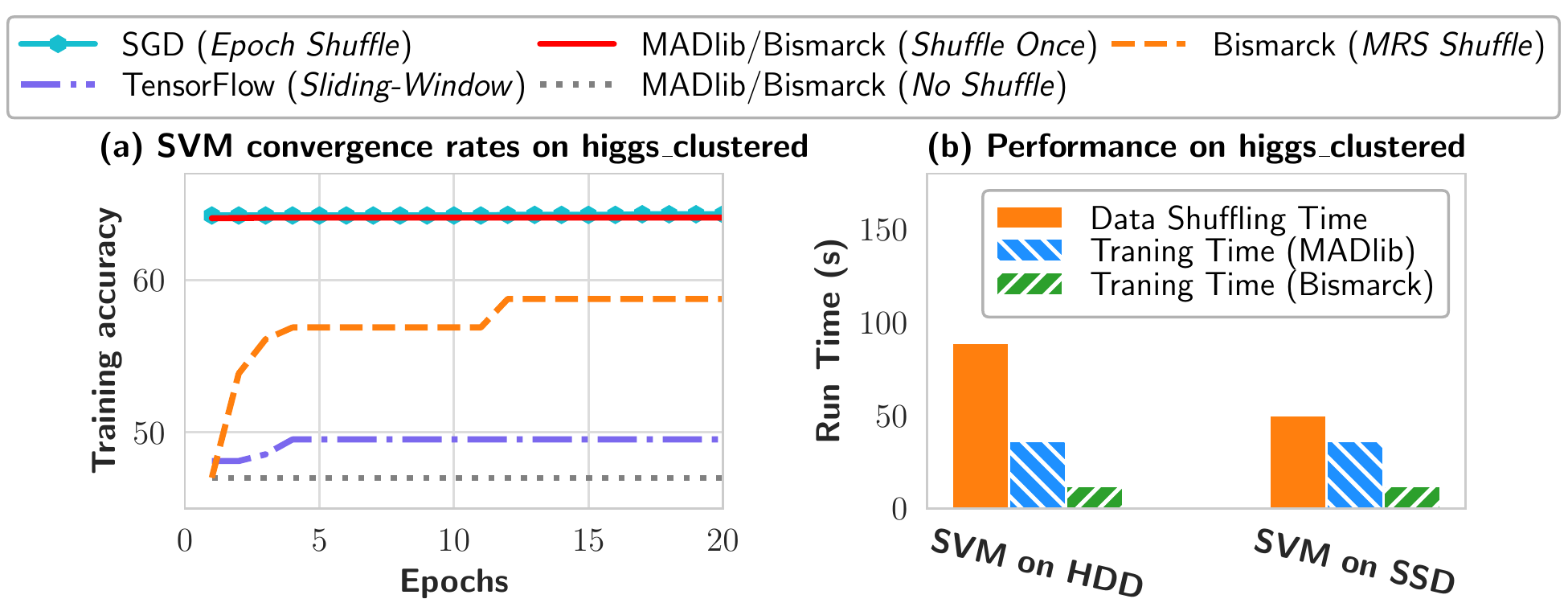}
\caption{The convergence rate and performance of SVM on the clustered \emph{higgs} dataset with different data shuffling strategies. (a) Today's SGD systems
over secondary 
storage, including in-DB ML
solutions (e.g., MADlib and Bismarck) and TensorFlow
file scanner, are sensitive to clustered 
data order. (b) Forcing a full data shuffle before 
training accommodates this clustered data issue, however, introduces large overhead that is often
more expensive than training itself.
}
\label{fig:convergence-of-all-tools}

\end{figure*}

\paragraph*{\underline{Existing Landscape and Challenges}}
Various solutions have been proposed,
in the context of both deep learning and in-DB ML systems.
TensorFlow
provides a shuffling strategy based on 
a \emph{sliding window} over the data~\cite{sliding-window}.
In Bismarck~\cite{Feng:2012:TUA:2213836.2213874},
the authors suggest a 
``multiplexed reservoir sampling'' (MRS) shuffling
strategy, in which two concurrent threads
update the same model---one 
reads data sequentially with reservoir sampling and the other 
reads from a small, shuffled 
in-memory buffer.  
Both significantly improve 
the SGD convergence rate and have been widely adopted 
in practice. 
Despite these efforts, however, they suffer from some shortcomings. As illustrated
in Figure~\ref{fig:convergence-of-all-tools}(a), both 
strategies proposed by 
Bismarck and TensorFlow perform 
suboptimally given a \emph{clustered}
data order. 
Meanwhile, the idea of shuffling data once \emph{before} training, i.e., 
the curve corresponding to 
``MADlib/Bismarck (\emph{Shuffle Once}),''
can accommodate for such convergence problem
but also introduce a significant 
overhead as shown in Figure~\ref{fig:convergence-of-all-tools}(b), which is consistent with 
the observations from previous work~\cite{Feng:2012:TUA:2213836.2213874}.

\paragraph*{\underline{Our Contributions}}
We are inspired by these
previous efforts. In this paper, we ask the following questions:
\begin{quote}
\em Can we design an SGD-style algorithm with efficient shuffling strategy 
that can converge without requiring a full data shuffle? 
Can we provide a rigorous theoretical 
analysis on the convergence behavior of such an algorithm? Can we integrate such an algorithm into both deep learning systems and database systems?
\end{quote}

In this work, we systematically study these
questions and make the following contributions.

\vspace{0.5em}
{\bf C1. An Anatomy and Empirical 
Study of Existing Algorithms.}
We start with a systematic 
evaluation of existing data shuffling strategies for SGD, including (1) \emph{Epoch Shuffle}, which performs a full shuffle before each epoch, (2) \emph{Shuffle Once}, (3) \emph{No Shuffle}, (4) \emph{Sliding-Window Shuffle}, and (5) \emph{MRS Shuffle}.
We evaluate them in the context of using SGD to train generalized linear models and deep learning models, over a variety of datasets.
Our evaluation reveals
 that existing strategies cannot \emph{simultaneously} achieve good hardware efficiency (I/O performance) and statistical efficiency (convergence rate and converged accuracy).
Specifically, \emph{Epoch Shuffle} and \emph{Shuffle Once} achieve the best statistical efficiency,
since the data has been fully shuffled; however, their hardware efficiency is suboptimal given the 
additional shuffle overhead and storage overhead.
In contrast, \emph{No Shuffle} achieves the best hardware efficiency as no data shuffle is required; however, its statistical efficiency suffers as it might converge slowly or even diverge.
The other two strategies, \emph{Sliding-Window Shuffle} and \emph{MRS Shuffle}, 
can be viewed as a compromise between \emph{Shuffle Once} and \emph{No Shuffle}, which trade statistical efficiency for hardware efficiency. Nevertheless, 
both strategies  suffer in terms of statistical 
efficiency (Section~\ref{sec:benchmark}). 

{\bf C2. A Simple, but Novel, Algorithm
with Rigorous Theoretical Analysis.}
Motivated by the limitations of existing strategies, we propose
\sys, a novel
SGD-style algorithm  
based on a \emph{two-level hierarchical} data shuffle strategy.\footnote{Although we give unquestionable love to dogs, the name comes from the shuffling strategy that is a combination of \textit{pile shuffle} and \textit{corgi shuffle}, two commonly used strategies to shuffle a deck of cards.} 
The main idea is to first sample and shuffle the data at a \emph{block level}, and then shuffle data at a \emph{tuple level} within the sampled data blocks, i.e., first sampling data blocks (e.g., a batch of table pages per block in DB), then merging the sampled blocks in a buffer, and finally shuffling the tuples in the buffer for SGD.
While this two-level strategy seems quite simple, it can achieve both good hardware efficiency and statistical efficiency.
Although the hardware efficiency is easy to understand---accessing random data blocks is much more efficient than accessing random tuples, especially when the block size is large, the statistical efficiency requires some non-trivial analysis.
To this end, we further provide a rigorous theoretical study on the convergence behavior.

{\bf C3. Implementation, Optimization, and Deep Integration with PyTorch and PostgreSQL.}
While the benefit of \sys for hardware efficiency is intuitive, its realization requires careful design, implementation, and optimization. 
Unlike previous in-DB ML systems such as MADlib and Bismarck that 
integrate ML algorithms using User-Defined Aggregates (UDAs),
our technique requires a deeper system integration since it needs to directly interact with the buffer manager.
Therefore, we operate at the ``physical level'' and enable in-DB ML inside PostgreSQL~\cite{PostgreSQL}
via three new \textit{physical operators}: \T{BlockShuffle} operator, \T{TupleShuffle} operator, and \T{SGD} operator for our customized SGD implementation\footnote{The code of \sys in PostgreSQL is available at~\url{https://github.com/DS3Lab/CorgiPile-PostgreSQL}.}.
We can then construct an \emph{execution plan} for the SGD computation by chaining these operators together to form a pipeline,  naturally following the built-in Volcano paradigm~\cite{Graefe94} of PostgreSQL. 
We also design a \emph{double-buffering} mechanism to optimize the \T{TupleShuffle} operator. For deep learning systems, we extend \sys to work in a parallel/distributed environment, by enhancing the \T{TupleShuffle} with multiple buffers\footnote{The code of \sys in PyTorch is available at~\url{https://github.com/DS3Lab/CorgiPile-PyTorch}.}.

{\bf C4. Extensive Empirical Evaluations.}
We conduct extensive evaluations to demonstrate the effectiveness of \sys. We first compare \sys with other shuffling strategies in PyTorch using deep learning workloads of image classification and text classification. The results show that \sys achieves similar model accuracy to the best \emph{Shuffle Once} baseline, whereas other data shuffling strategies suffer from lower accuracy. Specifically, for the \texttt{ImageNet} dataset, \sys is 1.5$\times$ faster than \emph{Shuffle Once} to converge with 8 GPUs. 
We then compare the end-to-end performance of our PostgreSQL implementation with two state-of-the-art in-DB ML systems, MADlib and Bismarck.
The results again show that \sys achieves comparable model accuracy to the best \emph{Shuffle Once} baseline,
but is significantly faster since it does not require full data shuffle. Other strategies suffer from lower SGD convergence rate on \emph{clustered} datasets. 
Overall, \sys can achieve 1.6$\times$-12.8$\times$ speedup compared to MADlib and Bismarck over clustered data. Furthermore, for the datasets ordered by features instead of the label, our \sys still achieves comparable accuracy to the \emph{Shuffle Once}, whereas \emph{No Shuffle} suffers from lower accuracy.

\paragraph*{\underline{Overview}} This technical report is organized as follows.
We first review the SGD algorithm and its implementation (Section~\ref{sec:Preliminaries}). We next empirically study the SGD convergence rate using current state-of-the-art data shuffling strategies (Section~\ref{sec:benchmark}). We then propose our \sys strategy as well as a theoretical analysis on its convergence (Section~\ref{sec:design}). We present our implementation of \sys inside PyTorch in Section~\ref{deeplearningImpl} and inside PostgreSQL in Section~\ref{sec:impl}. We compare the convergence rate and performance of \sys with other baselines in Section~\ref{sec:eval}. 
We summarize related work in Section~\ref{sec:related} and conclude in Section~\ref{sec:conclusion}.

\section{Preliminaries} 
\label{sec:Preliminaries}
In this section, we briefly review the standard SGD algorithm and its implementation in the state-of-the-art deep learning and in-DB ML systems.

\subsection{Stochastic Gradient Descent (SGD)}

Given a dataset with $m$ training examples $\{\t_i\}_{i\in[m]}$, e.g., $m$ tuples if the training set is stored as a table in a database, the typical ML task essentially solves an optimization problem that can be cast into minimizing
a finite sum over $m$ data examples with respect to model $\x$:
\[
F(\x) = \frac{1}{m}\sum\nolimits_{i=1}^m f_i(\x),
\]
where each $f_i$ corresponds to 
the loss over each training tuple $\t_i$.
SGD is an \emph{iterative} procedure that takes as input hyperparameters such as the learning rate $\eta$ and the maximum number of epochs $S$.
It then works as follows.

\begin{enumerate}[noitemsep,topsep=0pt,parsep=0pt,partopsep=0pt,leftmargin=*]
\item \textbf{Initialization} -- Initialize the model $\x$, often randomly.
\item \textbf{Iterative computation} -- 
In each iteration 
it draws a (batch of)
tuple $\t_i$, \textit{randomly
with replacement}, computes the stochastic gradient 
$\nabla f_i(\x)$
and updates the parameters of model $\x$.
In practice, most systems implement 
a variant, where the random tuples are drawn \emph{without replacement}~\cite{bottou2012stochastic, Feng:2012:TUA:2213836.2213874, de2020random, yun2021can}. 
To achieve this, one \emph{shuffles} all tuples \emph{before} each epoch and sequentially scans 
these shuffled tuples. For each tuple, we compute the stochastic gradient and update the model parameters.

\item \textbf{Termination} -- The procedure ends when it converges (i.e., the parameters of model $\mathbf{x}$ no longer change) or has attained the maximum number of epochs. 
\end{enumerate}

\subsection{Deep Learning Systems} \label{dlsystem}
Deep learning systems such as PyTorch and TensorFlow are now widely used in industry and academia for AI tasks, including image analysis, natural language processing, speech recognition, etc. 
These systems usually leverage SGD optimizer or its variants \cite{SGDoverview, PyTorchOptim, TFOptim} for training deep learning models. To facilitate data loading, these systems classify the datasets into two types, including \emph{map-style} datasets and \emph{iterable-style} datasets. 
Map-style datasets refer to the datasets whose tuples can be randomly accessed given indexes. For example, if an image dataset is stored in an in-memory array as $\langle\T{image}, \T{label} \rangle$ tuples, it is a map-style dataset that can be randomly accessed by the array index. 
Iterable-style datasets refer to the datasets that can only be accessed in sequence, which is usually used for the datasets that cannot fit in memory. 
For map-style datasets, it is easy to shuffle them since we only need to shuffle the indexes and access the tuples based on the shuffled indexes. However, this random access usually leads to low I/O performance for secondary storage, as shown in Figure~\ref{IOTest} in the Appendix.
For iterable-style datasets, PyTorch currently does not provide any built-in shuffling strategies while TensorFlow provides \emph{Sliding-Window Shuffle} using sliding-window based sampling. 
As we will see in Section~\ref{sec:benchmark}, the problem of this shuffling strategy is that it suffers from low accuracy for the clustered dataset.

\subsection{In-database Machine Learning Systems}

There has been a plethora of work in the past decade focusing on in-DB ML~\cite{zhang2014dimmwitted, madlib-paper, Feng:2012:TUA:2213836.2213874, DBLP:conf/sigmod/SchleichOC16, learning-over-join, F-system, linear-algebra, Scale-factorization-ml, DBLP:conf/pods/Khamis0NOS18, scalable-linear-algebra,  Jankov2021distributed,luo2021automatic,yuan2021tensor,kara2018columnml}. 
Most existing in-DB ML systems implement SGD as ``User-Defined Aggregates'' (UDA)~\cite{Feng:2012:TUA:2213836.2213874, madlib-paper}.
Each epoch of SGD is done via an invocation of the corresponding UDA function, where the parameters of model $\x$ are treated as the \textit{state} and updated for each tuple.

To implement the data shuffling step required by SGD, different in-DB ML systems adopt distinct strategies.
For example, some systems such as MADlib~\cite{madlib-paper} and DB4ML~\cite{DB4ML} assume that the training data has already been shuffled, 
so they do not perform any data shuffling.
Other systems, such as Bismarck~\cite{Feng:2012:TUA:2213836.2213874}, do not make this assumption.
Instead, they either perform a \emph{pre-shuffle} of the data in an offline manner and then store the shuffled data as a replica in the database, or perform \emph{partial} data shuffling based on sampling technologies such as \emph{reservoir sampling} and \emph{sliding-window sampling}.
As we will see in the next section, such partial data shuffling strategies, despite alleviating the computation and storage overhead of the preshuffle strategy, raise new issues regarding the convergence of SGD, since the data is insufficiently shuffled and does not follow the purely random order required.

%% file: s2-empirical-study.tex
\section{Data Shuffling Strategies for SGD}
\label{sec:benchmark}

In this section, we present a systematic analysis of data shuffling strategies 
used by existing in-DB ML systems.
We consider five common data shuffling strategies:
(1) \emph{Epoch Shuffle}, (2) \emph{Shuffle Once}, (3) \emph{No Shuffle}, (4) \emph{Sliding-Window Shuffle}~\cite{sliding-window}, and 
(5) \emph{MRS Shuffle}~\cite{Feng:2012:TUA:2213836.2213874}.
We use diverse SGD workloads, including generalized linear models such as logistic regression (LR) and support vector machine (SVM), as well as deep learning models such as VGG~\cite{VGG19} and ResNet~\cite{Resnet}.

\vspace{0.5em}
\noindent \textbf{Experimental Setups.}
We use the \texttt{criteo} dataset~\cite{LIBSVM-dataset} for generalized linear models, and use the \texttt{cifar-10} image dataset~\cite{Cifar-10} for deep learning. 
Each dataset has two versions: a \emph{shuffled} version and a \emph{clustered} version.
In the \emph{shuffled} version,
all tuples are randomly shuffled,
whereas in the \emph{clustered} version 
all tuples are clustered by their \textit{labels}.
The use of \emph{clustered} datasets is inspired by similar settings leveraged in~\cite{Feng:2012:TUA:2213836.2213874}, with the goal of testing the worst-case scenarios of data shuffling strategies for SGD. 
For example, the clustered version of \texttt{criteo} dataset has the \emph{negative} tuples (with ``-1'' labels) ordered before the \emph{positive} tuples (with ``+1'' labels).

\subsection{``Shuffle Once'' and ``Epoch Shuffle''}

The \emph{Shuffle Once} strategy performs an offline shuffle of all data tuples, either in-place or by storing the shuffled tuples as a copy in the database.
SGD is then executed over this shuffled copy without any further shuffle during the execution of SGD.
Albeit a simple (but costly) idea, it is arguably a strong baseline that many state-of-the-art in-DB ML systems assume when they take as input an already shuffled dataset.
For \emph{Epoch Shuffle}, it shuffles 
the training set before each training epoch. Therefore, the data shuffling 
cost of \emph{Epoch Shuffle} grows \emph{linearly} with respect to the number of epochs.

\noindent \textbf{Convergence.}
As illustrated in Figure~\ref{bench-clustered-data}, \emph{Shuffle Once} can achieve a convergence rate comparable to \emph{Epoch Shuffle}
on \emph{both} shuffled and clustered datasets, 
confirming previous observations~\cite{Feng:2012:TUA:2213836.2213874}.

\vspace{0.5em}
\noindent \textbf{Performance.} Although \emph{Shuffle Once} reduces the number of data shuffles to only once, the shuffle itself can be very expensive on large datasets due to the random access of tuples, as we will show in our experiments. Previous work has also reported that shuffling a huge dataset could not be finished in one day~\cite{Feng:2012:TUA:2213836.2213874}.
Another problem of \emph{Shuffle Once} is that,
when in-place shuffle is not feasible,
it needs to duplicate the data, which can double the space overhead.

\begin{figure*}[t!]
\centering
\includegraphics[width=0.99\textwidth]{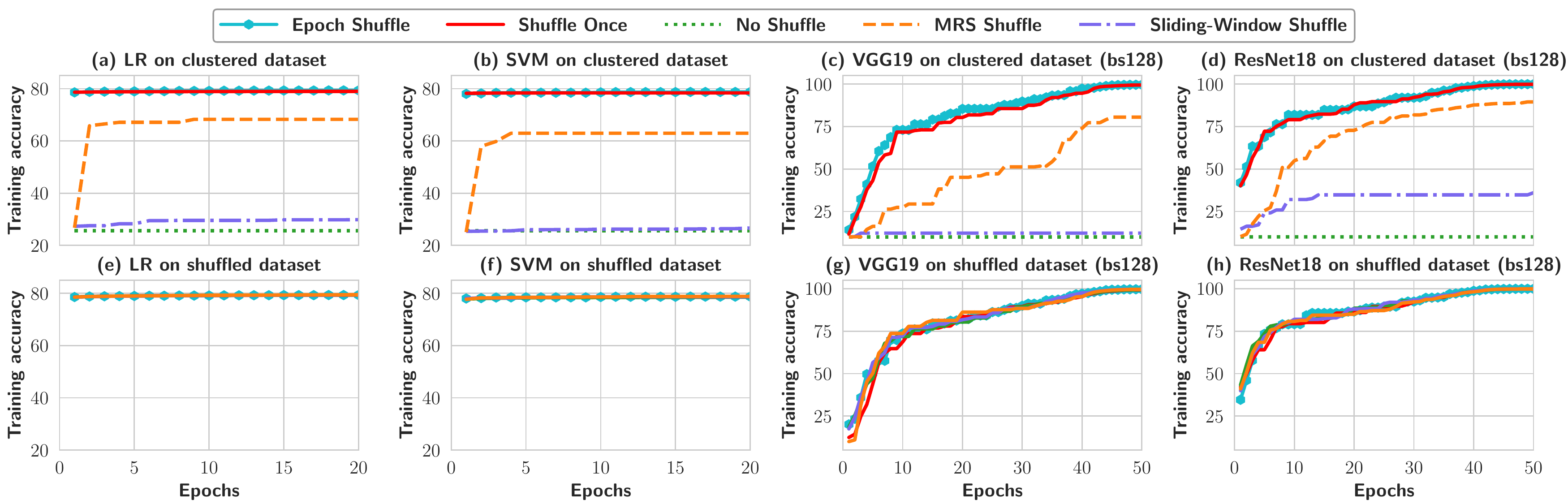}
\caption{The convergence rates of SGD with different data shuffling strategies, for clustered and shuffled datasets, using the same buffer size (10\% of the dataset size) for MRS and Sliding-Window Shuffles.}
\label{bench-clustered-data}
\end{figure*}

\subsection{``No Shuffle''}

The \emph{No Shuffle} strategy does not perform any data shuffle at all, i.e., the SGD algorithm runs over the given data order in each epoch. Simply running MADlib over a
 dataset or running PyTorch over \T{IterableDataset} picks
the \emph{No Shuffle} strategy.

\vspace{0.5em}
\noindent \textbf{Convergence.}
On \emph{shuffled} data, \emph{No Shuffle} can achieve comparable convergence rate to \emph{Shuffle Once}. However, for \emph{clustered} data, \emph{No Shuffle} leads to a significantly lower model accuracy. 
This is not surprising, as SGD relies on a random data order to converge.

\vspace{0.5em}
\noindent \textbf{Performance.}
\emph{No Shuffle} is the fastest among the five data shuffling strategies, as it can always \emph{sequentially}, instead of randomly, access the data tuples~\cite{HDDPerf}.

\subsection{``Sliding-Window Shuffle''}

The \emph{Sliding-Window Shuffle} strategy leverages a sliding window to perform \emph{partial} data shuffling, which is used by TensorFlow~\cite{sliding-window}. It involves the following steps:

\begin{enumerate}[noitemsep,topsep=0pt,parsep=0pt,partopsep=0pt,leftmargin=*]
    \item Allocate a sliding window and fill tuples as they are scanned.
    \item Randomly select a tuple from the window and use it for the SGD computation.
    The slot of the selected tuple in the window is then filled in by the next incoming tuple.
    \item Repeat (2) until all tuples are scanned.
\end{enumerate}

\vspace{0.5em}
\noindent \textbf{Convergence.}
As illustrated in Figure~\ref{bench-clustered-data}, for clustered datasets, \emph{Sliding-Window Shuffle} can achieve higher model accuracy than \emph{No Shuffle} but lower accuracy than \emph{Shuffle Once} when SGD converges.
The reason is that this strategy shuffles the data only \emph{partially}.
For two data examples $\t_i$ and $\t_j$ where 
$\t_i$ is stored much earlier than $\t_j$ ($i \ll j$),
it is likely that $\t_i$ is still selected before $\t_j$.
As a result, on the clustered datasets used in our study, negative tuples are more likely to be selected (for SGD) before positive ones, which distorts the training data seen by SGD and leads to low 
model accuracy.

\noindent \textbf{Performance.}
\emph{Sliding-Window Shuffle} can achieve I/O performance comparable to \emph{No Shuffle}, as it also only needs to \emph{sequentially} access the data tuples with limited additional CPU overhead 
to maintain and sample from the sliding window.

\subsection{``Multiplexed Reservoir Sampling Shuffle''}

\emph{Multiplexed Reservoir Sampling (MRS) Shuffle} uses two concurrent threads to read tuples and update a shared model~\cite{Feng:2012:TUA:2213836.2213874}. 
The first thread sequentially scans the dataset and performs \emph{reservoir sampling}.
The \emph{sampled} (i.e., selected) tuples are stored in a buffer $B_1$ and the \emph{dropped} (i.e., not selected) ones are used for SGD. 
The second thread loops over the sampled tuples using another buffer $B_2$ for SGD, where tuples are simply copied from the buffer $B_1$.

\vspace{0.5em}
\noindent \textbf{Convergence.}
As illustrated in Figure~\ref{bench-clustered-data},
\emph{MRS Shuffle} achieves higher accuracy than \emph{Sliding-Window Shuffle} but lower accuracy than \emph{Shuffle Once} when SGD converges.
The reason is quite similar to that given to \emph{Sliding-Window Shuffle}, as the shuffle based on reservoir sampling is again \emph{partial} and therefore is insufficient when dealing with clustered data.
Specifically, the order of the \emph{dropped} tuples is also generally \emph{increasing}, i.e., if $i \ll j$, $\t_i$ is likely to be processed by SGD before $\t_j$.
Moreover, looping over the sampled tuples may lead to suboptimal data distribution---the sampled tuples in the looping buffer $B_2$ may be used multiple times, which can cause data skew and thus decrease the model accuracy.

\vspace{0.5em}
\noindent \textbf{Performance.}
\emph{MRS Shuffle} is fast, as the first thread only needs to \emph{sequentially} scan the tuples for reservoir sampling.
It is slightly slower than \emph{Sliding-Window Shuffle} and \emph{No Shuffle}, as there is a second thread that loops over the buffered tuples.

\subsection{Analysis and Summary}
\label{sec:benchmark:summary}

\begin{table*}[t] \small
\centering
\caption{A Summary of Different Shuffling Strategies, where 
\textbf{bold} fonts represent the ``ideal'' scenario. We assume all methods that require an in-memory buffer have reasonably large buffer size, e.g., 1\%-10\% of the dataset size.}
\begin{tabular}{c|cccc} 
\hline
 {\bf Shuffling Strategy}  & \textbf{Convergence Behavior} & \textbf{I/O Perf.} & \textbf{In-memory buffer} & \textbf{Additional Disk Space}  \\
\hline
 \textit{No Shuffle}      & Slow; Lower Accuracy     & \textbf{Fast} & \textbf{No}  & \textbf{No} \\ 
 \textit{Epoch Shuffle}   & \textbf{Fast; High Accuracy}  & Slow    & Yes & 2$\times$ data size  \\
 \textit{Shuffle Once}    & \textbf{Fast; High Accuracy}  & Slow    & Yes & 2$\times$ data size  \\
 \hline
 \textit{MRS Shuffle}~\cite{Feng:2012:TUA:2213836.2213874}     & Worse than \textit{Shuffle Once}   & \textbf{Fast} & Yes & \textbf{No}\\
 \textit{Sliding-Window}~\cite{sliding-window}  & Worse than \textit{Shuffle Once}   & \textbf{Fast} & Yes & \textbf{No}\\
\hline
\sys             & \textbf{Comparable to \textit{Shuffle Once}}  & \textbf{Fast} & Yes & \textbf{No}\\
\hline
\end{tabular}
\label{tab:current-shuffles-tab}
\end{table*}

Table~\ref{tab:current-shuffles-tab} summarizes the characteristics 
of different data shuffling strategies.
As discussed, the effectiveness of data shuffling strategies for SGD largely depends on two somewhat conflicting factors, namely, (1) the degree of \emph{data randomness} of the shuffled tuples and (2) the \emph{I/O efficiency} when scanning data from disk.
There is an apparent \emph{trade-off} between these two factors:
\begin{itemize}
    \item The more random the tuples are, the better the convergence rate of SGD is. \emph{Epoch Shuffle} introduces data randomness at the highest level, but is too expensive to implement in in-DB ML systems. \emph{Shuffle Once} also introduces significant data randomness, which is usually the best practice in terms of SGD convergence for in-DB ML systems.
    \item A higher degree of randomness implies more random disk accesses and thus lower I/O efficiency. As a result, the \emph{No Shuffle} strategy is the best in terms of I/O efficiency.
\end{itemize}
The other 
strategies (\emph{Sliding-Window} and \emph{MRS})
try to sacrifice data randomness for better I/O efficiency, leaving room for improvement.

\begin{example}\label{example:id-label-distribution}

\begin{figure*}
    \centering
    \begin{subfigure}[No Shuffle]{
        \includegraphics[width=0.225\textwidth]{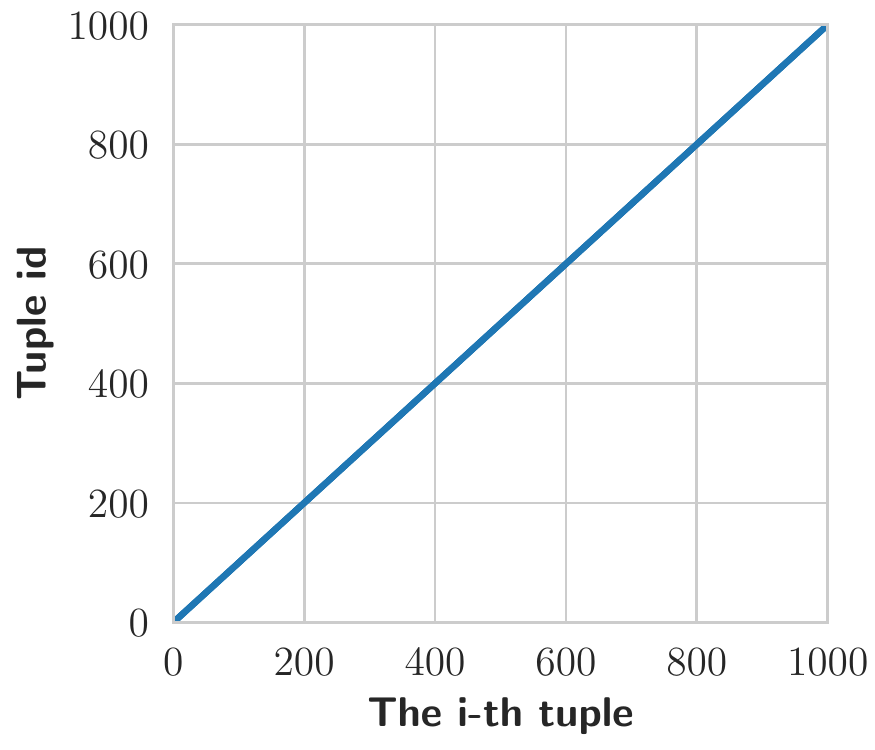}
        \label{fig:no-shuffle-id-dist}}
    \end{subfigure}
    \begin{subfigure}[Sliding-Window Shuffle]{
        \includegraphics[width=0.225\textwidth]{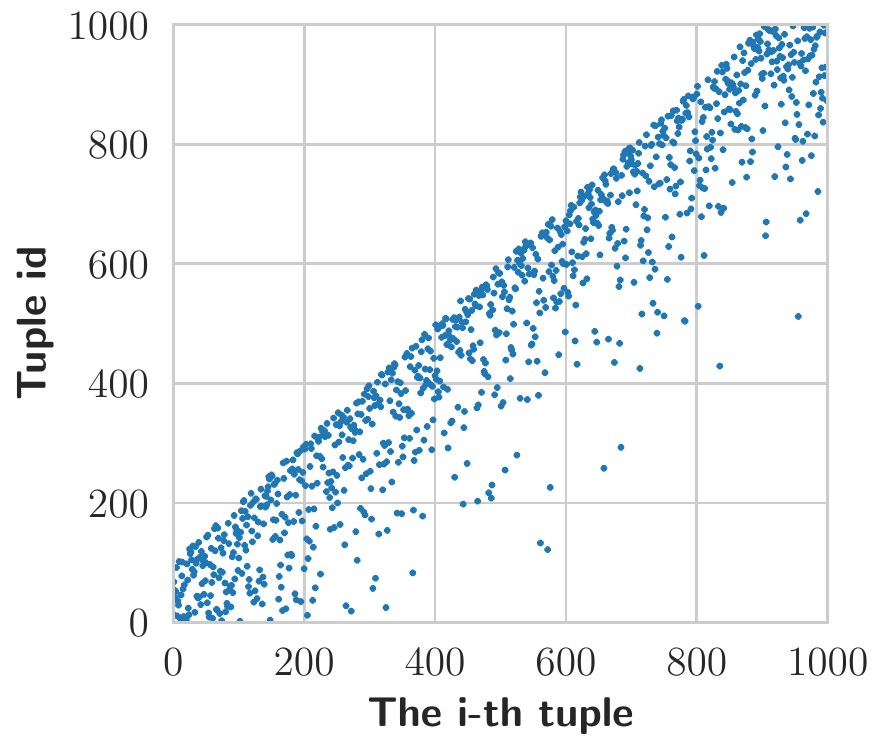}
        \label{fig:sliding-window-shuffle-id-dist}}
    \end{subfigure}
    \begin{subfigure}[MRS Shuffle]{
        \includegraphics[width=0.225\textwidth]{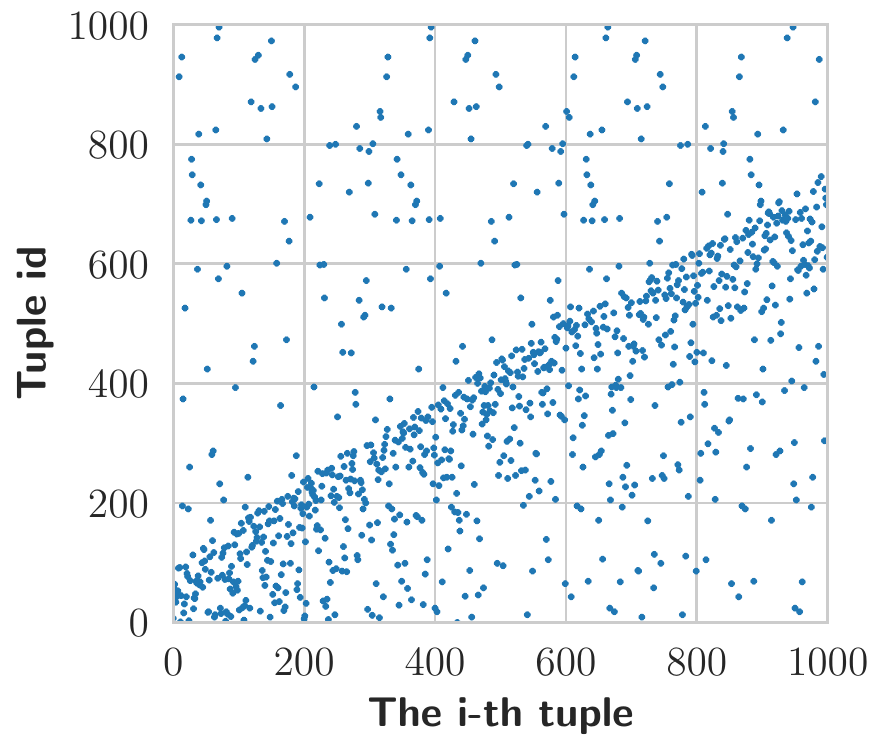}
        \label{fig:mrs-shuffle-id-dist}}
    \end{subfigure}
     \begin{subfigure}[Full Shuffle (ideal)]{
        \includegraphics[width=0.225\textwidth]{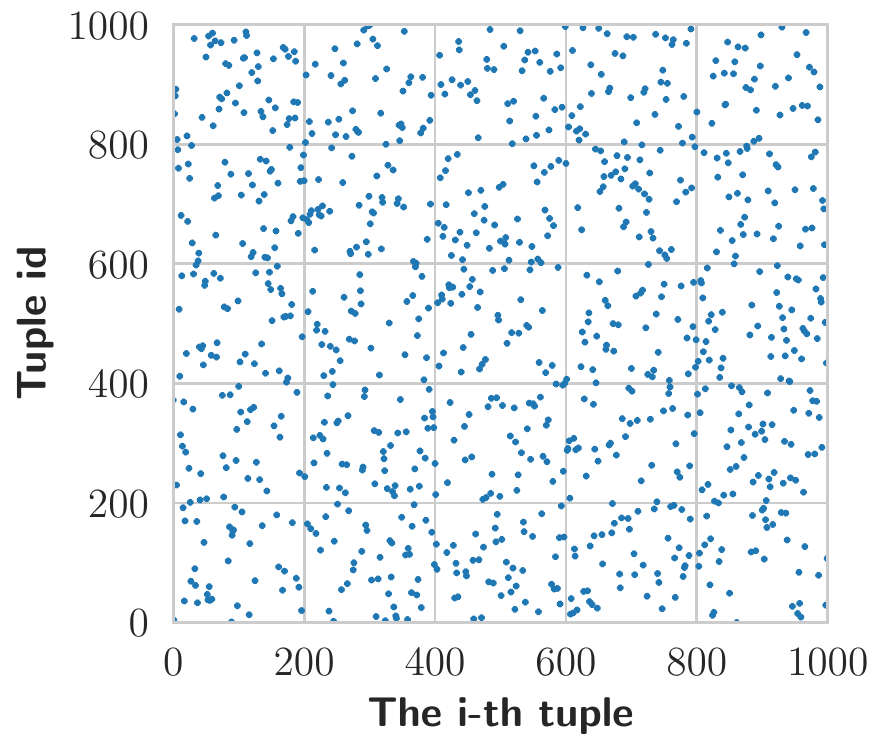}
        \label{fig:full-shuffle-id-dist}}
    \end{subfigure}

       \begin{subfigure}[No Shuffle]{
        \includegraphics[width=0.225\textwidth]{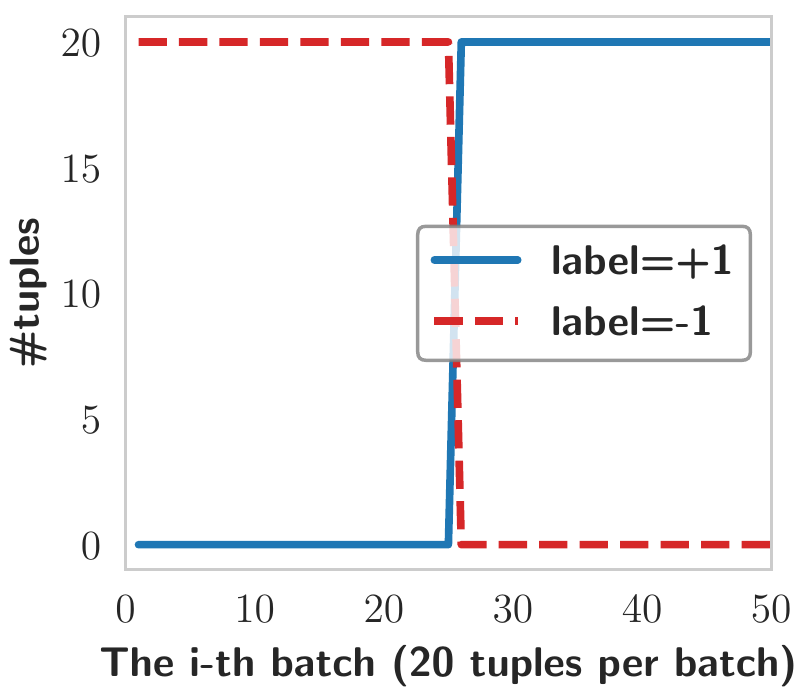}
        \label{fig:no-shuffle-label-dist}}
    \end{subfigure}
    \begin{subfigure}[Sliding-Window Shuffle]{
        \includegraphics[width=0.225\textwidth]{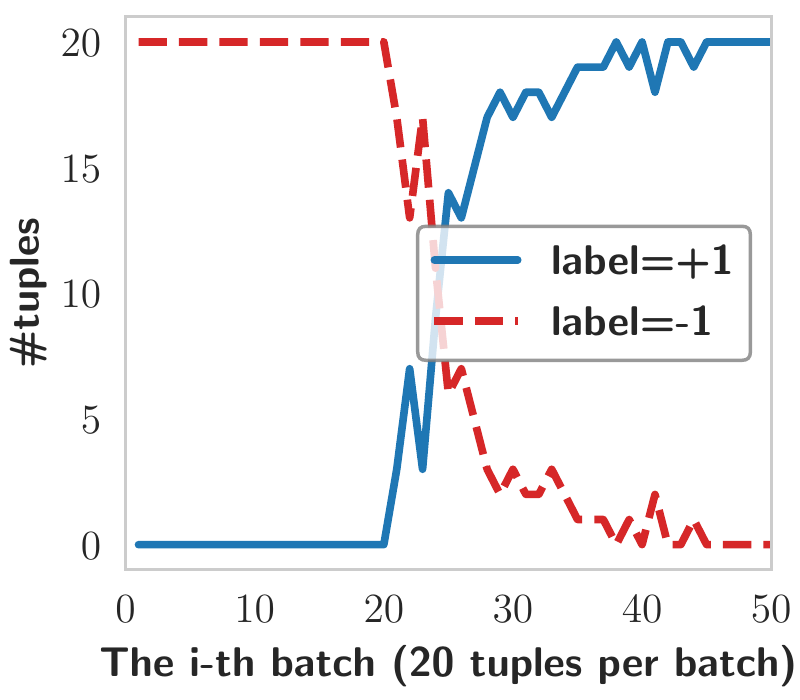}
        \label{fig:sliding-window-shuffle-label-dist}}
    \end{subfigure}
    \begin{subfigure}[MRS Shuffle]{
        \includegraphics[width=0.225\textwidth]{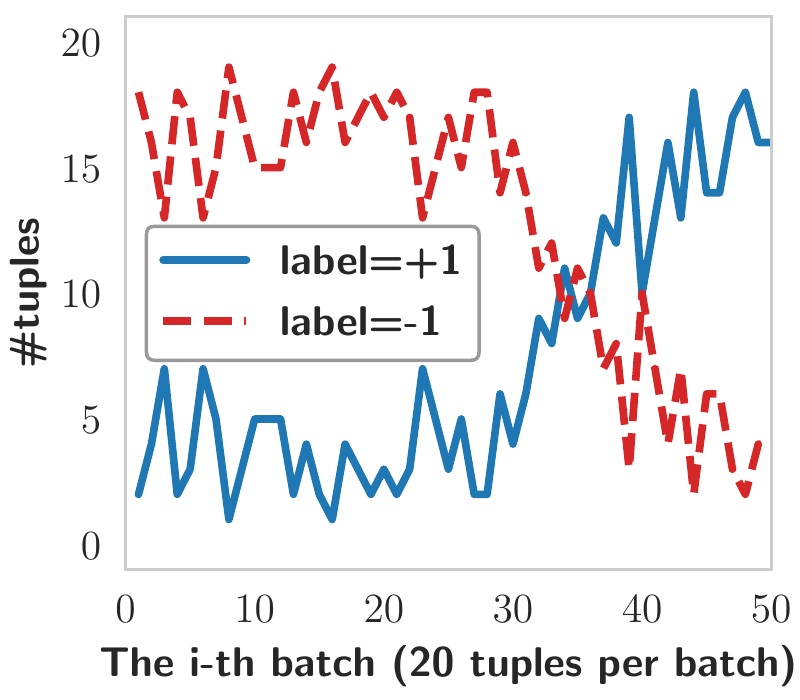}
        \label{fig:mrs-shuffle-label-dist}}
    \end{subfigure}
    \begin{subfigure}[Full Shuffle (ideal)]{
        \includegraphics[width=0.225\textwidth]{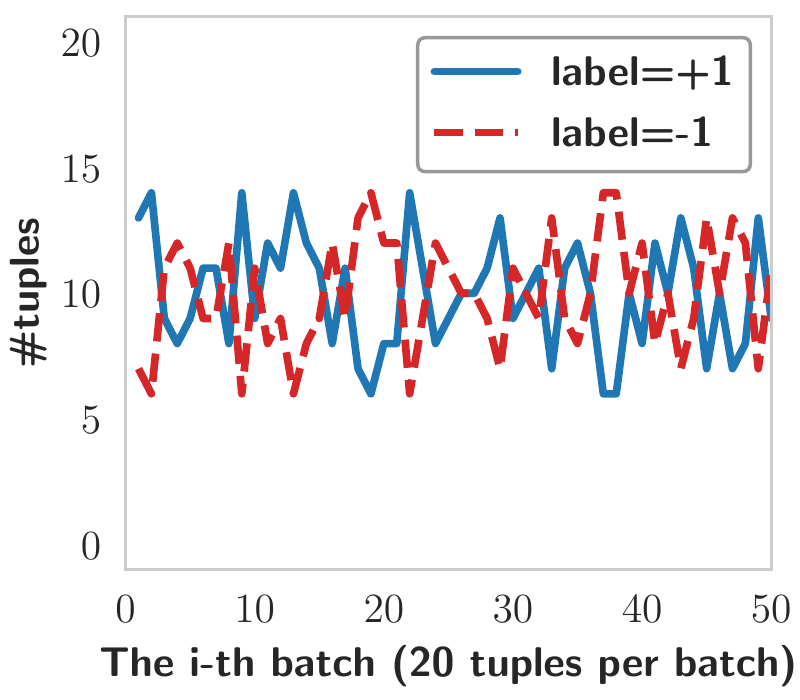}
        \label{fig:full-shuffle-label-dist}}
    \end{subfigure}
    \caption{The tuple id distribution (a-d) and the corresponding label distribution (e-h) of data shuffling strategies. Tuple id denotes the tuple position after shuffling. \#tuple denotes the number of negative/positive tuples in every 20 tuples shuffled.} 
    \label{shuffle-result-demo}
\end{figure*}

To better understand these issues, consider a clustered dataset with 1,000 tuples,
each of which has a \emph{tuple\_id} and a \emph{label}, where \emph{tuple\_id} of the $i$-th tuple is $i$. The first 500 tuples are negative and the next 500 tuples are positive. 
Figure~\ref{shuffle-result-demo} plots the distributions of \emph{tuple\_id} and corresponding \emph{label}s after Sliding-Window and MRS Shuffle, and compare them with the ideal distributions from a full shuffle.
Specifically, the \emph{tuple\_id} distribution illustrates the positions of the tuples after shuffling, whereas the \emph{label} distribution illustrates the number of negative/positive tuples in every 20 tuples shuffled.

We can observe that Sliding-Window results in a ``linear''-shape distribution of the \emph{tuple\_id} after shuffling, as shown by Figure~\ref{fig:sliding-window-shuffle-id-dist}, which suggests that the tuples are almost \emph{not} shuffled. The corresponding label distribution in Figure~\ref{fig:sliding-window-shuffle-label-dist} further confirms this, where almost all negative labels still appear before positive ones after shuffling.
Similar patterns can be observed for MRS in Figures~\ref{fig:mrs-shuffle-id-dist} and~\ref{fig:mrs-shuffle-label-dist}, though MRS has improved over Sliding-Window.
In summary, the data randomness achieved by Sliding-Window or MRS is far from the ideal case, as shown in Figures~\ref{fig:full-shuffle-id-dist} and~\ref{fig:full-shuffle-label-dist}.

\end{example}

%% file: s3-our-approach-design.tex
\section{\sys}
\label{sec:design}

As illustrated in the previous 
section,  data shuffling strategies used by existing systems can be 
suboptimal when dealing with clustered
data. Although recent efforts
have significantly improved over
baseline methods, there is still large 
room for improvement.
Inspired by these previous efforts, 
we present a simple but novel data shuffling strategy named \sys.
The key idea of \sys lies in the following two-level hierarchical shuffling mechanism:
\begin{quote}
\em
We first randomly select 
a set of blocks (each block refers to a set of contiguous tuples) and put them 
into an in-memory buffer; we then 
randomly shuffle
all tuples in the buffer and use 
them for SGD.
\end{quote}
Despite its simplicity, \sys is 
highly effective. 
In terms of hardware efficiency,
when the block size is 
large enough (e.g., 10MB+), a random access on the \emph{block} level can 
be as efficient as a \emph{sequential} scan, as shown in the I/O performance test on HDD and SSD in Appdenx~\ref{IOTestOnHDDandSSD}. 
In terms of statistical efficiency,
as we will show,
\textit{given the same buffer size},
\sys converges much better than \textit{Sliding-Window} and \emph{MRS}.
Nevertheless, both the convergence 
analysis and its integration into PyTorch and
PostgreSQL are non-trivial. 
In the following, we first describe 
the \sys algorithm precisely and then present 
a theoretical analysis on its convergence behavior.

\vspace{0.5em}
\noindent\textbf{Notations and Definitions.} 
The following is a list of notations and definitions that we will use:

\begin{itemize}[noitemsep,topsep=0pt,parsep=0pt,partopsep=0pt,leftmargin=*]
\item $\|\cdot \|$: the $\ell_2$-norm for vectors and the spectral norm for matrices;

\item $\lesssim$: For two arbitrary vectors $a, g$, we use $a_s\lesssim g_s$ to denote that
there exists a certain constant $C$ that satisfies $a_s\leq C g_s$ for all $s$;

\item $N$, the total number of blocks ($N\geq 2$);
	
\item $n$, the buffer size (i.e., the number of blocks kept in the buffer);

\item $b$, the size (number of tuples) of each data block;
    
\item $B_l$, the set of tuple indices in the $l$-th block ($l \in [N]$ and $|B_l| = b$);
    
\item $m$, the number of tuples for the finite-sum objective ($m = N b$);

\item $f_i(\cdot )$, the function associated with the $i$-th tuple;

\item $\nabla F(\cdot )$ and $\nabla f_i(\cdot)$, the gradients of the functions $F(\cdot)$ and $f_i(\cdot)$; 

\item $H_i(\cdot) := \nabla^2 f_i(\cdot )$, the Hessian matrix of the function $f_i(\cdot)$;

\item $\x^*$, the global minimizer of the function $F(\cdot)$;
    
\item $\x^s_k$, the model $\x$ in the $k$-th iteration at the $s$-th epoch;

\item $\mu$-strongly convexity: function $F(\x)$ is $\mu$-strongly convex if $\forall\x, \y$,
\begin{align} \label{eq:strongcvx}
F(\x)  \geq F(\y) + \left\langle\x-\y, \nabla F(\y) \right \rangle + \frac{\mu}{2}\|\x-\y\|^2.
\end{align}

\end{itemize}

\subsection{The \sys Algorithm}
\label{algo-design}

\begin{algorithm}[t]

\caption{\sys Algorithm}\label{alg:partialrs_sample}
	\begin{algorithmic}[1]
		\STATE {\bfseries Input: } $N$ blocks with $m$ total tuples, total epochs $S$ ($S \geq 1$), $a \geq 1$, $F(\cdot) = \frac{1}{m} \sum_{i=1}^m f_i(\cdot)$.
		\STATE {\bfseries Initialize} $\x^0_0$;
		\FOR{$s=0,\cdots,S$}
        	\STATE Randomly pick $n$ blocks without replacement, each containing $b$ tuples. Load these blocks into the buffer; 
		    \STATE Shuffle tuple indices among all $n$ blocks in the buffer and obtain the permutation $\boldsymbol{\psi}_s$;
		    \FOR{$k=1,...,bn$}
                \STATE Update  $\x^s_k = \x^s_{k-1} - \eta_s \nabla f_{\psi_s(k)} \left( \x^s_{k-1} \right)$;
		    \ENDFOR
		    \STATE $\x^{s+1}_0 = \x^{s}_{bn}$;
		\ENDFOR
		\STATE {\bfseries Return} $x_{bn}^{S}$;
	\end{algorithmic}
\end{algorithm}

Algorithm~\ref{alg:partialrs_sample} illustrates the details of \sys.
At each epoch (say, the $s$-th epoch), \sys runs the following steps:
\begin{enumerate}[noitemsep,parsep=5pt,partopsep=0pt,leftmargin=*]
\item (\textbf{Sample}) Randomly sample $n$ blocks out of $N$ data blocks \emph{without replacement} and load the $n$ blocks into the buffer. 
Note that
we use \emph{sample without replacement} to avoid 
visiting the same tuple multiple times for each epoch, which can converge faster and is a standard
practice in most  ML systems~\cite{bottou2009curiously, bottou2012stochastic,gurbuzbalaban2019random,haochen2018random, Feng:2012:TUA:2213836.2213874}.
\item (\textbf{Shuffle}) Shuffle all tuples in the buffer. 
We use $\boldsymbol{\psi}_s$ to denote an ordered set, whose elements are 
the indices of the shuffled tuples at the $s$-th epoch.
The size of $\boldsymbol{\psi}_s$ is $bn$, where $b$ is the number of tuples per block. $\boldsymbol{\psi}_s(k)$ is the $k$-th element in $\boldsymbol{\psi}_s$.

\item (\textbf{Update}) Perform gradient descent by scanning each tuple with the shuffle indices in $\boldsymbol{\psi}_s$, yielding the updating rule 
\begin{align*}
    \x^s_k = \x^s_{k-1} - \eta_s \nabla f_{\boldsymbol{\psi}_s(k)} \left( \x^s_{k-1} \right),
\end{align*}
where $\nabla f_{\boldsymbol{\psi}_s(k)}(\cdot)$ is the gradient function averaging the gradients of all samples in the tuples indexed by $\boldsymbol{\psi}_s(k)$, and $\eta_s$ is the learning rate for gradient descent at the epoch $s$. The parameter update is performed for all $k=1,...,bn$ in one epoch.
\end{enumerate}

\begin{figure}[t!]
    \centering
    \begin{subfigure}[Tuple id distribution of \sys]{
        \includegraphics[width=0.3\textwidth]{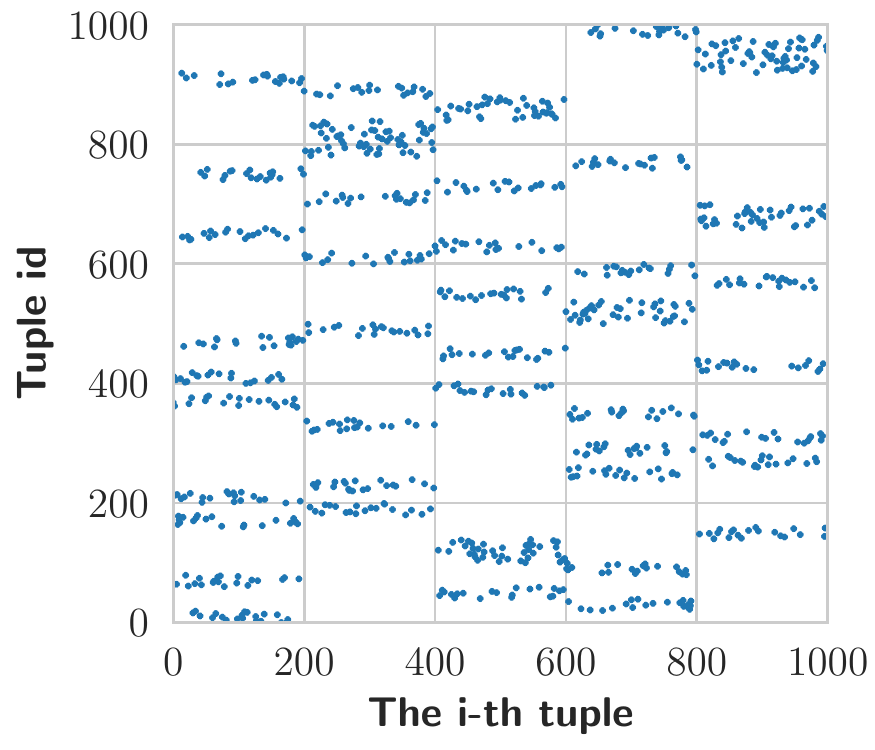}
        \label{fig:our-shuffle-id-dist}}
    \end{subfigure}
    \begin{subfigure}[Label distribution of \sys]{
        \includegraphics[width=0.3\textwidth]{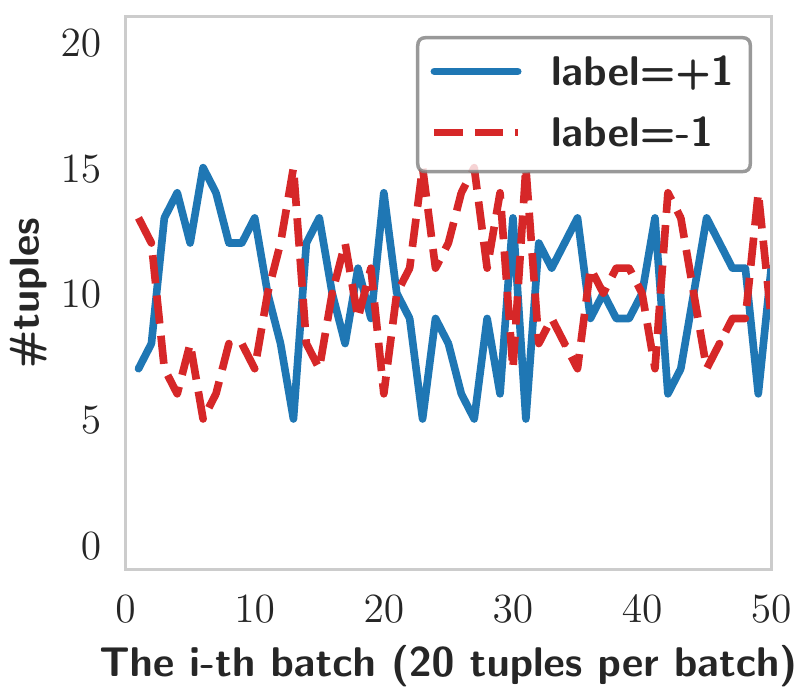}
        \label{fig:our-shuffle-label-dist}}
    \end{subfigure}
    \caption{The tuple id/label distribution of \sys. }
    \label{our-shuffle-result}
\end{figure}

\vspace{0.5em}
\noindent \textbf{Intuition behind \sys.}
Before we present the formal 
theoretical analysis, we first
illustrate the intuition behind \sys, following the same example used in Section~\ref{sec:benchmark:summary}.

\begin{example}

Consider the same settings as those in Example~\ref{example:id-label-distribution}.
Recall that \sys contains both block-level and tuple-level shuffles.
Suppose that the block-level shuffle generates a random order of blocks as \{b20, b8, b45, b0, ...\} and the buffer can hold 10 blocks. The tuple-level shuffle will put the first 10 blocks into the buffer, whose \emph{tuple\_id}s are \{b20[400, 419], b8[160, 179], b45[900, 919], b0[0, 19], ...\}. After shuffling, the buffered tuples will have random \emph{tuple\_id}s in a large \emph{non-contiguous} interval that is the union of \{[0, 19], [160, 179], ..., [900, 919]\}, as shown in the first 200 tuples in Figure~\ref{fig:our-shuffle-id-dist}. The buffered tuples therefore follow a random order closer to what is given by a full shuffle. As a result, the corresponding label distribution, as shown in Figure~\ref{fig:our-shuffle-label-dist}, is closer to a \emph{uniform} distribution. 
\end{example}

\noindent \textbf{Performance.} \label{pref-analysis}
While \emph{No Shuffle} only requires \emph{sequential} I/Os, our \sys needs to (1) randomly access blocks, (2) copy all tuples in these blocks into a buffer, and (3) shuffle the tuples inside the buffer.
Here, random accessing a block means randomly picking a block and reading the tuples of this block from secondary storage (e.g., disk) into memory.
If the block size is large enough, the I/O performances of random and sequential accesses are close. 
\sys incurs additional overheads for \emph{buffer copy} and
in-memory \emph{shuffle}.
However, these I/O overheads 
can be hidden via standard techniques such as double buffering.
As we will show in our experiments on PostgreSQL,
the optimized version of \sys only incurs 11.7\% additional overhead compared to the most efficient \emph{No Shuffle} baseline.

\subsection{Convergence Analysis} \label{Convergence_analysis}

Despite its simplicity, the convergence analysis of \sys is not trivial---even reasoning about
the convergence of SGD with
\textit{sample without replacement}
is a open question for decades~\cite{shamir2016without,gurbuzbalaban2015random,ying2019stochastic,haochen2018random}, not to say 
a hierarchical sampling scheme like ours.  Luckily, a recent theoretical 
advancement~\cite{haochen2018random} provides us with the technical language 
to reason about \sys's convergence.
In the following, we present a novel theoretical analysis for \sys.

Note that in our analysis, one epoch represents going through all the tuples in the sampled $n$ blocks.

{\assumption \label{assump:global} We make the following standard assumptions, as that in other previous work on SGD convergence analysis \cite{Bottou18, First-Order}: 

\begin{enumerate}
    \item $ F(\cdot )$ and $ f_i(\cdot)$ are twice continuously differentiable.
    
    \item $L$-Lipschitz gradient: $\exists L \in \mathbb{R}_+$, $\|\nabla f_i (\x) - \nabla f_i (\y)\| \leq L \|\x-\y\|$ for all $i \in [m]$.
    
    \item $L_H$-Lipschitz Hessian matrix: $\|H_i(\x) - H_i(\y)\| \leq L_H\|\x-\y\|$ for all $i\in [m]$.
    
   \item Bounded gradient: $\exists G \in \mathbb{R}_+$, $\|\nabla f_i (\x_k^s)  \| \leq G$ for all $i \in [m]$, $k\in [K-1]$, and $s\in\{0,1\ldots,S\}$. 
   
    \item Bounded Variance: $\E_{\xi} [\|\nabla  f_{\xi} (\x) - \nabla F(\x)  \|^2] = \frac{1}{m}\sum_{i=1}^m\|\nabla f_i(\x)-\nabla F(\x)\|^2\leq \sigma^2$ where $\xi$ is the random variable that takes the values in $[m]$ with equal probability $1/m$.
    
\end{enumerate}
}

\noindent \textbf{Factor $h_D$}. 
In our analysis, we use the factor $h_D$ to characterize the upper bound of a block-wise data variance: 
\begin{align*}
\frac{1}{N}\sum_{l=1}^N \left\|\nabla f_{B_l}(\x) - \nabla F(\x) \right\|^2 \leq h_D  \frac{\sigma^2}{b},
\end{align*}
where $b=|B_l|$ is the size of each data block (recall the definition of $b$).
Here, $h_D$ is an essential parameter to measure the ``cluster'' effect within the original data blocks.  
Let's consider two extreme cases: 1) ($h_D=1$) all samples in the data set are fully shuffled, such that the data in each block follows the same distribution; 2) ($h_D=b$) samples are well clustered in each block, for example, all samples in the same block are identical. Therefore, the larger $h_D$, the more ``clustered'' the data.

We now present the results for both strongly convex objectives (corresponding to 
generalized linear models) and non-convex objectives (corresponding to deep learning models) respectively, in order to show the correctness and efficiency of \sys. The proof of the following theorems is at the end of the Appendix.

\paragraph*{\underline{Strongly convex objective}}
We first show the result for strongly convex objective that satisfies the strong convexity condition \eqref{eq:strongcvx}.

\begin{theorem} \label{thm:partialrs_sample}
Suppose that $F(\x)$ is a smooth and $\mu$-strongly convex function. Let $T = S n b$, that is, the total number of samples used in training and $S\geq 1$ is the number of tuples iterated, and choosing $\eta_s = \frac{6}{bn\mu(s+a)}$ 
where  $a \geq \max \left\{\frac{8LG + 24L^2 + 28L_H G}{\mu^2}, \frac{24L}{\mu} \right\}$, under Assumption \ref{assump:global}, \sys has the following convergence rate 
\begin{align} \label{eq:thm1}
    \E [F\left( \bar{\x}_S \right) - F(\x^*)]  \lesssim  ( 1-\alpha)h_D\sigma^2 \frac{1}{T}  + \beta \frac{1}{T^2}  + \gamma \frac{  m^3}{T^3},
\end{align}
where $\bar{\x}_S = \frac{\sum_s (s+a)^3 \x_s}{\sum_s (s+a)^3}$, and
\begin{align*}
\alpha := \frac{n-1}{N-1}, \beta := \alpha^2 + ( 1-\alpha)^2(b-1)^2, \gamma := \frac{n^3}{N^3}.
\end{align*}
\end{theorem}

\noindent\textbf{Tightness.} The convergence rate of \sys is tight in the following sense:
\begin{itemize}[noitemsep,topsep=0pt,parsep=0pt,partopsep=0pt,leftmargin=*] 
\item $\alpha = 1$: It means that $n=N$, i.e., all tuples are fetched to the buffer. Then \sys reduces to full-shuffle SGD~\cite{haochen2018random}. In this case, the upper bound in Theorem~\ref{thm:partialrs_sample} is $O(1/T^2 + m^3/T^3)$, which matches the result of the full shuffle SGD algorithm~\cite{haochen2018random}. 
\item $\alpha = 0$: It means that $n=1$, i.e., only sampling one block each time. Then \sys is very close to \emph{mini-batch} SGD (by viewing a block as a mini-batch), except that the model is updated once per data tuple. Ignoring the higher-order terms in~\eqref{eq:thm1}, our upper bound $O(h_D\sigma^2/ T)$ is consistent with that of mini-batch SGD. 
\end{itemize}

\vspace{0.5em}
\noindent\textbf{Comparison to vanilla SGD.} In vanilla SGD, we only randomly select one tuple from the dataset to update the model. It admits the convergence rate $O(\sigma^2/T)$. Comparing to the leading term $(1-\alpha)h_D(\sigma^2/T)$ in~\eqref{eq:thm1} for our algorithm, if $n \gg  (h_D -1) (N-1)/h_D + 1$ (for $h_D > 0$), $(1-\alpha) h_D$ will be much smaller than $1$, indicating that our algorithm outperforms vanilla SGD in terms of sample complexity. 
It is also worth noting that, even if $n$ is small, \sys may still significantly outperform vanilla SGD.
Assuming that reading 
a random single tuple incurs an overhead of
$t_{\text{lat}}+t_{\text{t}}$ and reading 
a block of $b$ tuples
incurs an overhead of 
$t_{\text{lat}}+bt_{\text{t}}$, where
$t_{\text{lat}}$ is 
the ``latency''
for one read/write operation that 
does not grow linearly with respect to 
the amount of data that one reads/writes
(e.g., SSD read/write latency or HDD ``seek and rotate'' time),
and $t_{\text{t}}$ is the time that one
needs to transfer a single tuple.
To reach an error of $\epsilon$, 
vanilla SGD requires, in physical time,
\[
O\left(\frac{\sigma^2}{\epsilon} t_{\text{lat}} +  \frac{\sigma^2}{\epsilon} t_{\text{t}} \right),
\]
whereas \sys requires
\[
O\left((1-\alpha) \frac{h_D}{b} \cdot \frac{\sigma^2}{\epsilon} t_{\text{lat}}
+  (1-\alpha) h_D \cdot \frac{\sigma^2}{\epsilon} t_{\text{t}} \right).
\]
Because
$(1-\alpha) \frac{h_D}{b} < 1$,
\sys always provides benefit over vanilla SGD in terms of the read/write latency
$t_{\text{lat}}$.
When $t_{\text{lat}}$ dominates the transfer time $t_{\text{t}}$,
\sys can outperform vanilla SGD even for small buffers.

\paragraph*{\underline{Non-convex objective}}
We further conduct an analysis on objectives that are non-convex 
or satisfy the Polyak-\L ojasiewicz condition,
which leads to similar insights on the behavior of \sys.

\begin{theorem} \label{thm:partialrs_sample_smooth}
Suppose that $F(\x)$ is a smooth function. Letting $T = S n b$ be the number of tuples iterated, under Assumption \ref{assump:global}, \sys has the following convergence rate: 
\begin{enumerate}
\item When $\alpha \leq \frac{N-2}{N-1}$,  choosing $\eta_s = \frac{1}{\sqrt{bn(1-\alpha)h_D\sigma^2 S} }$ 
and assuming  $S \geq \frac{bn(\frac{104}{3}L + \frac{4}{3}L_H)^2}{\sigma^2 (1-\alpha) h_D}$, we have
\begin{align*}
\frac{1}{S} \sum_{s=1}^S \E\|\nabla F(\x_0^s)\|^2 \lesssim &  ( 1-\alpha)^{1/2} \frac{\sqrt{h_D } \sigma }{\sqrt{T}}  + \beta \frac{1}{T}  +\gamma \frac{  m^3}{T^{\frac{3}{2}}},
\end{align*}
where the factors are defined as
\begin{align*}
&\alpha := \frac{n-1}{N-1}, \beta := \frac{\alpha^2}{1-\alpha} \frac{1}{h_D \sigma^2} + ( 1-\alpha)\frac{(b-1)^2}{h_D \sigma^2}, \gamma := \frac{n^3}{(1-\alpha)N^3};
\end{align*}

\item When $\alpha =1$, choosing $\eta_s = \frac{1}{(m S)^{\frac{1}{3}} }$ and assuming $S \geq (\frac{416}{3}L + \frac{16}{3}L_H)^3 b^2 n^3 / N$, we have
\begin{align*}
 \frac{1}{S} \sum_{s=1}^S \E\|\nabla F(\x_0^s)\|^2\lesssim  \frac{1}{T^{\frac{2}{3}}}  +\gamma' \frac{  m^3}{T},
\end{align*}
where we define $\gamma' := \frac{n^3}{N^3}$.
\end{enumerate}

\end{theorem}

We can apply a similar analysis as that of Theorem~ \ref{thm:partialrs_sample} to compare \sys with vanilla SGD, in terms of convergence rate, and reach similar insights.

%% file: s4-implementation.tex
\section{Multi-process \sys and the implementation in PyTorch}
\label{deeplearningImpl}

We have integrated \sys into PyTorch, one of the state-of-the-art deep learning systems. For this integration, the main challenge is to extend \sys to work for the parallel/distributed environment, since deep learning systems usually use multiple processes with multiple GPUs to train models. For example, apart from single-process training, PyTorch also supports multi-process training using the \texttt{DistributedDataParallel} (DDP) mode \cite{DDP-PyTorch}. In this mode, PyTorch runs multiple processes (typically one process per GPU) in parallel in a single machine or across a number of machines to train models. We call this parallel/distributed multi-process training the \emph{multi-process} mode.  

\subsection{Multi-process \sys}

We found that our \sys can be naturally extended to work in the multi-process mode, by enhancing the tuple-level shuffle. As mentioned in Section~\ref{algo-design}, \sys contains both block-level shuffle and tuple-level shuffle. For the multi-process mode as shown in Figure~\ref{fig:pytorch-impl}(a), we can naturally implement block-level shuffle by randomly distributing data blocks into different processes. For the tuple-level shuffle, we can use multi-buffer based shuffling instead of single-buffer based shuffling---in each process we allocate a local buffer to read blocks and shuffle their tuples locally. For the following SGD computation, the deep learning system itself can read the shuffled tuples to perform the forward/backward/update computation as well as gradient/parameter communication/synchronization among different processes. We name this enhanced \sys as multi-process \sys and implement it as a new \texttt{CorgiPileDataset} API in PyTorch as follows. Users just need to initialize the \texttt{CorgiPileDataset} with necessary parameters and then use it as one parameter of the \texttt{DataLoader}. The \texttt{train()} constantly extracts a batch of tuples from \texttt{DataLoader} and performs mini-batch SGD computation on the tuples. We next detail the implementation of multi-process \sys in PyTorch with four steps, and further demonstrate that this multi-process \sys can achieve similar random data order for SGD as that of the single-process \sys in Section~\ref{demo-single-multiple-corgipile}. 

\begin{verbatim} 
    train_dataset = CorgiPileDataset(dataset_path, block_index_path, other_args)
    train_loader = torch.utils.data.DataLoader(train_dataset, other_args)
    train(train_loader, model, other_args)
\end{verbatim}

\textbf{(1) Block partitioning:} The first step is to partition the dataset into blocks. In the parallel/distributed environment, we usually store the dataset on the block-based parallel/distributed file systems such as HDFS~\cite{HDFS}, Amazon EBS~\cite{EBS}, and Lustre~\cite{Lustre}. For example, our ETH Euler cluster~\cite{ETH} uses high-performance parallel Lustre file system, which reads/writes data in blocks (by default 4 MB)~\cite{ETH-Euler} and does not allow users to store/read massive small files like raw images in a directory. Therefore, for training ImageNet dataset with 1.3 million raw images~\cite{Imagenet}, we need to convert these images into binary data files like widely-used \texttt{TFRecords}~\cite{TFRecord, TFRecord-PyTorch} and store them in Lustre before training. In addition, for the next block-level shuffle, we need to build a block index to identify the start/end of each block, by using the block information provided by the file system or running indexing tools such as \texttt{PyTorch-TFRecord} library~\cite{TFRecord-PyTorch} on the dataset. If the dataset itself contains tuple index such as \textit{map-style dataset} in PyTorch, we can also partition the dataset into blocks according to the tuple index. 

\vspace{0.5em}
\textbf{(2) Block shuffle:} For block shuffle, we just need to let each process randomly pick $\textit{BN}/\textit{PN}$ blocks, where $\textit{BN}$ is the number of total blocks and $\textit{PN}$ is the number of total processes. We implement this block shuffle in our \texttt{CorgiPileDataSet}, which uses the \emph{dataset path} and the \emph{block index} of all blocks as the input parameters. In the $i$-th process, at the beginning of each epoch, \texttt{CorgiPileDataSet} first shuffles the block index of all blocks and splits it into $\textit{PN}$ parts, and then only reads the blocks with indexes in the $i$-th part. Since we set the same random seed for each process, the shuffled block indexes in all the processes are the same. Therefore, different processes can obtain different blocks.

 \begin{figure*}[t!]
\centering
\includegraphics[width=0.9\textwidth]{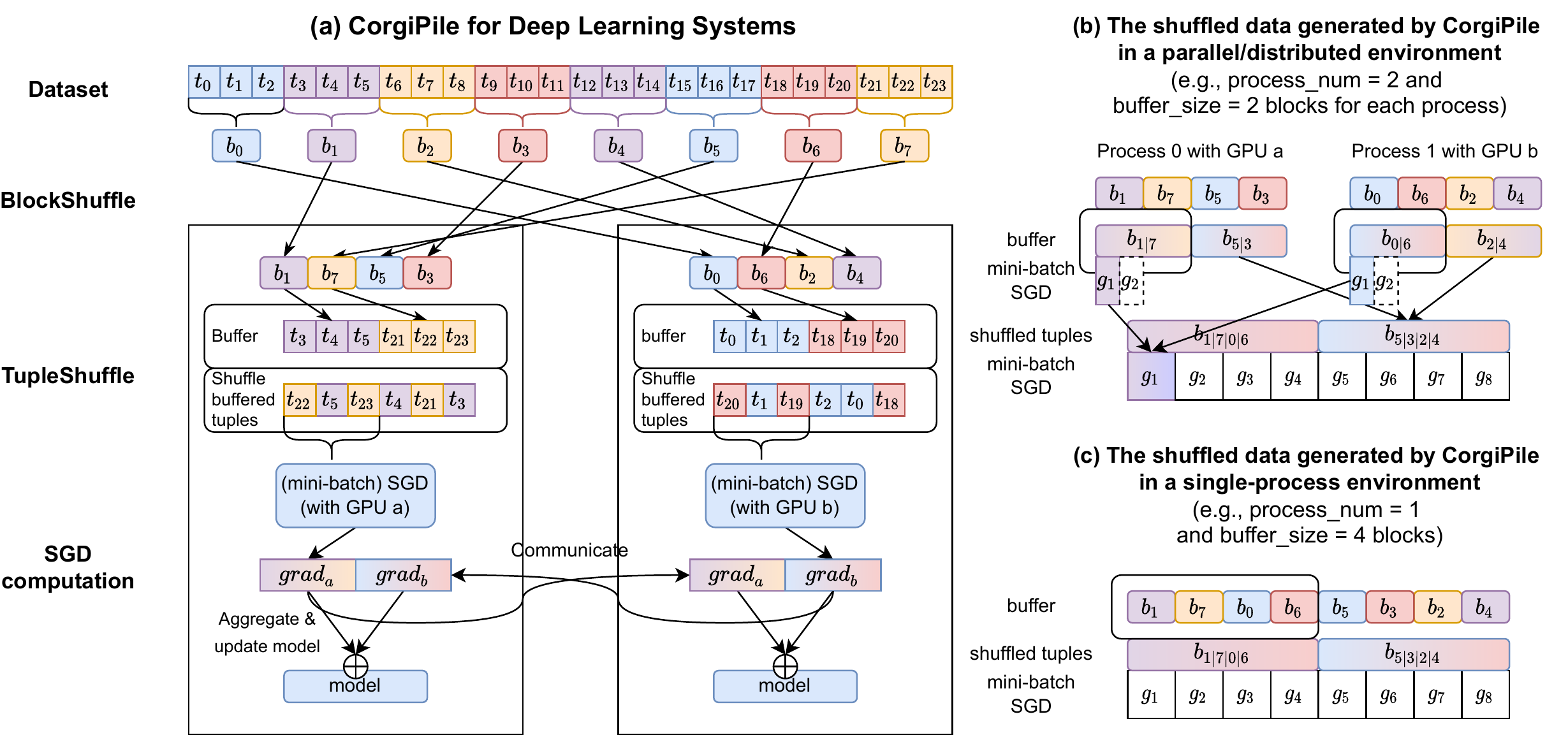}
\caption{(a) The implementation of \sys in a parallel/distributed environment, e.g., PyTorch with multiple processes and GPUs. (b) The shuffled data generated by \sys in a parallel/distributed environment are similar with (c) the shuffled data generated by \sys in a single-process environment.}
\label{fig:pytorch-impl}
\end{figure*}
\vspace{0.5em}
\textbf{(3) Tuple shuffle:} For tuple shuffle, each process first  allocates a small buffer in memory and then constantly reads the blocks into the buffer. Once the buffer is full, the process will shuffle the buffered tuples. This is also implemented in \texttt{CorgiPileDataSet}, whose \texttt{iter()} reads blocks into a buffer and returns the shuffled tuples one by one to the next SGD computation. Note that the buffer size here is much smaller than that used in the single-process mode. For example, if we set $\textit{buffer}\_\textit{size} = \textit{BS}$ in the single-process mode, we can choose $\textit{buffer}\_\textit{size} = \textit{BS}/\textit{PN}$ for the multi-process mode. We will compare the shuffled data orders of these two modes later.

\vspace{0.5em}
\textbf{(4) SGD computation:} After the block shuffle and tuple shuffle, each process can perform mini-batch SGD computation on the shuffled tuples. Different from the single-process mode that performs mini-batch SGD on the whole dataset with $\textit{batch}\_\textit{size}=\textit{bs}$, each process in the multi-process mode performs mini-batch SGD on partial dataset with a smaller batch size ($\textit{bs}/\textit{PN}$) and updates the model with gradient synchronization every batch. As shown in Figure~\ref{fig:pytorch-impl}(a), after each batch, all the processes will synchronize/aggregate the gradients using communication protocols like \texttt{AllReduce}, and then updates the local model. This procedure is executed inside $\texttt{train()}$ of PyTorch, which automatically performs the gradient computation/communication/synchronization and model update every time after reading a batch of tuples from \texttt{CorgiPileDataset}. 

\subsection{Single-process \sys vs. Multi-process \sys}
\label{demo-single-multiple-corgipile}

We found that the generated data order of multi-process \sys is comparable to that of single-process \sys, when using mini-batch SGD. Here, we use a simple example as shown in Figure~\ref{fig:pytorch-impl} to demonstrate this. As shown in Figure~\ref{fig:pytorch-impl}(a), there are two processes and each of them randomly picks 4 blocks from the dataset. Each process can read two blocks into the buffer at once and shuffle their tuples. As a result, as shown in Figure~\ref{fig:pytorch-impl}(b), the shuffled tuples of process 0 are in sequence from block 1/7 (denoted as $b_{1|7}$) and then from block 5/3. Likewise, the shuffled tuples of process 1 are in sequence from block 0/6 and then from block 2/4. Since PyTorch sequentially performs mini-batch SGD on the first $\textit{batch}\_\textit{size}/\textit{PN}$ tuples of each process (denoted as $g_1$ on block 1/7 and block 0/6) and synchronizes their gradients (sums and averages $g_1$) every batch, it is the same as performing mini-batch SGD on the first $\textit{batch}\_\textit{size}$ tuples from block 1/7/0/6 (i.e., $g_1$ on $b_{1|7|0|6}$). Therefore, from the view of the whole dataset, PyTorch performs mini-batch SGD on the tuples first from block 1/7/0/6 and then from block 5/3/2/4. This is similar to the data order generated by single-process mode in Figure~\ref{fig:pytorch-impl}(c), where the buffer size is $\textit{PN}$ times of that of the multi-process mode. Here, the $\textit{PN}$ is 2 and the buffer can keep 4 blocks at once. In summary, due to block shuffle, multi-buffer based tuple shuffle, and synchronization protocol of mini-batch SGD, multi-process \sys can achieve shuffled data order similar to that of the single-process \sys.

\section{Implementation in the Database}
\label{sec:impl}

We integrate \sys into PostgreSQL. 
Our implementation provides a simple SQL-based interface for users to invoke \sys, with the following query template: 
\begin{verbatim}
                    SELECT * FROM table TRAIN BY model WITH params.
\end{verbatim}
This interface is similar to that offered by existing in-DB ML systems such as MADlib~\cite{madlib-paper, madlib} and Bismarck~\cite{Feng:2012:TUA:2213836.2213874}.
Examples of the \texttt{params} include \emph{learning\_rate = 0.1}, \emph{max\_epoch\_num = 20}, and \emph{block\_size = 10MB}.
\sys outputs various metrics after each epoch, such as \emph{training loss}, \emph{accuracy}, and \emph{execution time}.

\paragraph*{\underline{The Need of a Deeper Integration}}
Unlike existing in-DB ML systems, we choose not to implement our \sys strategy 
using UDAs. Instead, we choose to integrate \sys into 
PostgreSQL by introducing physical operators. \textit{Is it necessary for such a deeper integration with database system internals, compared to a potential UDA-based implementation without modifying the internals?}

While a UDA-based implementation is conceptually possible, it is not natural for \sys,
which requires accessing low-level data layout information such as table pages, tuples, and buffers.
A deeper integration with database internals makes it much easier to reuse such functionalities that have been built into the core APIs offered by database system internals but 
not yet have been externally exposed as UDAs.
Moreover, such a physical-level integration opens up the door for more advanced optimizations, such as double-buffering that will be illustrated in Section~\ref{Optimizations}.

\subsection{Design Considerations}

As discussed in Section~\ref{algo-design}, \sys consists of three steps: (1) block-level shuffling, (2) tuple-level shuffling, and (3) SGD computation.
Accordingly, we design three physical operators, one for each of the three steps:
\begin{itemize}
    \item \texttt{BlockShuffle}, an operator for randomly accessing blocks;
    \item \texttt{TupleShuffle}, an operator for buffering a batch of blocks and shuffling their tuples;
    \item \texttt{SGD}, an operator for the SGD computation.
\end{itemize}
We then chain these three operators together to form a pipeline, and implement the \texttt{getNext()} method for each operator, following the classic Volcano-style execution model~\cite{Graefe94} that is also the query execution paradigm of PostgreSQL.

One challenge is the design and implementation of the \emph{SGD} operator, which requires an \emph{iterative} procedure that is not typically supported by database systems.
We choose to implement it by leveraging the built-in \emph{re-scan} mechanism of PostgreSQL to reshuffle and reread the data after each epoch.

We store the dataset as a table in PostgreSQL using the schema of $\langle \textit{id}, \textit{features\_k}[], \textit{features\_v}[], \textit{label} \rangle$, 
which is similar to the one used by Bismarck~\cite{Feng:2012:TUA:2213836.2213874}.
For sparse datasets, $\textit{features\_k}[]$ indicates which dimensions have non-zero values, and $\textit{features\_v}[]$ refers to the corresponding non-zero feature values. For dense dataset, only $\textit{features\_v}[]$ is used.

Currently, we store the (learned) machine learning model as an in-memory object (a C-style Struct) with an ID in the PostgreSQL's kernel instead of using UDA. Users can initialize the model hyperparameters via the query. For the inference, users can execute a query as ``\texttt{SELECT table PREDICT BY model ID}'', which invokes the learned model for prediction.

\subsection{Physical Operators}

The control flow of the three operators is shown in Figure~\ref{fig:implemenation}, which leverages a PostgreSQL's pull-style dataflow to read tuples and perform the SGD computation. 
In the following, we assume that the readers are familiar with
the structure of PostgreSQL's operators, e.g., functions such as \T{ExecInit()} and \T{getNext()}.

After parsing the input query, \sys invokes \T{ExecInit()} of each operator to initialize their states such as ML models and I/O buffers. 
At each epoch, the \texttt{SGD} operator pulls tuples from the \texttt{TupleShuffle} operator for SGD computation, which further pulls tuples from the \texttt{BlockShuffle} operator.
The \texttt{BlockShuffle} operator is responsible for shuffling blocks and reading their tuples.
We now present the implementation of these operators.

\begin{figure*}[t]
\centering
\includegraphics[width=0.9\textwidth]{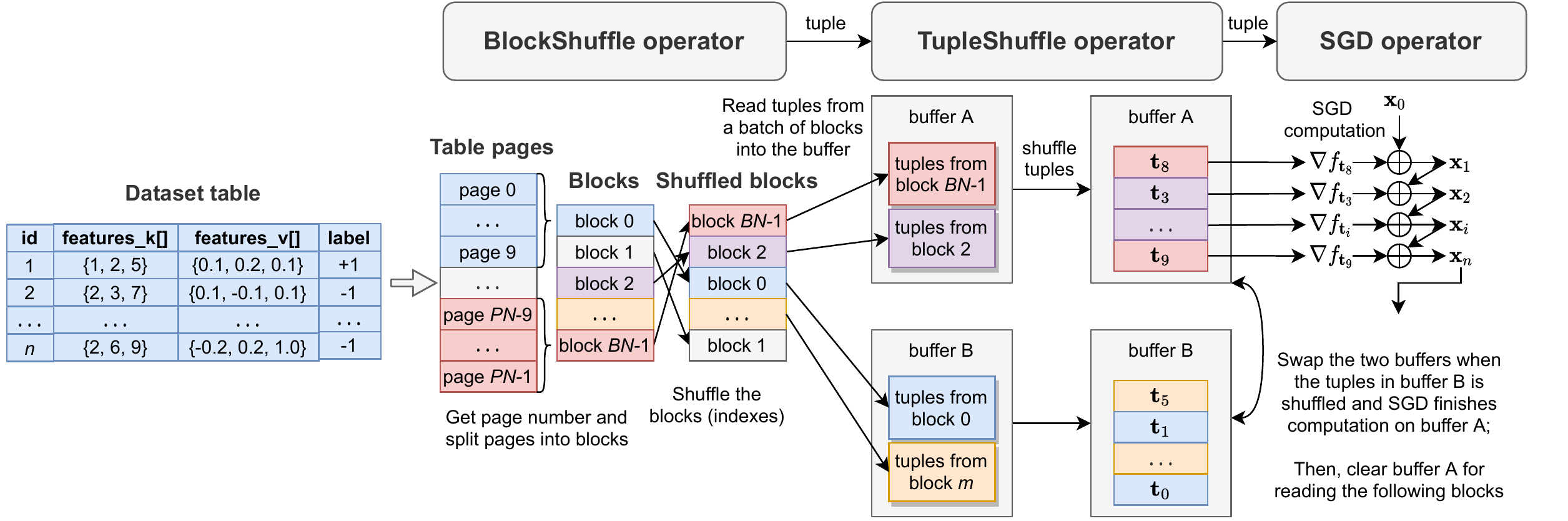}
\caption{The implementation of \sys with three new operators and the ``double-buffering'' optimization, in PostgreSQL.}

\label{fig:implemenation}
\end{figure*}

\vspace{0.5em}
\textbf{(1) BlockShuffle:}
This operator first obtains the total number of pages by PostgreSQL's  internal function as \T{RelationGetNumberOfBlocks()}.
It then computes the number of blocks $\textit{BN}$ by 
$\textit{BN} = \textit{page\_num} * \textit{page\_size} / \textit{block\_size}.$
After that, it shuffles the block indices $[0, \dots, \textit{BN}-1]$ and obtains shuffled block ids, where each block corresponds to a batch of contiguous table pages.
For each shuffled block id, it reads the corresponding pages using \texttt{heapgetpage()} and returns each fetched tuple to the \texttt{TupleShuffle} operator. 
The \texttt{BlockShuffle} operator is somewhat similar to PostgreSQL's \texttt{Scan} operator, although the \texttt{Scan} operator reads pages sequentially instead of randomly.

\vspace{0.5em}
\textbf{(2) TupleShuffle:}
It first allocates a buffer, and then pulls the tuples one by one from the \texttt{BlockShuffle} operator by invoking its \T{ExecTupleShuffle()}, i.e., \T{getNext()}. 
Each pulled tuple is transformed to an \T{SGDTuple} object, which is then copied to the buffer.
Once the buffer is filled, it shuffles the buffered tuples, which is similar to how the \T{Sort} operator works in PostgreSQL. 
After that, the shuffled tuples are returned one by one to the \T{SGD} operator.

\vspace{0.5em}
\textbf{(3) SGD:}
It first initializes an ML model in \T{ExecInitSGD()} and then executes SGD in \T{ExecSGD()}. 
At each epoch, \T{ExecSGD()} pulls tuples from \T{TupleShuffle} one by one, and runs SGD computation.
Once all tuples are processed, an epoch ends.
It then has to reshuffle and reread the tuples for the next epoch, using the \emph{re-scan} mechanism of PostgreSQL.
Specifically, after each epoch, \texttt{SGD} invokes \T{ExecReScan()} of \T{TupleShuffle} to reset the I/O states of the buffer.
It further invokes \texttt{ExecReScan()} of \T{BlockShuffle} to reshuffle the block ids. After that, \T{SGD} operator can reread shuffled tuples via \T{ExecSGD()} for the next epoch. This is similar to the multiple table/index scans in PostgreSQL's \T{NestedLoopJoin}.

\lstdefinestyle{mystyle}{
    backgroundcolor=\color{backcolour},
    commentstyle=\color{codegreen},
    keywordstyle=\color{magenta},
    numberstyle=\tiny\color{codegray},
    stringstyle=\color{codepurple},
    basicstyle=\ttfamily\footnotesize,
    breakatwhitespace=false,         
    breaklines=true,                 
    captionpos=b,                    
    keepspaces=true,                 
    numbers=left,                    
    numbersep=5pt,                  
    showspaces=false,                
    showstringspaces=false,
    showtabs=false,                  
    tabsize=3
}
\lstset{style=mystyle}

\subsection{Optimizations} \label{Optimizations}

As discussed in Section~\ref{algo-design}, \sys introduces 
additional overheads for \emph{buffer copy} and \emph{shuffle}.
To reduce them, we use a double-buffering strategy as shown in Figure~\ref{fig:implemenation}. 
Specifically, we launch two concurrent threads for \T{TupleShuffle} with two buffers.
One \emph{write} thread is responsible for pulling tuples from \T{BlockShuffle} into one buffer and shuffling the buffered tuples; the other \emph{read} thread is responsible for reading tuples from another buffer and returning them to \T{SGD}. 
The two buffers are swapped once one is full and the other has been consumed by \T{SGD}. 
As a result, the data loading (i.e., block-level and tuple-level shuffling) and SGD computation can be executed concurrently, reducing the overhead.

%% file: s5-evaluation.tex
\begin{table}[t]
\centering
\caption{Datasets. The first four are from LIBSVM~\cite{LIBSVM-dataset}. For \T{criteo}, we extract 98M tuples from the \T{criteo} terabyte dataset. For \T{yfcc}, we extract 3.6M tuples from the \T{yfcc100m} dataset~\cite{yfcc-dataset}; the \emph{outdoor} and \emph{indoor} tuples are marked as negative (-1) and positive (+1).
\#Tuples like 4.5/0.5M refer to 4.5M tuples for training and 0.5M tuples for testing.}
\label{table:data}
{
\begin{tabular}{l|c| c|c|c}
Name & Type & \#Tuples & \#Features & Size in DB or on disk  \\
\toprule
higgs & dense & 10.0/1.0M & 28 & 2.8 GB \\
susy & dense & 4.5/0.5M & 18 & 0.9 GB \\
epsilon & dense & 0.4/0.1M & 2,000 & 6.3 GB  \\
\toprule
criteo & sparse &92/6.0M & 1,000,000 & 50 GB \\
yfcc & dense & 3.3/0.3M & 4,096 & 55 GB  \\
\toprule
ImageNet & image &1.3/0.05M & 224*224*3 & 150 GB \\
cifar-10 & image & 0.05/0.01M & 3,072 & 178 MB  \\
yelp-review-full & text & 0.65/0.05M & - & 600 MB \\
\toprule

\end{tabular}
}
\end{table}

\section{Evaluation}
\label{sec:eval}

We evaluate \sys on both deep learning and in-DB ML systems. Our goal is to 
study the statistical and hardware efficiency of \sys  when applied to these systems, i.e., whether it can 
achieve both high accuracy and high performance. For deep learning system, we integrate \sys into PyTorch and compare it with other shuffling strategies, using both image classification and natural language processing workloads.
For in-DB ML systems, we compare our PostgreSQL-based implementation with two state-of-the-art systems, {\em Apache MADlib and Bismarck} with diverse linear models and datasets. Next, we first evaluate deep learning models in Section~\ref{dl-results}. After that, we evaluate linear models with standard SGD in PostgreSQL in Section~\ref{SGD-LR-SVM}. We further evaluate linear models with mini-batch SGD as well as other types of (continuous, multi-class, and feature-ordered) datasets in PostgreSQL in Section~\ref{Mini-batch-SGD}.

\subsection{Experimental Setup}

\subsubsection{Runtime} 
For deep learning workloads, we perform them on our ETH Euler cluster~\cite{ETH} as batch jobs. Each job can use maximum 16 CPU cores, 160 GB RAM, and 8 NVIDIA GeForce RTX 2080 Ti GPUs. The datasets are stored in the cluster's block-based Lustre parallel file system. 

For in-DB ML workloads, we perform the experiments on a single \textit{ecs.i2.xlarge} node in Alibaba Cloud.
It has 2 physical cores (4 vCPU), 32 GB RAM, 1000 GB HDD, and 894 GB SSD. The HDD has a maximum 140 MB/s bandwidth, and the SSD has a maximum 1 GB/s bandwidth. 
Moreover, \sys only uses a single physical core, and we bind the two threads (see Section~\ref{Optimizations}) to the same physical core using the ``\T{taskset -c}'' command.
We run all experiments in PostgreSQL under CentOS 7.6, and we clear the OS cache before running each experiment.

\subsubsection{Datasets}
For deep learning, we use both \T{cifar-10} dataset with 10 classes~\cite{Cifar-10} and \T{ImageNet} dataset with 1,000 classes~\cite{Imagenet} for image classification. We also use \T{yelp-review-full} dataset~\cite{yelp} with 5 classes for text classification.
For in-DB ML, we use a variety of datasets in our evaluation, including dense/sparse and small/large ones as shown in Table~\ref{table:data}. 
The datasets in Table~\ref{table:data} are stored in PostgreSQL for in-DB ML experiments.
For both deep learning and in-DB ML, we focus on the evaluation over the \emph{clustered} datasets, since SGD with various data shuffling strategies can achieve comparable convergence rates on the \emph{shuffled} datasets, as shown in Figure~\ref{bench-clustered-data}.

\subsubsection{Models and Parameters}
For the evaluation on deep learning system, we perform the classical VGG19 and ResNet18 models 
on the \T{cifar-10} dataset, and perform more complex ResNet50 model on the \T{ImageNet} dataset. We also perform the classical HAN \cite{HAN} and TextCNN \cite{textcnn} models for text classification, with pre-trained word embeddings \cite{glove} on the \T{yelp-review-full} dataset.
For the evaluation on in-DB ML systems, 
we mainly train two popular generalized linear models, logistic regression (LR) and support vector machine (SVM), that are also supported by Bismarck and MADlib. We briefly report the evaluation results for other liner models such as linear regression and softmax regression, which are currently only supported by MADlib.

Currently, Bismarck and MADlib only support two of the baseline data shuffling strategies, namely, \emph{No Shuffle} and \emph{Shuffle Once}, which we compare our PostgreSQL-based implementation against. 
Note that the code of \emph{MRS Shuffle} has not been released 
by Bismarck yet.\footnote{We have confirmed this with the author of Bismarck (private communication).}
Therefore, 
we leave it out of our end-to-end comparisons.
Instead, we implemented \emph{MRS Shuffle} by ourselves in PyTorch
and compare with it
when we discuss the convergence behavior of
different data shuffling strategies (like Figure~\ref{fig:convergence-rates}).

The model hyperparameters include the learning rate, 
the decay factor, and the maximum number of epochs. By default,
we use an exponential learning rate decay with 0.95. 
We set the number of epochs to 20 for in-DB ML and 50 for deep learning models. Only for ResNet50 on \texttt{ImageNet}, we set the epoch number to 100 and decay the learning rate every 30 epochs, following the official PyTorch-ImageNet code~\cite{ImageNet-PyTorch}. We use grid search to tune the best learning rate from \{0.1, 0.01, 0.001\}. 
For in-DB ML, we use the same initial parameters 
and hyperparameters among the compared systems, including MADlib, Bismarck, and \sys.

\subsubsection{Settings of \sys}
\sys has two more parameters, i.e., the buffer size and the block size.
We experiment with a diverse range of buffer sizes in \{1\%, 2\%, 5\%, 10\%\}
and the block size is chosen in \{2MB, 10MB, 50MB\}.
We always use the same buffer size (by default 10\% of the whole dataset size) for \emph{Sliding-Window Shuffle}, \emph{MRS Shuffle}, and our \sys.

\subsubsection{Settings of PostgreSQL}
For PostgreSQL, we set the \T{work\_mem} to be the maximum RAM size and 
tune \T{shared\_buffers}.
Note that PostgreSQL can further compress high-dimensional datasets using the so-called TOAST~\cite{TOAST} technology, which tries to compress large field value or break it into multiple physical rows. 
For our dense \T{epsilon} and \T{yfcc} datasets with 2,000+ dimensions, PostgreSQL uses TOAST to compress their \emph{features\_v} columns.

\subsection{Evaluation with Deep Learning System}
\label{dl-results}

\sys is a general data shuffling strategy for any SGD implementation. To understand its impact on deep learning systems and workloads, 
we implement the \sys strategy as well as others in PyTorch and compare them
using deep learning models, for both image classification and text classification. In the following parts, we first evaluate the end-to-end performance and convergence rate of \sys on the \T{ImageNet} dataset. We then
study the convergence rate of \sys and compare it with others on other datasets, including \T{cifar-10} and \T{yelp-review-full}.  Furthermore, we explore whether \sys can work on other first-order optimization methods such as Adam~\cite{Adam-paper} in Section~\ref{adam-results}.

\subsubsection{Performance comparison} To evaluate the performance of \sys in PyTorch, we perform ResNet50 model on \texttt{ImageNet} dataset, which has 1.3 million images in 1,000 classes. We run this experiment using multi-process \sys with 8 GPUs and 16 CPU cores in our ETH Euler cluster. We evaluate two different block sizes (5MB and 10MB that are about 50 and 100 images per block), as our Euler cluster reads data in terms of 4MB+ blocks. The batch size is set to $512$ images, so each process performs SGD computation on $512/8 = 64$ images per batch. The buffer size of each process is 1.25\% of the whole dataset, thus the total buffer size of all processes is 10\% of the whole dataset. The number of data loading threads for each process is set to two, since we have twice as many CPU cores as GPUs. The learning rate is initialized as 0.1 and is decayed every 30 epochs.

Figure~\ref{fig:imagenet-results} illustrates the end-to-end execution time of ResNet50 model on the large \texttt{ImageNet} dataset, using different shuffling strategies. We report both the Top 1 and Top 5 accuracy. From Figure~\ref{fig:imagenet-results}(a) and \ref{fig:imagenet-results}(b), we can observe that \sys is 1.5$\times$ faster than \emph{Shuffle Once} to converge and the converged accuracy of \sys is similar to that of \emph{Shuffle Once}. The main reason of the slowness of \emph{Shuffle Once} is that it needs about 8.5 hours to shuffle the large ($\sim$150 GB) \texttt{ImageNet} dataset and store the shuffled dataset in our Euler cluster. In contrast, \sys eliminates this long data shuffling time. The second reason is that our \sys has limited per-epoch overhead. Although \sys has block shuffle and tuple shuffle overhead, the per-epoch time of \sys with 5MB or 10MB block is only $\sim$15\% longer than that of the fastest \emph{No Shuffle} baseline. The reason is that \sys reads data in terms of blocks which is comparable to sequential read on block-based parallel file system.

For the convergence rate comparison as shown in  Figure~\ref{fig:imagenet-results}(c) and ~\ref{fig:imagenet-results}(d), we can see that the convergence rates of both \sys with 5 MB block and \sys with 10 MB block are comparable to that of \emph{Shuffle Once}. Although \sys with 10 MB block has a bit lower convergence rate than \emph{Shuffle Once} in the first 30 epochs, it can catch up with \emph{Shuffle Once} in the following epochs and converges to the similar accuracy as \emph{Shuffle Once}. In contrast, the converged accuracy of \emph{No Shuffle} is close to $0\%$.

\begin{figure*}[t]
\centering
\includegraphics[width=0.99\textwidth]{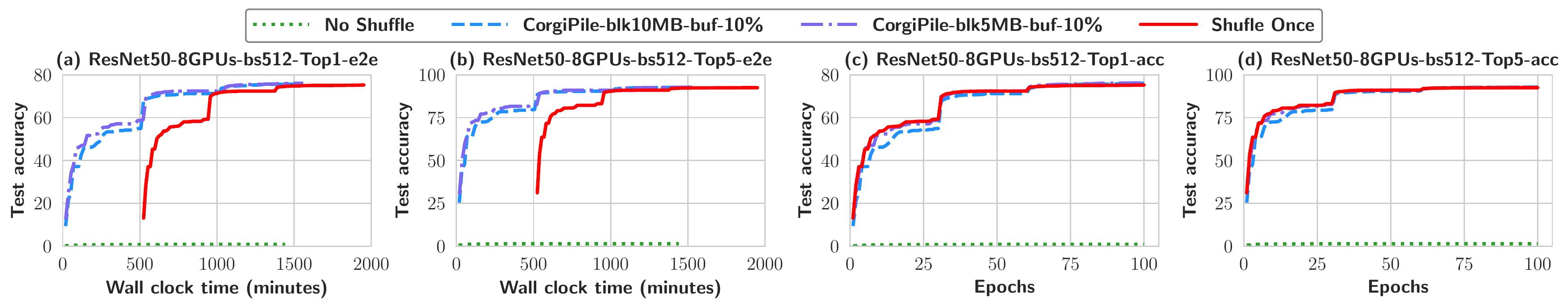}
\caption{The convergence rates of ResNet50 with different data shuffling strategies, for the clustered ImageNet dataset. TopN refers to the Top-N accuracy.}
\label{fig:imagenet-results}
\end{figure*}

\begin{figure*}[t]
\centering
\includegraphics[width=0.99\textwidth]{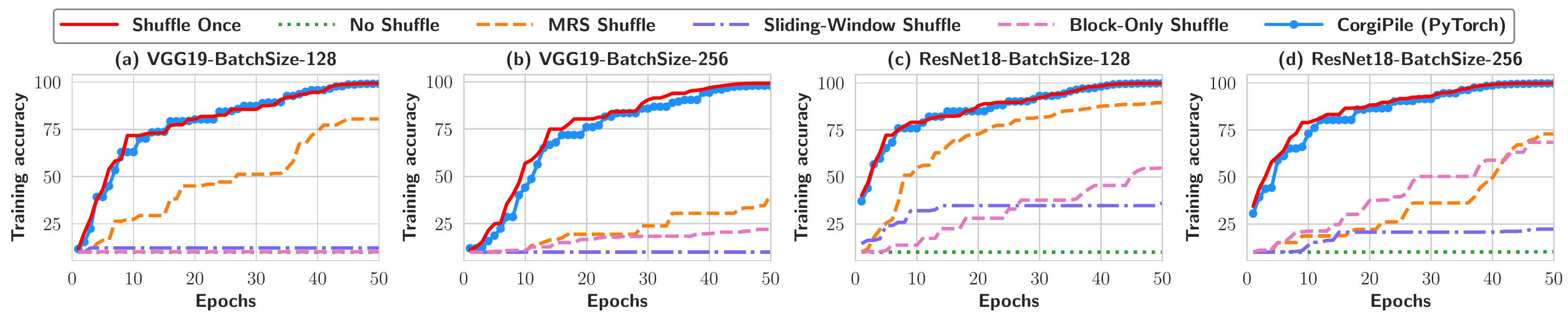}
\caption{The convergence rates of deep learning models with different data shuffling strategies and batch sizes, for the clustered 10-class \T{cifar-10} image dataset.}
\label{fig:cifar-10-results}
\end{figure*}

\begin{figure*}[th]
\centering
\includegraphics[width=0.99\textwidth]{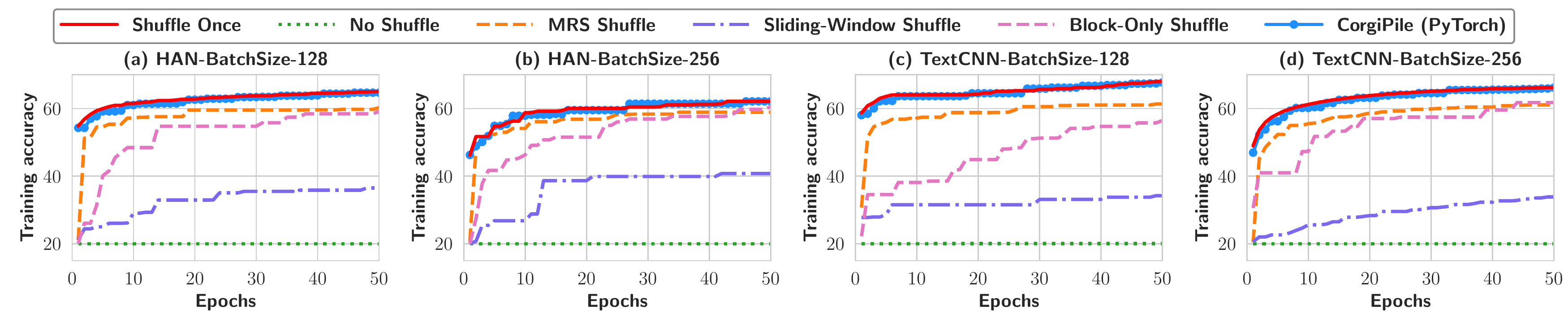}
\caption{The convergence rates of deep learning models with different data shuffling strategies and batch sizes, for the clustered 5-class \T{yelp-review-full} text dataset.}
\label{fig:yelp-nlp-results}
\end{figure*}

\begin{figure*}[th]
\centering
\includegraphics[width=0.99\textwidth]{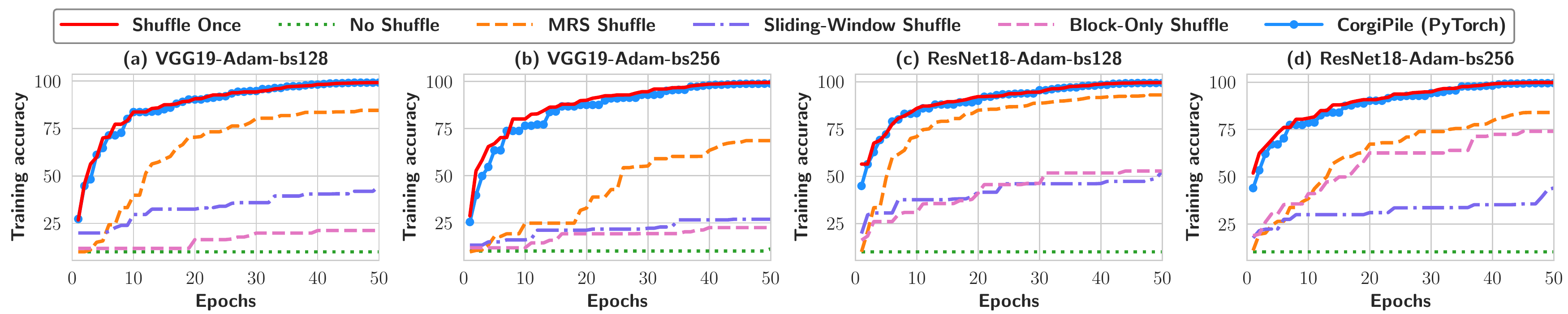}
\caption{The convergence rates of deep learning models with different data shuffling strategies and batch sizes, using Adam instead of SGD for the clustered cifar-10 image dataset.}
\label{fig:cifar-adam}
\end{figure*}
\subsubsection{Convergence rate comparison} 
We perform PyTorch with \sys and other strategies on \T{cifar-10} image dataset and \T{yelp-review-full} dataset using a single GPU. The \T{cifar-10} dataset contains 50,000 training images in 10 classes, while the \T{yelp-review-full} dataset has 650,000 reviews (text messages) in 5 classes.

For the image classification, Figure~\ref{fig:cifar-10-results} illustrates the convergence rates of VGG19 and ResNet18 models 
on the \emph{clustered} \T{cifar-10} dataset, with different shuffling strategies and different batch sizes (128 and 256).
The buffer size is $10\%$ of the whole dataset and the block size is set to 100 images per block.
This figure shows that \sys achieves comparable convergence rate and accuracy to the \emph{Shuffle Once} baseline, whereas other strategies suffer from lower accuracy due to the partially random order of the shuffled tuples. Specifically, the \emph{Sliding-Window Shuffle} used by TensorFlow only performs better than \emph{No Shuffle}, and suffers from large (50\%+) accuracy gap with \emph{Shuffle Once} and \sys.

For the text classification, Figure~\ref{fig:yelp-nlp-results} shows the convergence results of classical HAN \cite{HAN} and TextCNN \cite{textcnn} models with pre-trained word embeddings \cite{glove} on the clustered \T{yelp-review-full} dataset \cite{yelp} with different batch sizes (128 and 256). The buffer size is still $10\%$ of the whole dataset and the block size is set to 1,000 reviews per block. Again, \sys achieves similar convergence rate and accuracy to the \emph{Shuffle Once} baseline, whereas other strategies converge to lower accuracy. Specifically, \emph{No Shuffle} only achieves about $20\%$ accuracy for the two models, and the \emph{Sliding-Window Shuffle} used by TensorFlow only achieves about $40\%$ accuracy. \emph{MRS Shuffle} converges faster and better than them but still suffers from lower accuracy than the \emph{Shuffle Once} baseline.

The above results indicate that \sys can achieve both good statistical efficiency and hardware efficiency for deep learning models on non-convex optimization problems. When integrated to PyTorch, \sys is 1.5$\times$ faster than the \emph{Shuffle Once} baseline on the large \texttt{ImageNet} dataset in our experiments.

\subsubsection{Beyond SGD optimizer}\label{adam-results}

Although our work focuses on the most popular SGD optimizer,
we are confident that \sys can also be used in other optimizers, such as more complex first-order optimizers like Adam~\cite{Adam-paper}.

Here, we further perform VGG19 and ResNet18 models on the clustered \texttt{cifar-10} dataset using Adam instead of SGD. Figure~\ref{fig:cifar-adam} shows the convergence results with different batch sizes (128 and 256). The result of the convergence rate comparison is similar to that of SGD in Figure~\ref{fig:cifar-10-results}. Our \sys still achieves comparable convergence rate and accuracy to the best \emph{Shuffle Once} baseline, whereas other shuffling strategies suffer from lower accuracy.

\subsection{Evaluation on SGD with In-DB ML Systems} \label{SGD-LR-SVM}

\begin{figure*}[t]
\centering
\includegraphics[width=0.99\textwidth]{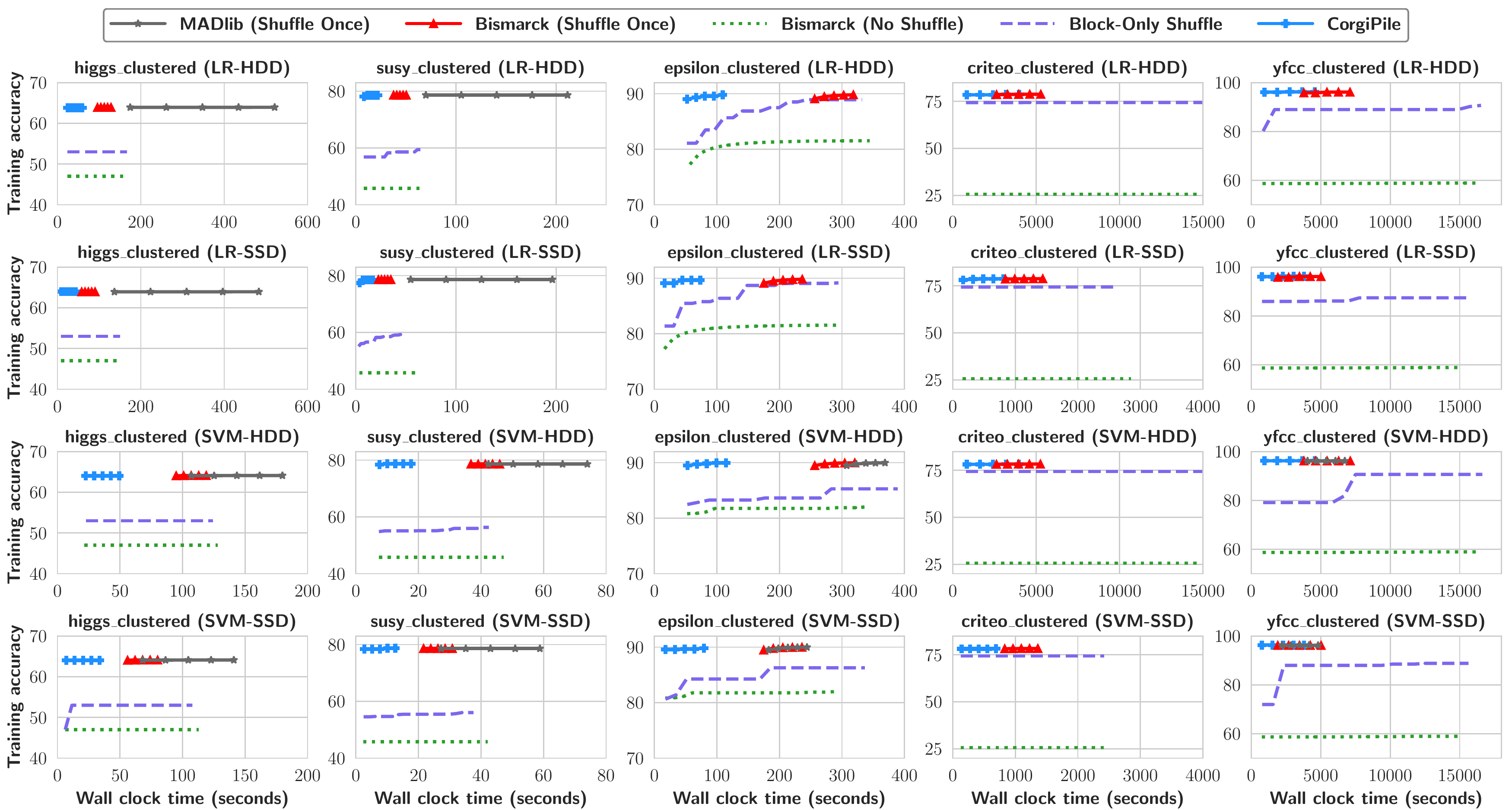}
\caption{The end-to-end execution time of SGD with different data shuffling strategies in PostgreSQL, for clustered datasets on HDD and SSD. We only show the first 5 epochs for \emph{Shuffle Once} and \sys, since they converge in 1-3 epochs.}
\label{convergence-results}
\end{figure*}

For in-DB ML, we first evaluate \sys in terms of the end-to-end execution time. The compared systems include \emph{No Shuffle} and \emph{Shuffle Once} strategies in MADlib and Bismarck,  
as well as a simpler version of our \sys named \emph{Block-Only Shuffle}, to see how \sys behaves without tuple-level shuffle.  
We then analyse the convergence rates, in comparison with other strategies, including \emph{MRS Shuffle} and \emph{Sliding-Window Shuffle}.
We finally study the overhead of \sys by comparing the per-epoch execution time of \sys with the fastest \emph{No Shuffle} baseline. 

In the following, we set the buffer size to $10\%$ of the whole dataset and 
block size to 10 MB for all methods. We report a sensitivity analysis 
on the impact of buffer sizes and block sizes in Section~\ref{sec:sensitivity}.

\subsubsection{End-to-end Execution Time}

Figure~\ref{convergence-results} presents the end-to-end execution time of SGD for in-DB ML systems, for \emph{clustered} datasets on both HDD and SSD. The end-to-end execution time includes: (1) the time for shuffling the data, i.e., \emph{Shuffle Once} needs to perform a full data shuffle before SGD starts running;\footnote{Therefore, \emph{Shuffle Once} in MADlib and Bismarck starts later than the others.} (2) the data caching time, i.e., the time spent on loading data from disk to the OS cache during the first epoch;\footnote{This is determined by the I/O bandwidth. Since SSD has higher I/O performance than HDD, the GLMs' first epoch on SSD starts earlier than that on HDD.}
and (3) the execution time of all epochs.

From Figure~\ref{convergence-results}, we can observe that \sys converges the fastest among all systems, and simultaneously achieves comparable converged accuracy to the best \emph{Shuffle Once} baseline,
usually within 1-3 epochs because of the large number of data tuples. 
Compared to \emph{Shuffle Once} in MADlib and Bismarck, \sys converges 2.9$\times$-12.8$\times$ faster than MADlib and 2$\times$-4.7$\times$ faster than Bismarck, on HDD and SSD. This is due to the eliminated data shuffling time. 
For example, for the clustered \T{yfcc} dataset on HDD, \sys can converge in 16 minutes, whereas \emph{Shuffle Once} in Bismarck needs 50 minutes to shuffle the dataset and another 15 minutes to 
execute the first epoch (to converge). 
That is, when \sys converges, \emph{Shuffle Once} is still performing data shuffling. For other datasets like \T{criteo} and \T{epsilon}, 
similar observations hold.
Moreover, data shuffling using \texttt{ORDER BY RANDOM()} in PostgreSQL, as implemented by \emph{Shuffle Once} in MADlib/Bismarck,
requires 2$\times$ disk space to generate and store the shuffled data. Therefore, \sys is both more efficient and requires less space.

MADlib is slower than Bismarck given that it performs more computation on some auxiliary statistical metrics and has less efficient implementation~\cite{DB4ML}. 
Moreover, for high-dimensional dense datasets, such as \T{epsilon} and \T{yfcc}, MADlib LR cannot finish even a single epoch within 4 hours,
due to some expensive matrix computations on a metric named \emph{stderr}.\footnote{We have confirmed this behavior with the MADlib developers.} 
MADlib's SVM implementation does not have this
problem and can finish its execution on high-dimensional dense datasets.
In addition, MADlib currently does not support training LR/SVM on sparse datasets such as \T{criteo} dataset.

\subsubsection{Convergence rate comparison} \label{converge-rate}

\begin{table}[t]
\centering
\caption{The final training and testing accuracy of \emph{Shuffle Once} (SO) and \sys.}
\label{table:acc}
{\small
\begin{tabular}{l|c| c|c|c}
 & \multicolumn{2}{c|}{LR (SO $|$ \sys)} & \multicolumn{2}{c}{SVM (SO  $|$ \sys)}   \\
\cmidrule(lr){2-5}
 Dataset & Train acc. (\%) & Test acc. (\%) & Train acc. (\%) & Test acc. (\%) \\
\toprule
higgs & 64.04  $|$ 64.07 & 64.04  $|$ 64.06 & 64.11  $|$ 64.22 & 63.93  $|$ 63.95 \\
susy & 78.61  $|$ 78.54 & 78.69  $|$ 78.66 & 78.61  $|$ 78.66 & 78.73  $|$ 78.66 \\
epsilon & 90.02  $|$ 90.01 & 89.77  $|$ 89.74 & 90.12  $|$ 90.11 & 89.81  $|$ 89.80  \\
\toprule
criteo & 78.97  $|$ 78.91 & 78.77  $|$ 78.69 & 78.31  $|$ 78.41 &  78.45  $|$ 78.44  \\
yfcc & 96.43  $|$ 96.38 & 96.14  $|$ 96.11 & 96.35  $|$ 96.31 & 96.23  $|$ 96.20 \\

\end{tabular}
}
\end{table}

\begin{figure*}[t]
\centering
\includegraphics[width=0.99\textwidth]{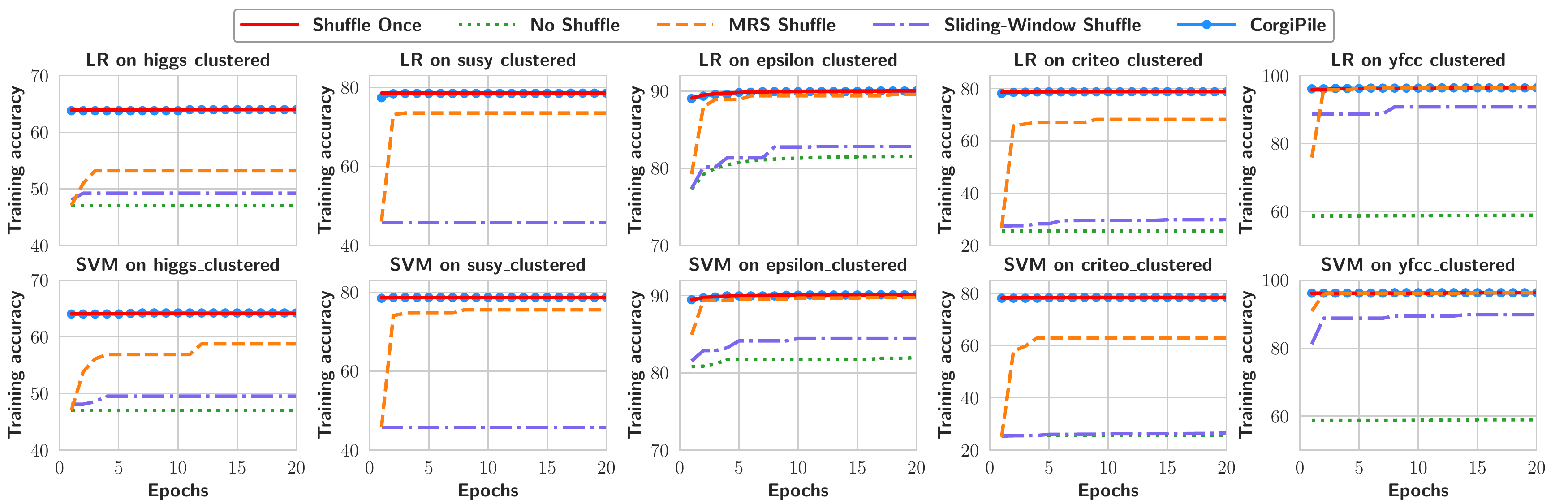}
\caption{The convergence rates of LR and SVM with different shuffling strategies for clustered datasets.}
\label{fig:convergence-rates}
\end{figure*}

For all datasets inspected, the gap between \emph{Shuffle Once} and \sys is below 1\%
for the final training/testing accuracy, as shown in Table~\ref{table:acc}. 
We attribute this to the fact that \sys can yield good data randomness in each epoch of SGD (Section~\ref{Convergence_analysis}).
\emph{No Shuffle} results in the lowest accuracy when SGD converges, as illustrated in Figure~\ref{convergence-results}.
The \emph{Block-Only Shuffle} baseline, where we simply omit tuple-level shuffle in \sys, can achieve higher accuracy than \emph{No Shuffle} but lower accuracy than \emph{Shuffle Once}. The reason is that \emph{Block-Only Shuffle} can only yield a \emph{partially} random order, and the tuples in each block can all be negative or positive for the clustered data.

Since \emph{MRS Shuffle} and \emph{Sliding-Window Shuffle} are not available in the current MADlib/Bismarck, we use our own implementations (in PyTorch)
and compare their convergence rates. Figure~\ref{fig:convergence-rates} shows the convergence rates of all strategies, where \emph{Sliding-Window}, \emph{MRS}, and \sys all use the same buffer size (10\% of the whole dataset). As shown in Figure~\ref{fig:convergence-rates}, \emph{Sliding-Window Shuffle} suffers from lower accuracy, whereas \emph{MRS Shuffle} only achieves comparable accuracy to \emph{Shuffle Once} on \T{yfcc} but suffers on the other datasets.

\begin{figure*}[t]
\centering
\includegraphics[width=0.99\textwidth]{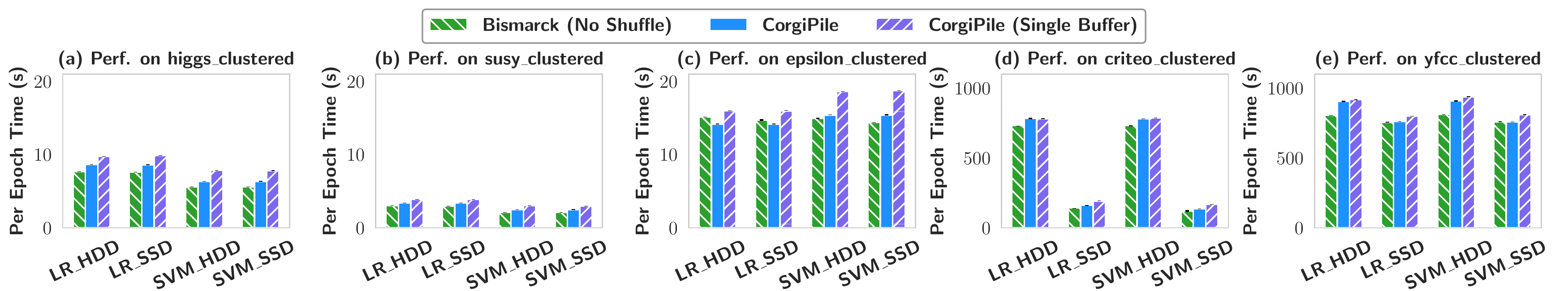}
\caption{The average per-epoch time of SGD with Bismarck (\emph{No Shuffle}), \sys, and \sys with single buffer in PostgreSQL, for clustered datasets on HDD and SSD. It shows that \sys is up to 11.7\% slower than the fastest \emph{No Shuffle}.}
\label{fig:per-iter-perf-clustered-data}
\end{figure*}

\subsubsection{Per-epoch Overhead} \label{overhead}
To study the overhead of \sys, we compare its per-epoch execution time with the fastest \emph{No Shuffle} baseline, as well as the single-buffer version of \sys, as shown in
Figure~\ref{fig:per-iter-perf-clustered-data}.
We make the following three observations.

\begin{itemize}[leftmargin=*]
    \item For small datasets with in-memory I/O bandwidth, the average per-epoch time  of \sys is comparable to that of \emph{No Shuffle}. 
    \item For large datasets with disk I/O bandwidth, the average per-epoch time of \sys is up to $\sim$1.1$\times$ slower than that of \emph{No Shuffle}, i.e., it incurs at most an additional 11.7\% overhead, due to buffer copy and tuple shuffle. 
    \item By using double-buffering optimization, \sys can achieve up to 23.6\% shorter per-epoch execution time, compared to its single-buffering version.
\end{itemize}

The above results reveal that \sys with double-buffering optimization can introduce limited overhead (11.7\% longer per-epoch execution time), compared to the best \emph{No Shuffle} baseline.

\subsubsection{Sensitivity Analysis}
\label{sec:sensitivity}

We next study the effects of different buffer sizes, I/O bandwidths, and block sizes
for \sys.

\vspace{0.5em}
\noindent\textbf{The effects of buffer size.} 
Figure~\ref{fig:buffer-size-var} reports the convergence behavior 
of \sys on the two largest datasets with different buffer sizes:
1\%, 2\%, and 5\% of the dataset size. We see that \sys 
only requires a buffer size of 2\% to maintain the same convergence
behavior as \textit{Shuffle Once}. With a 1\% buffer, it
only converges slightly slower than \emph{Shuffle Once},
but achieves the same final accuracy. On the other hand,
as discussed in previous sections, \textit{Sliding-Window Shuffle}
and \textit{MRS Shuffle} achieve a much lower accuracy even when
given a much larger buffer (10\%).

\begin{figure*}[t!]
    \centering
    \begin{subfigure}[\sys's convergence with varying buffer sizes.]{
        \includegraphics[width=0.48\textwidth]{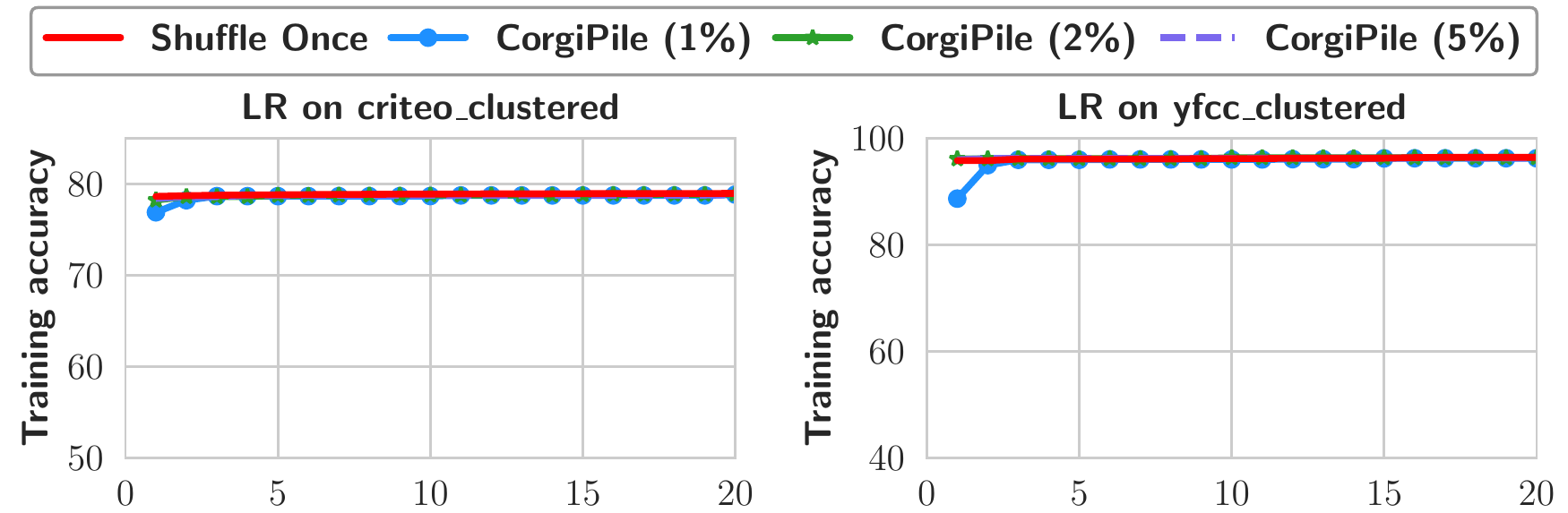}
      \label{fig:buffer-size-var}}
    \end{subfigure}
    \begin{subfigure}[Per-epoch time of \sys with varying block sizes.]{
        \includegraphics[width=0.48\textwidth]{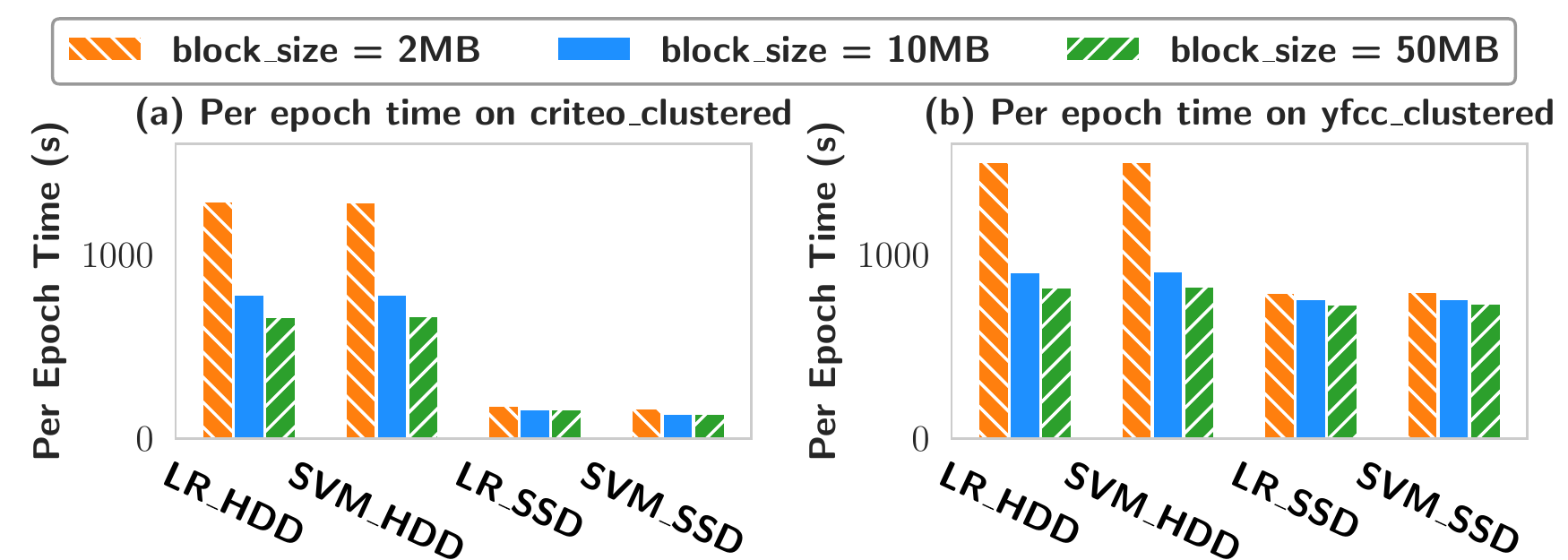}
        \label{fig:block-size-var}
    }
    \end{subfigure}
    \caption{The effects of buffer size and block size on \sys.}
    \label{all}
\end{figure*}

\vspace{0.5em}
\noindent\textbf{The effects of I/O bandwidth.} 
As shown in Figure~\ref{fig:per-iter-perf-clustered-data}, 
for smaller datasets such as \T{higgs}, \T{susy}, and \T{epsilon}, \sys on HDD and \sys on SSD achieve the similar per-epoch times, since these datasets have been cached in memory after the first epoch. For larger datasets such as \T{criteo},
\sys is faster on SSD than HDD, as expected.
Interestingly, for \T{yfcc}, \sys achieves similar performance on both HDD and SSD. The reason is that the TOAST compression on \T{yfcc} slows down data loading to only $\sim$130 MB/s on both SSD and HDD, whereas it achieves $\sim$700/130 MB/s on SSD/HDD for \T{criteo} without this compression. The same observation holds for \emph{No Shuffle} on HDD/SSD for these datasets.

\vspace{0.5em}
\noindent\textbf{The effects of block size.}
We vary the block size in $\{$2MB, 10MB, 50MB$\}$ on the large \T{criteo} and \T{yfcc} datasets. Figure~\ref{fig:block-size-var} shows that the per-epoch time decreases as the block size increases from 2MB to 50MB, due to the higher I/O bandwidth (throughput). However, the time difference between 10MB and 50MB is limited (under 10\%), because using 10MB has achieved the highest possible I/O bandwidth (130 MB/s on HDD). In practice, we recommend users to choose the smallest block size that can achieve high-enough I/O throughput, using I/O test commands such as ``fio'' in Linux.

\subsubsection{Performance comparison with PyTorch}
\begin{figure}[t]
\centering
\includegraphics[width=0.5\columnwidth]{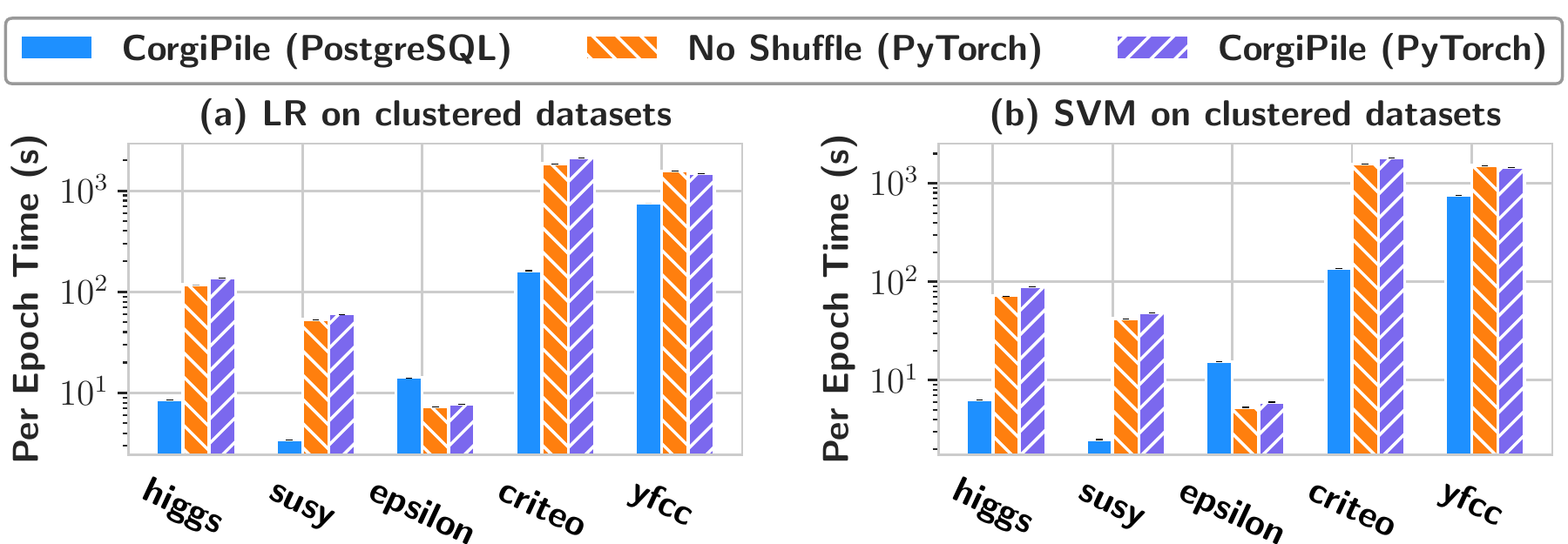}
\caption{The per-epoch time comparison between in-DB \sys and PyTorch on SSD. }
\label{pytorch-vs-corgipile}
\end{figure}

To further understand the performance gap between our in-DB \sys and the start-of-the-art PyTorch outside DB, we compare them in two ways.

(1) \sys in PostgreSQL vs. PyTorch:
Figure~\ref{pytorch-vs-corgipile} shows the per-epoch time comparison between \sys in PostgreSQL and PyTorch with \emph{No Shuffle}. For PyTorch, we load small datasets into memory before training to reduce the I/O overhead, and store the large \T{criteo} and \T{yfcc} datasets on disk. The comparison results in Figure~\ref{pytorch-vs-corgipile} show that our in-DB \sys is 2-16$\times$ faster than
 PyTorch on \T{higgs}, \T{susy}, \T{criteo}, and \T{yfcc} datasets. We speculate that this is because PyTorch has high overhead of Python-C++ invocations of forward/backward/update functions
for each tuple, and these datasets have a large number (3-92 millions) of tuples. Only for the \T{epsilon} dataset, PyTorch is 2-3$\times$ faster than \sys. The reason is that this dataset is compressed in DB by \emph{TOAST}~\cite{TOAST}. \sys needs to decompress each tuple, while PyTorch directly computes on the in-memory uncompressed data.

(2) Outside DB: Figure~\ref{pytorch-vs-corgipile} shows that PyTorch with \sys introduces small (up to 16\%) overhead
compared to PyTorch with \emph{No Shuffle}, which is consistent with
what we observed in DB.

\subsection{Evaluation on Mini-Batch SGD and other types of datasets with In-DB ML Systems} \label{Mini-batch-SGD}

In the previous experiments, we focus on the  standard SGD algorithm, which updates the model per tuple. Since it is also common to use mini-batch SGD, we implement mini-batch SGD for \emph{CorgiPile},  \emph{Once Shuffle}, \emph{No Shuffle}, and \emph{Block-Only Shuffle}, using our in-DB operators in PostgreSQL. Since MADlib and Bismarck currently do not support mini-batch SGD for linear models, we compare these shuffling strategies based on our PostgreSQL implementations.

\subsubsection{Mini-batch LR and SVM models}

\begin{figure*}[t]
\centering
\includegraphics[width=0.99\textwidth]{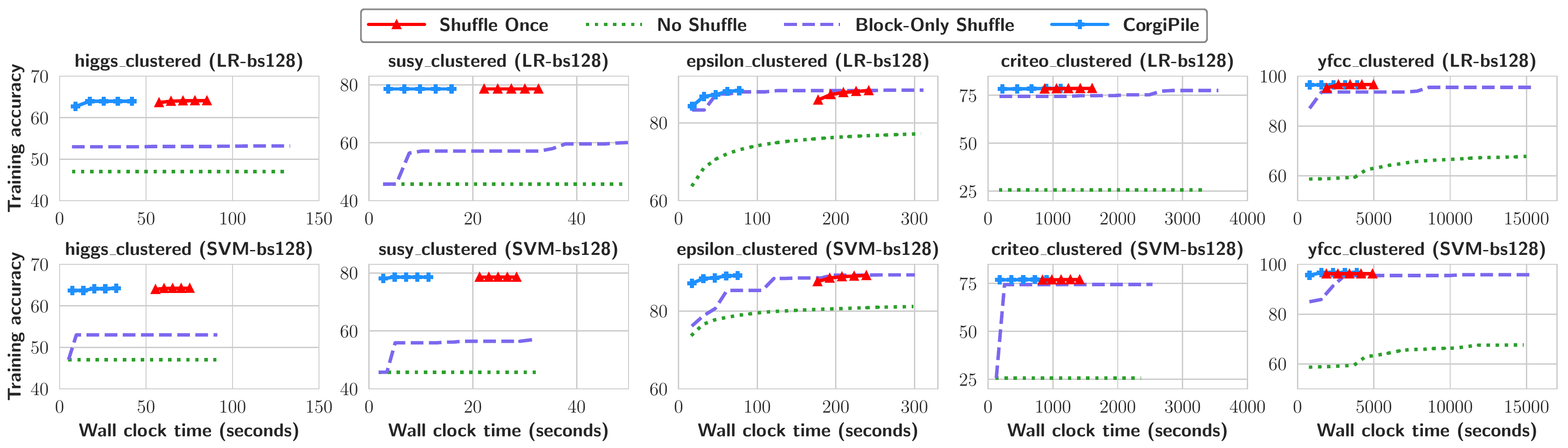}
\caption{The end-to-end execution time of LR and SVM using mini-batch SGD (\emph{batch\_size} = 128) in PostgreSQL, for clustered datasets on SSD.}
\label{lr-svm-mini-batch-sgd}
\end{figure*}

\begin{figure*}[t]
\centering
\includegraphics[width=0.99\textwidth]{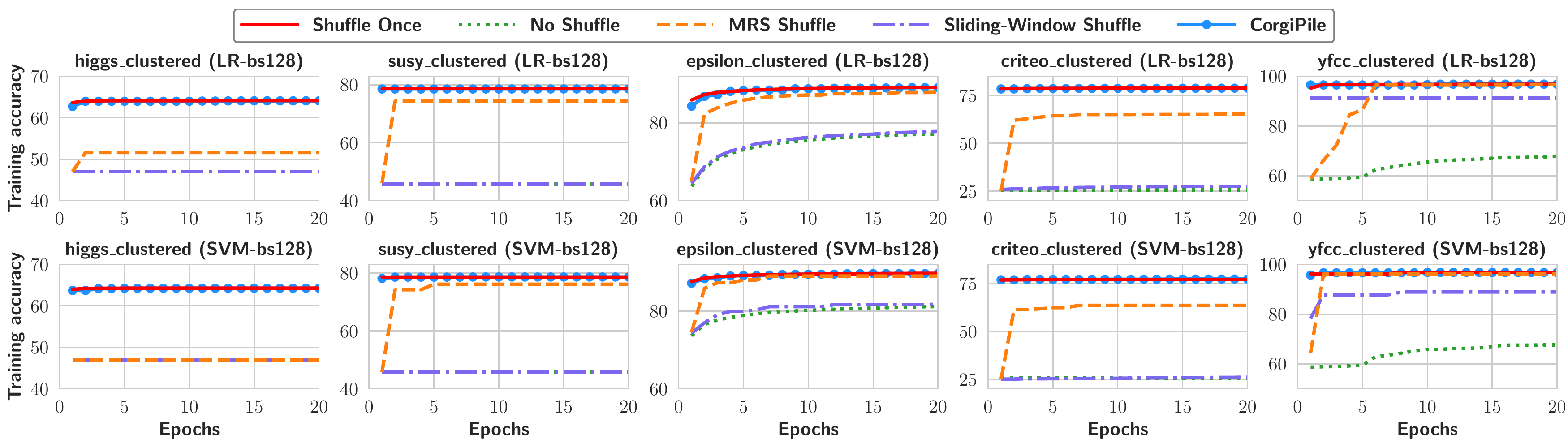}
\caption{The convergence rates of LR and SVM using mini-batch SGD (\emph{batch\_size} = 128), for clustered datasets.} 
\label{lr-svm-convengence-mini-batch-sgd-128}
\end{figure*}
We first
perform LR and SVM using mini-batch SGD on the clustered datasets. Figure~\ref{lr-svm-mini-batch-sgd} illustrates the end-to-end execution time of these two models in PostgreSQL on SSD. The result is similar to that of the standard SGD. Our \sys achieves comparable convergence rate and accuracy to \emph{Shuffle Once} but 1.7-3.3$\times$ faster than it to converge. Other strategies like \emph{No Shuffle} and \emph{Block-Only Shuffle} suffer from either lower converged accuracy or lower convergence rate. 

Figure~\ref{lr-svm-convengence-mini-batch-sgd-128} demonstrates the convergence rates of different shuffling strategies with \emph{batch\_size} =  128. From this figure, we can see that \emph{Sliding-Window Shuffle} and \emph{MRS Shuffle} have convergence rate or accuracy gap with \sys and \emph{Shuffle Once}, for clustered datasets. Only for \T{yfcc} dataset, \emph{MRS Shuffle} can converge to the similar accuracy to \sys and \emph{Shuffle Once}, but \emph{MRS Shuffle} has slower convergence rate.

\subsubsection{Linear regression and Softmax regression models}

\begin{figure}[t]
\centering
\includegraphics[width=0.99\textwidth]{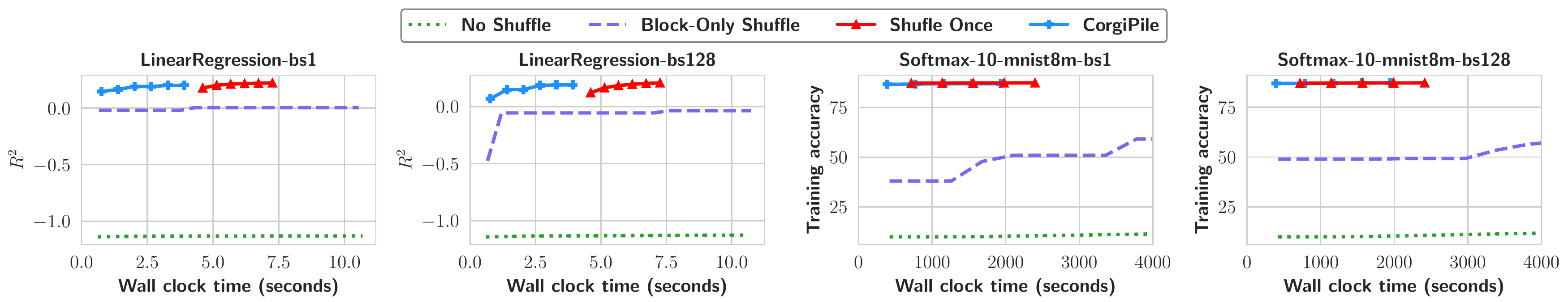}
\caption{The end-to-end time of linear and softmax regression in PostgreSQL, for clustered datasets on SSD. }
\label{linear-softmax-mini-batch-sgd}
\end{figure}
Apart from LR/SVM on binary-class datasets, users may also want to train ML models on continuous and multi-class datasets in DB. Thus, we further implement \textit{linear regression} for training continuous dataset and \textit{softmax regression} (i.e., multinomial logistic regression) for multi-class datasets, based on our in-DB operators in PostgreSQL.  Figure~\ref{linear-softmax-mini-batch-sgd} shows the end-to-end execution time of linear regression for continuous \T{YearPredictionMSD} dataset \cite{LIBSVM-dataset} and softmax regression for 10-class \T{mini8m} dataset \cite{LIBSVM-dataset}, with different batch sizes on SSD.
Our \sys again achieves similar convergence rate and accuracy (i.e., coefficient of determination $R^2$ for linear regression) with the best \emph{Shuffle Once}, but 1.6-2.1$\times$ faster than it to converge.

\subsubsection{Beyond label-clustered datasets}

We conduct additional experiments using LR and SVM on all the binary-class datasets ordered by features instead of the labels. For low-dimensional \T{higgs} and \T{susy}, we sort each feature of them and report the statistics of the converged accuracy in Figure~\ref{feature-order}. For the other three high-dimensional datasets, we select 9 features such that 3/3/3 of them have the highest/lowest/median correlations with the labels. 

Figure~\ref{feature-order} shows that \emph{No Shuffle} again leads to lower accuracy than \emph{Shuffle Once}. Only for \T{yfcc} with image-extracted features and \T{epsilon} (with unknown features~\cite{epsilon}), the accuracy gap is limited. In contrast, \sys achieves similar converged accuracy to \emph{Shuffle Once} on all the datasets. This implies simply scanning does not work on the datasets clustered by labels or by features.

\begin{figure}[t]
\centering
\includegraphics[width=0.5\textwidth]{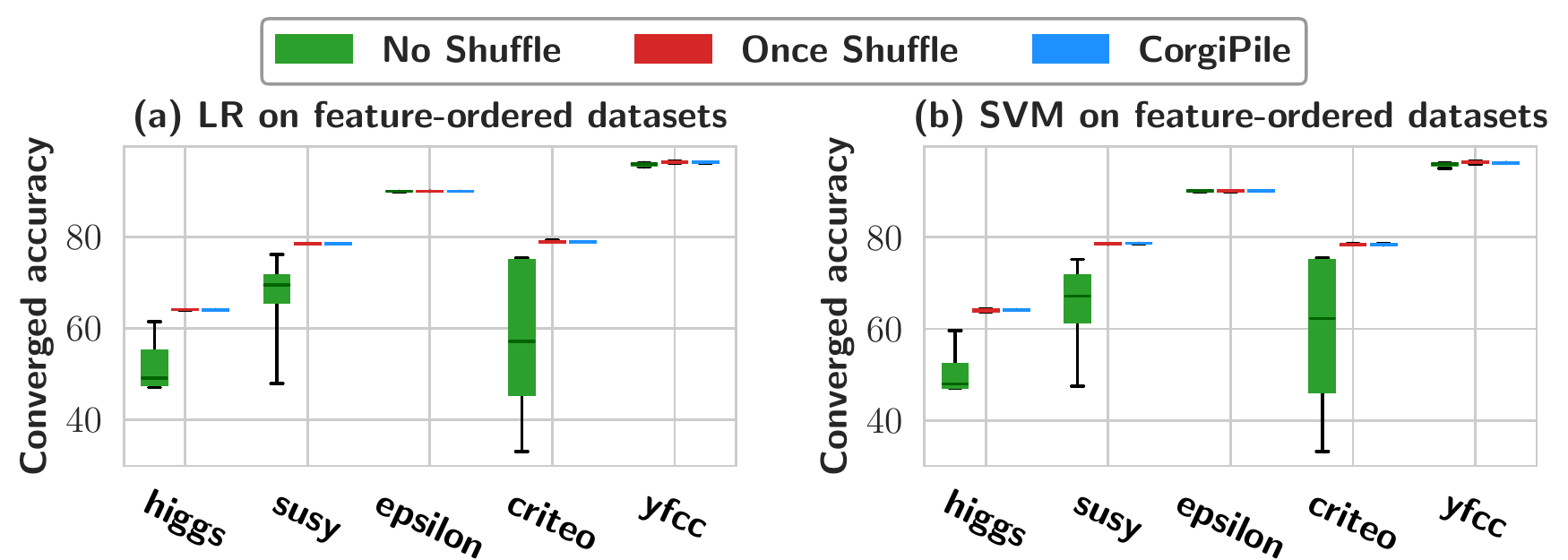}
\caption{The converged accuracy of LR and SVM on the datasets ordered by features instead of the label. }
\label{feature-order}
\end{figure}

%% file: s6-related-work.tex
\section{Related Work}
\label{sec:related}

\paragraph*{Stochastic gradient descent (SGD)} SGD is broadly used in machine learning to solve large-scale optimization problems~\cite{bottou2010large}. 
It admits the convergence rate $O(1/T)$ for strongly convex objectives, and $O(1/\sqrt{T})$ for the general convex case~\cite{moulines2011non, ghadimi2013stochastic}, where $T$ refers to the number of iterations. 
For non-convex optimization problems, an ergodic convergence rate $O(1/\sqrt{T})$ is proved in~\cite{ghadimi2013stochastic}, and the convergence rate is $O(1/T)$ (e.g.,~\cite{haochen2018random}) under the Polyak-\L ojasiewicz condition~\cite{polyak1963gradient}. In the analysis of the above cases, the common assumption is that data is sampled uniformly and independently \emph{with replacement} in each epoch. We call SGD methods based on this assumption as \emph{vanilla SGD}.

\paragraph*{Data shuffling strategies for SGD}  In practice, random-shuffle SGD is a more practical and efficient way of implementing SGD~\cite{bottou2012stochastic}. In each epoch, the data is reshuffled and iterated one by one \emph{without replacement}. Empirically, it can also be observed that random-shuffle SGD converges much faster than vanilla SGD~\cite{bottou2009curiously,gurbuzbalaban2019random,haochen2018random}. 
In Section~\ref{sec:benchmark}, we empirically studied the state-of-the-art data shuffling strategies for SGD, including \emph{Epoch Shuffle}, \emph{No Shuffle}, \emph{Shuffle Once}, \emph{Sliding-Window Shuffle}~\cite{sliding-window} and \emph{MRS Shuffle}~\cite{Feng:2012:TUA:2213836.2213874}. Our empirical study shows that \emph{Shuffle Once} achieves good convergence rate but suffers from low performance, whereas other strategies suffer from low accuracy when running on top of \emph{clustered} data.

\paragraph*{In-DB ML} Previous work~\cite{zhang2014dimmwitted, madlib-paper, Feng:2012:TUA:2213836.2213874, DBLP:conf/sigmod/SchleichOC16, learning-over-join, F-system, linear-algebra, Scale-factorization-ml, DBLP:conf/pods/Khamis0NOS18, scalable-linear-algebra,  Jankov2021distributed,luo2021automatic,yuan2021tensor,kara2018columnml,SAP, ORE} has intensively discussed how to implement ML models on relational data, such as linear models~\cite{DBLP:conf/sigmod/SchleichOC16, learning-over-join, F-system}, linear algebra~\cite{linear-algebra, scalable-linear-algebra, luo2021automatic}, factorization models~\cite{Scale-factorization-ml}, neural networks~\cite{Jankov2021distributed,luo2021automatic,yuan2021tensor} and other statistical learning models~\cite{DBLP:conf/pods/Khamis0NOS18}, using Batch Gradient Descent (BGD) or SGD, over join or self-defined matrix/tensors, etc. The most common way of integrating ML algorithm into RDBMS is to use User-Defined Aggregate Functions (UDA). The representative in-DB ML tools are Apache MADlib~\cite{madlib-paper, madlib} and Bismarck~\cite{Feng:2012:TUA:2213836.2213874}, which use PostgreSQL's UDAs to implement SGD, and leverage SQL LOOP (Bismarck) or Python driver (MADlib) to implement iterations. 
Recently, DB4ML~\cite{DB4ML} proposes another approach called \emph{iterative transactions} to implement iterative SGD/graph algorithm in DB.  However, it still uses/assumes the \emph{Shuffle Once} strategy as that of Bismarck/MADlib. Since the source code of DB4ML has not been released yet, we only compare with MADlib and Bismarck.

\paragraph*{Scalable ML for distributed data systems} In recent years, there has been active research on integrating ML models into distributed database systems to enable scalable ML, such as MADlib on Greenplum~\cite{Greenplum-madlib}, Vertica-ML~\cite{Vertica-ML}, Google’s BigQuery ML~\cite{BigQuery}, Microsoft SQL Server ML Services~\cite{SQL-Server-ML}, etc. Another trend is to leverage big data systems to build scalable ML models based on different architectures, e.g.,
MPI~\cite{xgboost,lightgbm}, MapReduce~\cite{MLlib,MLlib-star,SimSQL}, Parameter Server~\cite{distbelief,petuum,heterops} and decentralization~\cite{decentra_nips17,D2}.
Recent work also started discussing how to integrate deep learning into databases~\cite{Study-on-dl-on-db, Cerebro}. Our \sys is a general data shuffling strategy for SGD and has been integrated into PyTorch and PostgreSQL. We believe that \sys can be potentially integrated into more above distributed data systems.

%% file: s7-conclusion.tex
\section{Conclusion}
\label{sec:conclusion}

We have presented \sys, a simple but novel data shuffling strategy for efficient SGD computation on top of block-addressable secondary storage systems such as HDD and SSD.
\sys adopts a two-level (i.e., block-level and tuple-level) hierarchical shuffle mechanism that avoids the computation and storage overhead of full data shuffling while retaining similar convergence rates of SGD as if a full data shuffle were performed.
We provide a rigorous theoretical analysis on the
convergence behavior of \sys and further integrate it into both PyTorch and PostgreSQL.
Experimental evaluations demonstrate both statistical and hardware efficiency of \sys when compared to state-of-the-art deep learning system as well as the in-DB ML systems on top of PostgreSQL.

%% file: s8-appendix.tex
\section{I/O performance on HDD and SSD} \label{IOTestOnHDDandSSD}

We have performed an I/O test on the HDD and SSD with different block sizes.
As illustrated in Figure~\ref{IOTest}, on most modern devices, randomly accessing small data tuple can be significantly slower than sequentially scanning data tuples (shown as the dash lines). However, as
the block size grows to a reasonable size (e.g., 10MB in this case), the performance
of randomly accessing blocks matches the performance
of sequential scan. In other words, random access block-wise can match the speed of a full shuffle while the order
of data tuples stay untouched.

\begin{figure}[h]
\centering
\includegraphics[width=0.5\textwidth]{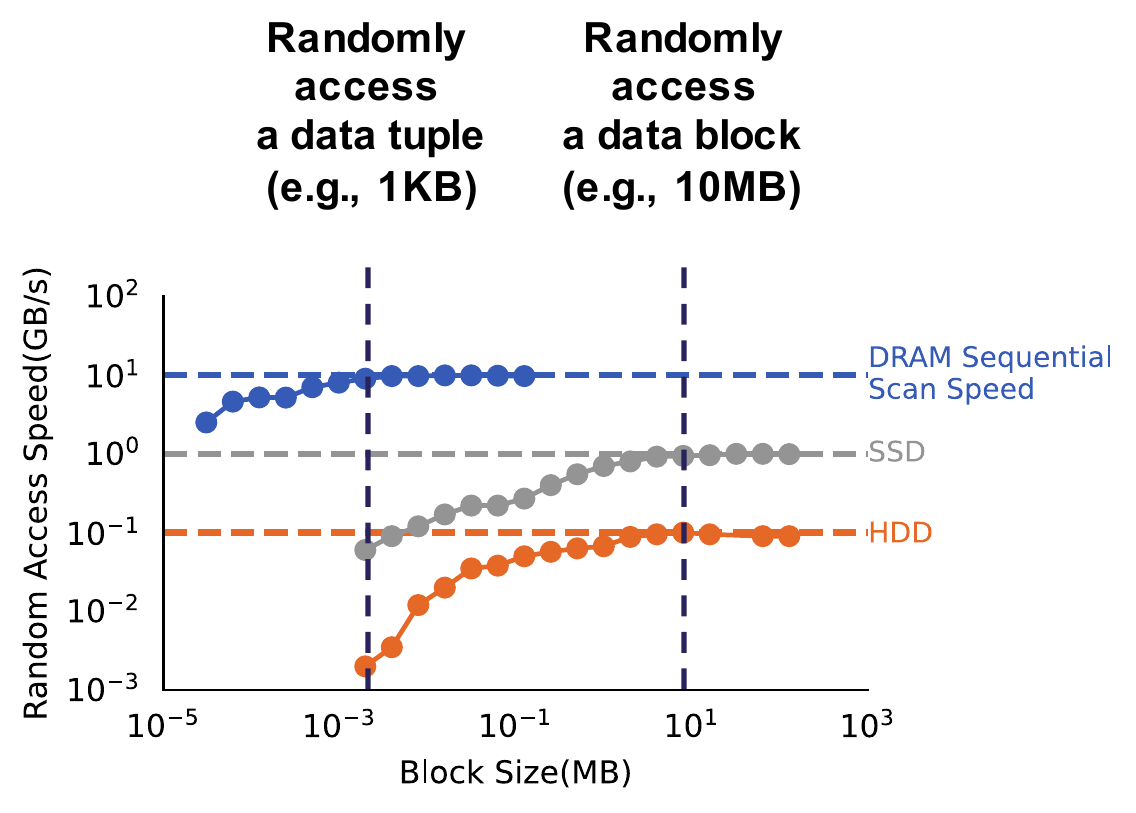}
\caption{Random Access Performance vs. Block Size. }
\label{IOTest}
\end{figure}

\section{Resource Usage in PostgreSQL}

\sys has higher CPU consumption than \emph{No Shuffle} and \emph{Shuffle Once}, because it performs data loading and the SGD computation concurrently using two hyper-threads on the same physical CPU core. 
For example, for \T{criteo} on HDD, 
the maximum CPU usage of our \sys is 115\%, where the SGD-computation thread runs with 100\% CPU usage and the data-loading thread runs with 15\% CPU usage. 
\emph{MRS Shuffle}, \emph{Sliding-Window Shuffle}, and \sys have higher memory consumption than \emph{No Shuffle}, since they need to allocate buffers in memory. 
\emph{Shuffle Once} not only requires additional memory for data shuffling, but also
 requires 2$\times$ disk space to store the shuffled data.

%% file: proof.tex
\newpage
\onecolumn

\centerline{ {\LARGE \textbf{The Proof of the Theorems}} } \label{proof}
\renewcommand{\thesection}{\Alph{section}}
\setcounter{section}{0}

\section{Preliminaries}
Before presenting our theoretical analysis, we first show some preliminary definitions and lemmas which is important to our proofs.

\begin{lemma} \label{lemma:strongcvx}
Suppose  $f: \R^d \mapsto \R$ is a $\mu$-strongly convex function. Then $\forall \x,\y \in \R^d$, there is
\begin{align*}
F(\x) - F(\y) \geq \langle \nabla F(\x), \y-\x \rangle + \frac{\mu}{2} \|\x-\y\|^2
\end{align*} 
\end{lemma}

\begin{lemma} \label{lemma:smooth}
Suppose $f(\x)$ is an $L$-smooth convex function. Then $\forall \x, \x^* \in \R^d$ where $\x^*$ is one global optimum of $f(\x)$,  there is
\begin{align*}
\|\nabla f(\x)\|_2^2  \leq 2L\left( f(\x) - f(\x^*) \right) 
\end{align*} 
\end{lemma}

\begin{fact} \label{fact:grad_hess}
Suppose $f: \R^d \mapsto \R$ is twice continuously differentiable. $\nabla f(\x) \in \R^{d}$ and $H(\x) \in \R^{d\times d}$ denote its derivative and Hessian at the point $\x\in \R^d$. Then we have, $\forall \y, \z \in \R^d$, 
\begin{align*}
\nabla f(\z) - \nabla f(\y) =\int_{0}^{\|\z-\y \|} H \left( \y + \frac{\z-\y}{\|\z-\y\|} t \right )  \frac{\z-\y}{\|\z-\y\|} \ \  \mathrm{d} t.
\end{align*}
For simplification, we further define 
\begin{align*}
\int_{\y}^{\z} H(\x) \mathrm{d}\x := \int_{0}^{\|\z-\y \|} H \left( \y + \frac{\z-\y}{\|\z-\y\|} t \right )  \frac{\z-\y}{\|\z-\y\|} \ \  \mathrm{d} t.
\end{align*}
which thus lead to $\nabla f(\z) - \nabla f(\y)  = \int_{\y}^{\z} H(\x) \mathrm{d}\x $.
\end{fact}

\begin{lemma} \label{lemma:converge}
Suppose there are two non-negative sequences $\{a_{s}\}_{s=0}^{+\infty}$, $\{b_s\}_{s=0}^{+\infty}$ satisfying
\begin{align*}
A_0 \eta_s a_s \leq (1-\mu A_1 \eta_s) b_s - b_{s+1} + A_2 \eta_s^2  + A_3 \eta_s^3 + A_4 \eta_s^4
\end{align*}
where $\eta_s = \frac{3}{A_1 \mu(s+a)}$ with $a\geq 1$, $\mu > 0, A_0>0, A_1>0, A_2>0, A_3>0, A_4 > 0$ are constants.
Then we have,  for $S \geq 0$,
\begin{align*}
 \frac{\sum_{s=1}^S w_s  a_s}{\sum_{s=1}^S w_s} \leq  \frac{ 4A_1 \mu a^4  b_1}{3 A_0 S^4}   + \frac{12A_2}{A_0 A_1 \mu} \frac{\sum_{s=1}^S (s+a)^2}{S^4}  + \frac{36A_3}{A_0 A_1^2 \mu^2} \frac{\sum_{s=1}^S (s+a)}{S^4} + \frac{108A_4}{A_0 A_1^3 \mu^3} \frac{1}{S^3}
\end{align*} 
where we let $w_s = (s+a)^3$.
\end{lemma}

\begin{proof}

\begin{align*}
\frac{1-\mu A_1 \eta_s}{\eta_s} w_s = \left(\frac{1}{\eta_s} - \mu A_1 \right) w_s = \frac{A_1 \mu (s+a-3)(t+a)^3 }{3} \leq \frac{w_{s-1}}{\eta_{s-1}} = \frac{A_1 \mu (t+a-1)^4}{3} 
\end{align*}
where the inequality can be easily verified with the condition $a \geq 1$.

Thus we have
\begin{align*}
A_0 w_s  a_s \leq & \frac{1-\mu A_1 \eta_s}{\eta_s} w_s b_s - \frac{w_s}{\eta_s }b_{s+1} + A_2 w_s \eta_s  + A_3 w_s \eta_s^2 + A_4 \eta_s^3 w_s \\
\leq&  \frac{ w_{s-1}}{\eta_{s-1}}  b_s - \frac{w_s}{\eta_s }b_{s+1} + A_2 w_s \eta_s  + A_3 w_s \eta_s^2 + A_4 \eta_s^3 w_s 
\end{align*} 
Taking summation on both sides of the above inequality, we have
\begin{align*}
A_0 \sum_{s=1}^S w_s  a_s \leq &  \frac{ w_0}{\eta_0}  b_1  + A_2 \sum_{s=1}^S w_s \eta_s  + A_3 \sum_{s=1}^S w_s \eta_s^2 + A_4 \eta_s^3 w_s 
\\
\leq &  \frac{ w_0}{\eta_0}  b_1  + \frac{3A_2}{A_1 \mu} \sum_{s=1}^S (s+a)^2  + \frac{9A_3}{A_1^2 \mu^2} \sum_{s=1}^S (s+a) + \frac{27A_4}{A_1^3 \mu^3} S
\end{align*}

On the other hand, we also see that $\sum_{s=1}^S w_s = \sum_{s=1}^S (s+a)^3 \geq \sum_{s=1}^S s^3 \geq \frac{S^4}{4}$. Dividing both sides by $\sum_{s=1}^S w_s$, we can obtain
\begin{align*}
 \frac{\sum_{s=1}^S w_s  a_s}{\sum_{s=1}^S w_s} \leq  \frac{ 4A_1 \mu a^4  b_1}{3A_0 S^4}   + \frac{12A_2}{A_0 A_1 \mu} \frac{\sum_{s=1}^S (s+a)^2}{S^4}  + \frac{36A_3}{A_0 A_1^2 \mu^2} \frac{\sum_{s=1}^S (s+a)}{S^4} + \frac{108A_4}{A_0 A_1^3 \mu^3} \frac{1}{S^3}
\end{align*}

\end{proof}

The main structure of our proof is based on the work of \cite{haochen2018random}, 
which tries to theoretically analyze the full shuffle SGD. However, our proofs below are not a trivial extension of the existing work. The key ingredients in our proofs employing some new techniques are the improvement of estimating the upper bounds of $\calI_1$ and $\calI_4$ in the proof of Theorem \ref{thm:partialrs_sample} and the related parts in the proofs of other theorems.

\section{Proofs for \sys} \label{proofs}

Recall that at the $k$-th iteration in the $s$-th epoch, our parameter updating rule can be formulated as follows,
\begin{align} \label{eq:update_1_step_ss}
\x^s_k = \x^s_{k-1} - \eta_s \nabla f_{\psi_s(k)} \left( \x^s_{k-1} \right),
\end{align}
where $\eta_s$ is the learning rate for the $s$-th epoch.
    
If we recursively apply this updating rule \eqref{eq:update_1_step_ss}, we have that, at the $k$-th iteration in the $s$-th epoch,
\begin{align} \label{eq:update_k_step_ss}
\x^s_k = \x^s_0 - \eta_s \sum_{k'=1}^k \nabla f_{\psi_s(k')} \left( \x^s_{k'-1} \right),
\end{align}

After the updates of one epoch, i.e., after $bn$ steps in the $s$-th epoch, applying \eqref{eq:update_k_step_ss}, we have 
\begin{align}\label{eq:update_1_epoch_ss}
\x^{s+1}_0 = \x^s_0 - \eta_s \sum_{k=1}^{bn}\nabla f_{\psi_s(k)} \left( \x^s_{k-1} \right),
\end{align}
where we use the fact that 
\begin{align*}
\x^{s+1}_{0} = \x^{s}_{bn}.
\end{align*}

\subsection{Proof of Theorem \ref{thm:partialrs_sample} }

Based on the updating rules above, our proof starts from the following formulation,

\begin{align}
    &\E \|\x_0^{s+1}-\x^* \|^2  
    \nonumber\\
    = & \E \|\x_0^s -   \eta_s \sum_{k=1}^{bn} \nabla f_{\psi_s(k)}(\x^s_{k-1}) -\x^* \|^2 
    \nonumber\\ 
    =&\E \| \x_0^s - \x^* \|^2 - 2 \eta_s \E \left \langle \x_0^s - \x^*,  \sum_{k=1}^{bn}  \nabla f_{\psi_s(k)}(\x^s_{k-1}) \right \rangle  +\eta_s^2 \E \left \| \sum_{k=1}^{bn}  \nabla f_{\psi_s(k)}(\x^s_{k-1}) \right \|^2
    \nonumber \\
    \leq & \E \| \x_0^s - \x^* \|^2 - 2 \eta_s \E \left \langle \x_0^s - \x^*, \sum_{k=1}^{bn} [ \nabla f_{\psi_s(k)}(\x^s_{k-1})  - \nabla f_{\psi_s(k)}(\x^s_0) ] \right \rangle \nonumber \\
    & - 2 \eta_s \E \left \langle \x_0^s - \x^*,   \sum_{k=1}^{bn} \nabla f_{\psi_s(k)}(\x^s_0) \right \rangle   + 2\eta_s^2 \E \left \| \sum_{k=1}^{bn}  [ \nabla f_{\psi_s(k)}(\x^s_{k-1})- \nabla f_{\psi_s(k)}(\x^s_0)] \right \|^2 \nonumber \\
    & + 2\eta_s^2 \E \left \| \sum_{k=1}^{bn} \nabla f_{\psi_s(k)}(\x^s_0) \right \|^2
    \nonumber \\
    = & \E \| \x_0^s - \x^* \|^2 \underbrace{ - 2 \eta_s \E \left \langle \x_0^s - \x^*, \sum_{k=1}^{bn} [ \nabla f_{\psi_s(k)}(\x^s_{k-1})  - \nabla f_{\psi_s(k)}(\x^s_0) ] \right \rangle }_{\calI_1} \nonumber \\
    & \underbrace{- 2 \eta_s \E \left \langle \x_0^s - \x^*,   \sum_{k=1}^{bn} \nabla f_{\psi_s(k)}(\x^s_0) \right \rangle}_{\calI_2}   + \underbrace{2\eta_s^2 \E \left \| \sum_{k=1}^{bn}  [ \nabla f_{\psi_s(k)}(\x^s_{k-1})- \nabla f_{\psi_s(k)}(\x^s_0)] \right \|^2}_{\calI_3} \label{eq:main} \\
    & + \underbrace{2\eta_s^2 \E \left \| \sum_{k=1}^{bn} \nabla f_{\psi_s(k)}(\x^s_0) - \E \sum_{k=1}^{bn} \nabla f_{\psi_s(k)}(\x^s_0) \right \|^2}_{\calI_4} + \underbrace{2\eta_s^2 \left \|\E \sum_{k=1}^{bn} \nabla f_{\psi_s(k)}(\x^s_0) \right \|^2}_{\calI_5} \nonumber
\end{align}
where the last equality uses the fact that $\E \|X-\E[X]\|^2 = \E \|X\|^2  - \|\E[X]\|^2$.

To prove the upper bound of \eqref{eq:main}, we need to bound $\calI_1$, $\calI_2$, $\calI_3$, $\calI_4$ and $\calI_5$ respectively.

\vfill
\noindent\textbf{Bound of $\calI_3$}

For $\calI_3$, we have that
\begin{align*}
&\left \| \sum_{k=1}^{bn}  [ \nabla f_{\psi_s(k)}(\x^s_{k-1})- \nabla f_{\psi_s(k)}(\x^s_0)] \right \|^2 
 \\
\leq & bn \sum_{k=1}^{bn}  \left \|    \nabla f_{\psi_s(k)}(\x^s_{k-1})- \nabla f_{\psi_s(k)}(\x^s_0) \right \|^2
 \\
\leq & bn \sum_{k=1}^{bn}  L^2  \left \| \x^s_{k-1}- \x^s_0 \right \|^2
 \\
\leq & bn \sum_{k=1}^{bn}  L^2  \left \|  \eta_s \sum_{k'=1}^{k-1}  \nabla f_{\psi_s(k')}(\x^s_{k'-1}) \right \|^2 
 \\
\leq & bn \sum_{k=1}^{bn} \eta_s^2 L^2 (k-1) \sum_{k'=1}^{k-1} \left \|   \nabla f_{\psi_s(k')}(\x^s_{k'-1}) \right \|^2 
 \\
\leq & \eta_s^2 L^2 G^2  bn \sum_{k=1}^{bn}  (k-1)^2 
\leq  \frac{1}{3}\eta_s^2 L^2 G^2 (bn)^4 
\end{align*}
where the first and the fourth inequalities uses the fact that $\| \sum_{k=1}^{bn} \mathbf{a}_k\|^2 \leq bn  \sum_{k=1}^{bn} \|\mathbf{a}_k\|^2$, the second inequality holds due to the Lipschitz continuity of the gradient, the third inequality is due to $\x^{s}_{k-1} = \x^s_{0} - \eta_s \sum_{k'=1}^{k-1}  \nabla f_{\psi_s(k')}(\x^s_{k'-1})$ and the last inequality is due to $\sum_{i=1}^n i^2 = \frac{n(n+1)(2n+1)}{6}$ and the boundedness of the gradient.

Therefore, we have
\begin{align}
\calI_3 = 2\eta_s^2 \E \left \| \sum_{k=1}^{bn}  [ \nabla f_{\psi_s(k)}(\x^s_{k-1})- \nabla f_{\psi_s(k)}(\x^s_0)] \right \|^2 \leq  \frac{2}{3}\eta_s^4 L^2 G^2 K^4 \label{eq:I3}
\end{align}

\vfill
\noindent\textbf{Bounds of $\calI_2$ and $\calI_5$}

For $\calI_2$ and $\calI_5$, the key is to know the form of $\E \sum_{k=1}^{bn} \nabla f_{\psi_s(k)}(\x^s_0)$.
\begin{align*} 
\E \sum_{k=1}^{bn} \nabla f_{\psi_s(k)}(\x^s_0) =& \E \sum_{B_l \in \calB_s} \sum_{i\in B_l} \nabla f_i(\x_0^s),
\end{align*}
where this equality holds since the random shuffling $\psi_s$ does not affect the summation in the LHS formula. 

Furthermore, we use indicator random variables to get the value of $\E \sum_{B_l \in \calB_s} \sum_{i\in B_l} \nabla f_i(\x_0^s)$. Let $\mathbb{I}_{B_l \in \calB_s}$ be the random variable to indicate whether the block $B_l$ is in $\calB_s$ or not. Therefore, we have
\begin{align*}
\mathbb{I}_{B_l \in \calB_s} = 
\begin{cases}
1, & \text{ if } B_l \in \calB_s \\ 
0, & \text{ if } B_l \notin \calB_s 
\end{cases},
\end{align*}
and
\begin{align*}
\mathbb{P}(\mathbb{I}_{B_l \in \calB_s} = 1) = \mathbb{P}(B_l \in \calB_s) =\frac{\binom{1}{1}  \binom{N-1}{n-1}}{\binom{N}{n}} = \frac{n}{N}
\end{align*} 
such that $\E[\mathbb{I}_{B_l \in \calB_s}] = \frac{n}{N}$.

Thus, we can obtain
\begin{align*}
\E \sum_{B_l \in \calB_s} \sum_{i\in B_l} \nabla f_i(\x_0^s) = \E \sum_{l = 1}^N \mathbb{I}_{B_l \in \calB_s} \left ( \sum_{i\in B_l} \nabla f_i(\x_0^s) \right) = \frac{n}{N} \sum_{l = 1}^N  \sum_{i\in B_l} \nabla f_i(\x_0^s) = \frac{n}{N}m \nabla F(\x^s_0).
\end{align*}

Therefore, we get the values of $\calI_2$ and $\calI_5$
\begin{align}
&\calI_2 =  - 2 \eta_s \E \left \langle \x_0^s - \x^*,   \sum_{k=1}^{bn} \nabla f_{\psi_s(k)}(\x^s_0) \right \rangle =- 2 \eta_s \frac{n}{N} m \E \left \langle \x_0^s - \x^*, \nabla F(\x_0^s) \right \rangle \label{eq:I2}
\\
&\calI_5 = 2\eta_s^2 \left \|\E \sum_{k=1}^{bn} \nabla f_{\psi_s(k)}(\x^s_0) \right \|^2 = 2\eta_s^2 \frac{n^2}{N^2} m^2 \left \| \nabla F(\x^s_0) \right \|^2 \label{eq:I5}
\end{align}

\vfill
\noindent\textbf{Bound of $\calI_4$}

Next, we will show the variance of sampling the $n$ blocks out of $N$ without replacement. We still use the indicator variables defined above.

The upper bound of $\calI_4$ determines the $\frac{1}{T}$ term and the $\frac{N-n}{N-1}$ factor existing in the convergence rate, which shows how the leading term $\frac{N-n}{N-1} \frac{1}{T}$ varying with the number of sampled blocks $n$.

Note that for any $l' \neq l''$, we have
\begin{align*}
\mathbb{P}(\mathbb{I}_{B_{l'}\in \calB_s } = 1, \mathbb{I}_{B_{l''}\in \calB_s} = 1 ) = \mathbb{P}( B_{l'}\in \calB_s \wedge B_{l''}\in \calB_s ) = \frac{\binom{2}{2} \binom{N-2}{n-2}}{\binom{N}{n}} = \frac{n(n-1)}{N(N-1)}
\end{align*}
Therefore, $\E[\mathbb{I}_{B_{l'}\in \calB_s } \cdot \mathbb{I}_{B_{l''}\in \calB_s}] = \frac{n(n-1)}{N(N-1)} $.

\begin{align*}
&\E \left \| \sum_{k=1}^{bn} \nabla f_{\psi_s(k)}(\x^s_0) - \E \sum_{k=1}^{bn} \nabla f_{\psi_s(k)}(\x^s_0) \right \|^2 
\\
=& \E \left \|\sum_{l = 1}^N \mathbb{I}_{B_l \in \calB_s} \left ( \sum_{i\in B_l} \nabla f_i(\x_0^s) \right) - \E \sum_{l = 1}^N \mathbb{I}_{B_l \in \calB_s} \left ( \sum_{i\in B_l} \nabla f_i(\x_0^s) \right) \right \|^2 
\\
=& \E \left \|\sum_{l = 1}^N \mathbb{I}_{B_l \in \calB_s} \left ( \sum_{i\in B_l} \nabla f_i(\x_0^s) \right) \right \|^2 - \left \|\frac{n}{N} \sum_{l = 1}^N \left ( \sum_{i\in B_l} \nabla f_i(\x_0^s) \right) \right \|^2 
\\
= & \E  \sum_{l = 1}^N \mathbb{I}_{B_l \in \calB_s}\cdot \mathbb{I}_{B_l \in \calB_s}  \left \| \sum_{i\in B_l} \nabla f_i(\x_0^s)  \right \|^2 + \E \sum_{l'\neq l''} \mathbb{I}_{B_{l'} \in \calB_s}\cdot \mathbb{I}_{B_{l''} \in \calB_s} \left \langle \sum_{i\in B_{l'}} \nabla f_i(\x_0^s), \sum_{i\in B_{l''}} \nabla f_i(\x_0^s)   \right \rangle  \\
&- \frac{n^2}{N^2}  \sum_{l = 1}^N \left \| \sum_{i\in B_l} \nabla f_i(\x_0^s)  \right \|^2 -  \frac{n^2}{N^2}  \sum_{l' \neq l''}^N \left \langle \sum_{i\in B_{l'}} \nabla f_i(\x_0^s), \sum_{i\in B_{l''}} \nabla f_i(\x_0^s)  \right \rangle 
\\
= & \E  \sum_{l = 1}^N \mathbb{I}_{B_l \in \calB_s}\cdot \mathbb{I}_{B_l \in \calB_s}  \left \| \sum_{i\in B_l} \nabla f_i(\x_0^s)  \right \|^2 + \E \sum_{l'\neq l''} \mathbb{I}_{B_{l'} \in \calB_s}\cdot \mathbb{I}_{B_{l''} \in \calB_s} \left \langle \sum_{i\in B_{l'}} \nabla f_i(\x_0^s), \sum_{i\in B_{l''}} \nabla f_i(\x_0^s)   \right \rangle  \\
&- \frac{n^2}{N^2}  \sum_{l = 1}^N \left \| \sum_{i\in B_l} \nabla f_i(\x_0^s)  \right \|^2 -  \frac{n^2}{N^2}  \sum_{l' \neq l''}^N \left \langle \sum_{i\in B_{l'}} \nabla f_i(\x_0^s), \sum_{i\in B_{l''}} \nabla f_i(\x_0^s)  \right \rangle 
\\
= & \left ( \frac{n}{N}-\frac{n^2}{N^2}\right ) \sum_{l = 1}^N  \left \| \sum_{i\in B_l} \nabla f_i(\x_0^s)  \right \|^2 +  \left ( \frac{n(n-1)}{N(N-1)}  - \frac{n^2}{N^2} \right) \sum_{l'\neq l''}  \left \langle \sum_{i\in B_{l'}} \nabla f_i(\x_0^s), \sum_{i\in B_{l''}} \nabla f_i(\x_0^s)   \right \rangle 
\end{align*}
where the last equality is due to $\E[\mathbb{I}_{B_{l}\in \calB_s } \cdot \mathbb{I}_{B_{l}\in \calB_s}] = \E[\mathbb{I}_{B_{l}\in \calB_s }] = \frac{n}{N} $ and  $\E[\mathbb{I}_{B_{l'}\in \calB_s } \cdot \mathbb{I}_{B_{l''}\in \calB_s}] = \frac{n(n-1)}{N(N-1)} $, and the second equality is due to $\E\|X-\E X\|^2 = \E\|X\|^2-\|\E X\|^2$.

On the other hand, we have 
\begin{align*}
&\E_{l} \left \| \sum_{i\in B_l} \nabla f_i(\x_0^s) - b \nabla F(\x_0^s) \right \|^2 
\\
= & \E_{l} \left \| \sum_{i\in B_l} \nabla f_i(\x_0^s) - \E_l \sum_{i\in B_l} \nabla f_i(\x_0^s) \right \|^2 
\\
= & \E_{l} \left \| \sum_{i\in B_l} \nabla f_i(\x_0^s) - \frac{1}{N} \sum_{l = 1}^N \sum_{i\in B_l} \nabla f_i(\x_0^s) \right \|^2 
\\
= & \E_{l} \left \| \sum_{i\in B_l} \nabla f_i(\x_0^s) \right \|^2 - \frac{1}{N^2} \left \|\sum_{l = 1}^N \sum_{i\in B_l} \nabla f_i(\x_0^s) \right \|^2 
\\
= & \frac{1}{N} \sum_{l=1}^N \left \| \sum_{i\in B_l} \nabla f_i(\x_0^s) \right \|^2 - \frac{1}{N^2} \sum_{l = 1}^N \left \|\sum_{i\in B_l} \nabla f_i(\x_0^s) \right \|^2 - \frac{1}{N^2} \sum_{l' \neq l''} \left \langle \sum_{i\in B_{l'}} \nabla f_i(\x_0^s), \sum_{i\in B_{l''}} \nabla f_i(\x_0^s)\right\rangle 
\\
= & \left( \frac{1}{N}- \frac{1}{N^2} \right) \sum_{l = 1}^N \left \|\sum_{i\in B_l} \nabla f_i(\x_0^s) \right \|^2 - \frac{1}{N^2} \sum_{l' \neq l''} \left \langle \sum_{i\in B_{l'}} \nabla f_i(\x_0^s), \sum_{i\in B_{l''}} \nabla f_i(\x_0^s)\right\rangle .
\end{align*}

By comparing the RHS of the above two equations, we can observe that
\begin{align*}
\E \left \| \sum_{k=1}^{bn} \nabla f_{\psi_s(k)}(\x^s_0) - \E \sum_{k=1}^{bn} \nabla f_{\psi_s(k)}(\x^s_0) \right \|^2  = \frac{n(N-n)}{N-1} \E_{\xi'} \left \| \sum_{i\in B_{\xi'}} \nabla f_i(\x_0^s) - b \nabla F(\x_0^s) \right \|^2.
\end{align*}

If we further apply our assumption that $\E_{\xi'} \left \| \frac{1}{b}\sum_{i\in B_{\xi'}} \nabla f_i(\x_0^s) - \nabla F(\x_0^s) \right \|^2 \leq  h_D   \frac{\sigma^2}{b} $, then there is
\begin{align}
\calI_4 = 2 \eta^2_s \E \left \| \sum_{k=1}^{bn} \nabla f_{\psi_s(k)}(\x^s_0) - \E \sum_{k=1}^{bn} \nabla f_{\psi_s(k)}(\x^s_0) \right \|^2 \leq 2 \eta^2_s\frac{n b (N-n)}{N-1} h_D \sigma^2 \label{eq:I4}
\end{align}
This result shows the connection between the variance of block-wise sampling without replacement and the variance of sampling single data point independently and uniformly.

\noindent\textbf{Bound of $\calI_1$}

The upper bound of $\calI_1$ is critical to the proof of obtaining a faster rate. Before presenting the upper bound of $\calI_1$, recall the Fact \ref{fact:grad_hess} that

\begin{align*}
\nabla f(\z) - \nabla f(\y) = \int_{\y}^{\z} H(\x) \mathrm{d}\x :=\int_{0}^{\|\z-\y \|} H \left( \y + \frac{\z-\y}{\|\z-\y\|} t \right )  \frac{\z-\y}{\|\z-\y\|} \ \  \mathrm{d} t.
\end{align*}

Now we can show the upper bound of $\calI_1$ as follows,
\begin{align*}
\calI_1 =&  - 2 \eta_s \E \left \langle \x_0^s - \x^*, \sum_{k=1}^{bn} [ \nabla f_{\psi_s(k)}(\x^s_{k-1})  - \nabla f_{\psi_s(k)}(\x^s_0) ] \right \rangle 
\\
=&  - 2 \eta_s \E \left \langle \x_0^s - \x^*, \sum_{k=1}^{bn}  \int_{\x^s_0}^{\x^s_{k-1}} H_{\psi_s(k)}(\x) \mathrm{d}\x \right \rangle 
\\
= &  - 2 \eta_s \E \left \langle \x_0^s - \x^*, \sum_{k=1}^{bn}  \int_{\x^s_0}^{\x^s_{k-1}}  H_{\psi_s(k)}(\x^*)  \mathrm{d}\x \right \rangle \\
&  - 2 \eta_s \E \left \langle \x_0^s - \x^*, \sum_{k=1}^{bn}  \int_{\x^s_0}^{\x^s_{k-1}} \left( H_{\psi_s(k)}(\x) - H_{\psi_s(k)}(\x^*) \right) \mathrm{d}\x \right \rangle
\\
= &  - 2 \eta_s \E \left \langle \x_0^s - \x^*, \sum_{k=1}^{bn}   H_{\psi_s(k)}(\x^*)  \left(\x^s_{k-1} - \x^s_0 \right) \right \rangle \\
&  - 2 \eta_s \E \left \langle \x_0^s - \x^*, \sum_{k=1}^{bn}  \int_{\x^s_0}^{\x^s_{k-1}} \left( H_{\psi_s(k)}(\x) - H_{\psi_s(k)}(\x^*) \right) \mathrm{d}\x \right \rangle 
\\
= &  2 \eta_s^2 \E \left \langle \x_0^s - \x^*, \sum_{k=1}^{bn}   H_{\psi_s(k)}(\x^*)  \left(\sum_{k'=1}^{k-1} \nabla f_{\psi_s(k')}(\x_{k'-1}^s)  \right) \right \rangle \\
&  - 2 \eta_s \E \left \langle \x_0^s - \x^*, \sum_{k=1}^{bn}  \int_{\x^s_0}^{\x^s_{k-1}} \left( H_{\psi_s(k)}(\x) - H_{\psi_s(k)}(\x^*) \right) \mathrm{d}\x \right \rangle 
\\
= & \underbrace{  2 \eta_s^2 \E \left \langle \x_0^s - \x^*, \sum_{k=1}^{bn}   H_{\psi_s(k)}(\x^*)  \left(\sum_{k'=1}^{k-1} \nabla f_{\psi_s(k')}(\x_0^s)  \right) \right \rangle }_{\calI_{11}}\\
&  \underbrace{2 \eta_s^2 \E \left \langle \x_0^s - \x^*, \sum_{k=1}^{bn}   H_{\psi_s(k)}(\x^*)  \left(\sum_{k'=1}^{k-1} \nabla f_{\psi_s(k')} (\x_{k'-1}^s) - \nabla f_{\psi_s(k')}(\x_0^s)  \right) \right \rangle }_{\calI_{12}} \\
&  \underbrace{- 2 \eta_s \E \left \langle \x_0^s - \x^*, \sum_{k=1}^{bn}  \int_{\x^s_0}^{\x^s_{k-1}} \left( H_{\psi_s(k)}(\x) - H_{\psi_s(k)}(\x^*) \right) \mathrm{d}\x \right \rangle }_{\calI_{13}}
\end{align*}
where the second equality holds due to the Fact \ref{fact:grad_hess} and the fifth equality holds since we have $\x^{s}_{k-1} = \x^s_{0} - \eta_s \sum_{k'=1}^{k-1}  \nabla f_{\psi_s(k')}(\x^s_{k'-1})$. Note that $H_{\psi_s(k)}(\x^*)$ is the Hessian of the function $f_{\psi_s(k)}(\x)$ at the point $\x^*$.

In order to obtain the upper bound of $\calI_1$, we need to bound $\calI_{11}$, $\calI_{12}$ and $\calI_{13}$ separately.

\textbf{Bound of $\calI_{11}$} 

In $\calI_{11}$, we need first compute $\E \sum_{k=1}^{bn}   H_{\psi_s(k)}(\x^*)  \left(\sum_{k'=1}^{k-1} \nabla f_{\psi_s(k')}(\x_0^s)  \right)$, which is the key ingredient for obtaining the $O(\frac{1}{T^2})$ term in the convergence rate.

To make our proof more clear, recall the manipulation in our algorithm for obtain $\psi_s$:
\begin{enumerate}
\item At the $s$-th epoch, random sample $n$ out of the total $N$ blocks \emph{without replacement} to get a set of sampled blocks $\calB_s$ with size of $n$. Each block has $b$ data samples.
\item Then, perform random shuffling of the $nb$ data samples and obtain the shuffled index sequence $\psi_s$ with $|\psi_s| = nb$.
\end{enumerate}
where we define $\calB_s$ being the set of blocks that are sampled each epoch.

In order to compute the expectation, we use the indicator random variable for a more clear derivation.

Define $\mathbb{I}_{\psi_s(k) = i}$ be the indicator random variable showing whether one data sample with index $i$ located in the $k$-th place after the above 2-step manipulation. The event $\psi_s(k) = i$ is equivalent to the event that $i \in B_l \wedge B_l \in \calB_s \wedge \psi_s(k) = i$ where $B_l$ is the block that the $i$-th sample lies in.  

Thus, we have
\begin{align*}
&\mathbb{I}_{\psi_s(k) = i} = \begin{cases} 
1, & \text{ if } \psi_s(k) = i \\ 
0, & \text{ if } \psi_s(k) \neq i 
\end{cases} \\
\text{and} \quad &\\
&\mathbb{P}(\mathbb{I}_{\psi_s(k) = i}=1)  = \mathbb{P}(i \in B_l,  B_l \in \calB_s, \psi_s(k) = i) = \frac{\binom{1}{1}\binom{N-1}{n-1}(bn-1)!}{\binom{N}{n} (bn)!}=\frac{1}{Nb}.
\end{align*}
Thus, we can observe that
\begin{align*}
&H_{\psi_s(k)}(\x^*) = \sum_{i=1}^{m} \mathbb{I}_{\psi_s(k) = i} H_i(\x^*) \\
&\nabla f_{\psi_s(k')}(\x_0^s) = \sum_{j=1}^{m} \mathbb{I}_{\psi_s(k') = j} \nabla f_j(\x_0^s) 
\end{align*}

Based on the above definition, we are ready to compute the expectation.

\begin{align*}
&\E \sum_{k=1}^{bn}   H_{\psi_s(k)}(\x^*)  \left(\sum_{k'=1}^{k-1} \nabla f_{\psi_s(k')}(\x_0^s)  \right) 
\\
=& \E \sum_{k=1}^{bn}   \sum_{i=1}^{m} \mathbb{I}_{\psi_s(k) = i} H_i(\x^*)  \left(\sum_{k'=1}^{k-1} \sum_{j=1}^{m} \mathbb{I}_{\psi_s(k') = j} \nabla f_j(\x_0^s) \right) 
\\
=& \sum_{k=1}^{bn} \sum_{k'=1}^{k-1}  \sum_{i=1}^{m}    \sum_{j=1}^{m}  \E \left[ \mathbb{I}_{\psi_s(k) = i} \cdot \mathbb{I}_{\psi_s(k') = j} \right] H_i(\x^*) \nabla f_j(\x_0^s)
\\
=& \sum_{k=1}^{bn} \sum_{k'=1}^{k-1}  \sum_{i \neq j} \E \left[ \mathbb{I}_{\psi_s(k) = i} \cdot \mathbb{I}_{\psi_s(k') = j} \right] H_i(\x^*) \nabla f_j(\x_0^s)
\end{align*}
where the last equality holds since 
\begin{align*}
\mathbb{P}(\mathbb{I}_{\psi_s(k) = i} = 1 , \mathbb{I}_{\psi_s(k') = i} = 1) = 0 \quad \Rightarrow \quad \E \left[ \mathbb{I}_{\psi_s(k) = i} \cdot \mathbb{I}_{\psi_s(k') = i} \right] = 0
\end{align*}
because $k > k'$ and one data sample cannot appear in different positions at the same time. 

Therefore, for any $k > k'$ and $i \neq j$, 

\begin{enumerate}
\item If $i\in B_l$,$j\in B_{l}$, we have
\begin{align*}
&\mathbb{P}(\mathbb{I}_{\psi_s(k) = i}=1, \mathbb{I}_{\psi_s(k') = j} = 1)
=\frac{\binom{1}{1}\binom{N-1}{n-1}(nb-2)!}{\binom{N}{n}(nb)!} = \frac{1}{Nb(nb-1)}
\\
\Rightarrow &
\\
&\E \left[ \mathbb{I}_{\psi_s(k) = i} \cdot \mathbb{I}_{\psi_s(k') = j} \right] = \frac{1}{Nb(nb-1)}.
\end{align*}
where $\binom{1}{1}$ means $B_l$ are chosen ahead,  $\binom{N-1}{n-1}$ means $(n-1)$ blocks excluding $B_l$ are randomly chosen from $N-1$ blocks excluding $B_l$, $\binom{N}{n}$ is the total number of ways of choosing $n$ blocks from $N$ blocks, $(nb-2)!$ is the number of ways of shuffling the data in $n$ blocks expect $i$ and $j$, and $(nb)!$ is the number of ways of shuffling all the data in $n$ blocks.

\item If $i\in B_l$,$j\in B_{l}$ and $l \neq l'$, we have
\begin{align*}
&\mathbb{P}(\mathbb{I}_{\psi_s(k) = i}=1, \mathbb{I}_{\psi_s(k') = j} = 1)
=\frac{\binom{2}{2}\binom{N-2}{n-2}(nb-2)!}{\binom{N}{n}(nb)!} = \frac{n-1}{Nb(N-1)(nb-1)}.
\\
\Rightarrow &
\\
&\E \left[ \mathbb{I}_{\psi_s(k) = i} \cdot \mathbb{I}_{\psi_s(k') = j} \right] = \frac{n-1}{Nb(N-1)(nb-1)}.
\end{align*}
where $\binom{2}{2}$ means $B_l$ and $B_{l'}$ are chosen ahead,  $\binom{N-2}{n-2}$ means $(n-2)$ blocks excluding $B_l$ and $B_{l'}$ are randomly chosen from $N-2$ blocks excluding $B_l$ and $B_{l'}$, $\binom{N}{n}$ is the total number of ways of choosing $n$ blocks from $N$ blocks, $(nb-2)!$ is the number of ways of shuffling the data in $n$ blocks expect $i$ and $j$, and $(nb)!$ is the number of ways of shuffling all the data in $n$ blocks.
\end{enumerate}

Thus, we have
\begin{align*}
&\E \sum_{k=1}^{bn}   H_{\psi_s(k)}(\x^*)  \left(\sum_{k'=1}^{k-1} \nabla f_{\psi_s(k')}(\x_0^s)  \right) 
\\
=& \sum_{k=1}^{bn} \sum_{k'=1}^{k-1}  \sum_{i \neq j} \E \left[ \mathbb{I}_{\psi_s(k) = i} \cdot \mathbb{I}_{\psi_s(k') = j} \right] H_i(\x^*) \nabla f_j(\x_0^s)
\\
=& \sum_{k=1}^{bn} \sum_{k'=1}^{k-1}  \Bigg (\sum_{l=1}^N\sum_{\substack{i \neq j\\i,j \in B_l } } \E \left[ \mathbb{I}_{\psi_s(k) = i} \cdot \mathbb{I}_{\psi_s(k') = j} \right] H_i(\x^*) \nabla f_j(\x_0^s) \\
& + \sum_{l\neq l'} \sum_{ \substack{i,j\\ i \in B_l, j \in B_{l'}}}  \E \left[ \mathbb{I}_{\psi_s(k) = i} \cdot \mathbb{I}_{\psi_s(k') = j} \right] H_i(\x^*) \nabla f_j(\x_0^s)  \Bigg)
\\
=& \frac{nb(nb-1)}{2}  \left (\sum_{l=1}^N\sum_{\substack{i \neq j\\i,j \in B_l } } \frac{1}{Nb(nb-1)} H_i(\x^*) \nabla f_j(\x_0^s) +  \sum_{l\neq l'} \sum_{ \substack{i,j\\ i \in B_l, j \in B_{l'}}} \frac{n-1}{Nb(N-1)(nb-1)} H_i(\x^*) \nabla f_j(\x_0^s)  \right)
\\
=& \sum_{l=1}^N\sum_{\substack{i \neq j\\i,j \in B_l } } \frac{n}{2N} H_i(\x^*) \nabla f_j(\x_0^s) +  \sum_{l\neq l'} \sum_{ \substack{i,j\\ i \in B_l, j \in B_{l'}}} \frac{n(n-1)}{2N(N-1)} H_i(\x^*) \nabla f_j(\x_0^s)  
\end{align*}

Plugging in the above result into $\calI_{11}$, we can get 
\begin{align*}
\calI_{11} =&  2 \eta_s^2 \E \left \langle \x_0^s - \x^*, \sum_{l=1}^N\sum_{\substack{i \neq j\\i,j \in B_l } } \frac{n}{2N} H_i(\x^*) \nabla f_j(\x_0^s) +  \sum_{l\neq l'} \sum_{ \substack{i,j\\ i \in B_l, j \in B_{l'}}} \frac{n(n-1)}{2N(N-1)} H_i(\x^*) \nabla f_j(\x_0^s) \right \rangle
\\
=&  2 \eta_s^2 \E \left \langle \x_0^s - \x^*, \sum_{l=1}^N\sum_{\substack{i \neq j\\i,j \in B_l } } \frac{n}{2N} H_i(\x^*) \left(\nabla f_j(\x_0^s)-\nabla f_j(\x^*)\right)\right\rangle\\
&+ 2 \eta_s^2 \left\langle \x_0^s - \x^*, \sum_{l\neq l'} \sum_{ \substack{i,j\\ i \in B_l, j \in B_{l'}}} \frac{n(n-1)}{2N(N-1)} H_i(\x^*) \left(\nabla f_j(\x_0^s)-\nabla f_j(\x^*)\right) \right \rangle\\
& + 2 \eta_s^2 \E \left \langle \x_0^s - \x^*, \sum_{l=1}^N\sum_{\substack{i \neq j\\i,j \in B_l } } \frac{n}{2N} H_i(\x^*) \nabla f_j(\x^*)+  \sum_{l\neq l'} \sum_{ \substack{i,j\\ i \in B_l, j \in B_{l'}}} \frac{n(n-1)}{2N(N-1)} H_i(\x^*) \nabla f_j(\x^*) \right \rangle
\\
=&  \underbrace{ \eta_s^2 \frac{n}{N} \sum_{l=1}^N\sum_{\substack{i \neq j\\i,j \in B_l } }    \left \langle H_i(\x^*)(\x_0^s - \x^*),  \nabla f_j(\x_0^s)-\nabla f_j(\x^*)\right\rangle}_{\calJ_1}\\
&+ \underbrace{ \eta_s^2 \frac{n(n-1)}{N(N-1)}  \sum_{l\neq l'} \sum_{\substack{i,j\\ i \in B_l, j \in B_{l'}}} \left\langle H_i(\x^*) (\x_0^s - \x^*), \nabla f_j(\x_0^s)-\nabla f_j(\x^*) \right \rangle}_{\calJ_2}\\
& + \underbrace{ \eta_s^2 \frac{n}{N} \left \langle \x_0^s - \x^*, \sum_{l=1}^N\sum_{\substack{i \neq j\\i,j \in B_l } } H_i(\x^*) \nabla f_j(\x^*)+  \sum_{l\neq l'} \sum_{ \substack{i,j\\ i \in B_l, j \in B_{l'}}} \frac{n-1}{N-1} H_i(\x^*) \nabla f_j(\x^*) \right \rangle}_{\calJ_3}
\end{align*}

To bound $\calI_{11}$, we need bound the terms $\calJ_1$, $\calJ_2$ and $\calJ_3$ separately.

\textbf{Bound of $\calJ_{1} + \calJ_{2}$}

\begin{align*}
\calJ_{1} + \calJ_{2} =&   \eta_s^2 \frac{n}{N} \sum_{l=1}^N\sum_{\substack{i \neq j\\i,j \in B_l } }    \left \langle H_i(\x^*)(\x_0^s - \x^*),  \nabla f_j(\x_0^s)-\nabla f_j(\x^*)\right\rangle\\
&+  \eta_s^2 \frac{n(n-1)}{N(N-1)}  \sum_{l\neq l'} \sum_{\substack{i,j\\ i \in B_l, j \in B_{l'}}} \left\langle H_i(\x^*) (\x_0^s - \x^*), \nabla f_j(\x_0^s)-\nabla f_j(\x^*) \right \rangle
\\
\leq &   \eta_s^2 \frac{n}{N} \sum_{l=1}^N\sum_{\substack{i \neq j\\i,j \in B_l } }    \| H_i(\x^*)\| \|\x_0^s - \x^*\| \| \nabla f_j(\x_0^s)-\nabla f_j(\x^*)\| \\
&+  \eta_s^2 \frac{n(n-1)}{N(N-1)}  \sum_{l\neq l'} \sum_{\substack{i,j\\ i \in B_l, j \in B_{l'}}} \| H_i(\x^*) \| \|\x_0^s - \x^*\| \|\nabla f_j(\x_0^s)-\nabla f_j(\x^*) \| 
\\
\leq &   \eta_s^2 \frac{n}{N} L^2 \sum_{l=1}^N\sum_{\substack{i \neq j\\i,j \in B_l } }  \|\x_0^s - \x^*\|^2+  \eta_s^2 \frac{n(n-1)}{N(N-1)} L^2 \sum_{l\neq l'} \sum_{\substack{i,j\\ i \in B_l, j \in B_{l'}}} \|\x_0^s - \x^*\|^2 
\\
\leq &   \eta_s^2 \frac{n}{N} L^2 N b(b-1)  \|\x_0^s - \x^*\|^2+  \eta_s^2 \frac{n(n-1)}{N(N-1)} L^2 N(N-1) b^2 \|\x_0^s - \x^*\|^2 
\\
\leq &   \eta_s^2 L^2 n b(nb-1) \|\x_0^s - \x^*\|^2 
\\
\leq &   \eta_s^2 L^2 n^2 b^2 \|\x_0^s - \x^*\|^2 
\end{align*}

\textbf{Bound of $\calJ_{3}$}

\begin{align*}
\calJ_3 =& 2 \eta_s^2  \left \langle \x_0^s - \x^*, \sum_{l=1}^N\sum_{\substack{i \neq j\\i,j \in B_l } } \frac{n}{2N} H_i(\x^*) \nabla f_j(\x^*)+  \sum_{l\neq l'} \sum_{ \substack{i,j\\ i \in B_l, j \in B_{l'}}} \frac{n(n-1)}{2N(N-1)} H_i(\x^*) \nabla f_j(\x^*) \right \rangle
\\
\leq & \eta_s^2 \frac{n}{N} \|  \x_0^s - \x^*\| \left \| \sum_{l=1}^N\sum_{\substack{i \neq j\\i,j \in B_l } } H_i(\x^*) \nabla f_j(\x^*)+  \sum_{l\neq l'} \sum_{ \substack{i,j\\ i \in B_l, j \in B_{l'}}} \frac{n-1}{N-1} H_i(\x^*) \nabla f_j(\x^*) \right \|
\end{align*}

Note the fact that
\begin{align*}
\left (\sum_{l=1}^N  \sum_{i\in B_l}  H_i(\x^*) \right ) \left(\sum_{l'=1}^N \sum_{j\in B_{l'}} \nabla f_j(\x^*) \right) = 0
\end{align*}
since $\sum_{l'=1}^N \sum_{j\in B_{l'}} \nabla f_j(\x^*)  = Nb\nabla f(\x^*) = 0$.

Therefore, letting $\rho \geq 0$, we can obtain a tighter bound for the following term, 
\begin{align}
&\left \| \sum_{l=1}^N\sum_{\substack{i \neq j\\i,j \in B_l } }  H_i(\x^*) \nabla f_j(\x^*) +  \sum_{l\neq l'} \sum_{ \substack{i,j\\ i \in B_l, j \in B_{l'}}} \frac{n-1}{N-1} H_i(\x^*) \nabla f_j(\x^*) \right \| 
\nonumber \\
=&\Bigg\| \sum_{l=1}^N\sum_{\substack{i \neq j\\i,j \in B_l } }  H_i(\x^*) \nabla f_j(\x^*) +  \sum_{l\neq l'} \sum_{ \substack{i,j\\ i \in B_l, j \in B_{l'}}} \frac{n-1}{N-1} H_i(\x^*) \nabla f_j(\x^*) \nonumber  \\
&-  \rho \left (\sum_{l=1}^N  \sum_{i\in B_l}  H_i(\x^*) \right ) \left(\sum_{l'=1}^N \sum_{j\in B_{l'}} \nabla f_j(\x^*) \right) \Bigg \|
\nonumber  \\
=&\Bigg\| (1-\rho) \sum_{l=1}^N\sum_{\substack{i \neq j\\i,j \in B_l } }  H_i(\x^*) \nabla f_j(\x^*) +  \sum_{l\neq l'} \sum_{ \substack{i,j\\ i \in B_l, j \in B_{l'}}} (\frac{n-1}{N-1}-\rho) H_i(\x^*) \nabla f_j(\x^*) - \rho \sum_{l=1}^N\sum_{i \in B_l  }  H_i(\x^*) \nabla f_i(\x^*) \Bigg \|
\nonumber  \\
\leq & |1-\rho| Nb(b-1) LG +  N(N-1) b^2\left|\frac{n-1}{N-1}-\rho \right | + \rho Nb  LG  \label{eq:discuss}
\end{align}

To find a tight upper bound of \eqref{eq:discuss}, we need to discuss its value as follows:
\begin{enumerate}
	\item When $\rho \geq 1$, the RHS of the above inequality is increasing with respect to $\rho$.
    \item When $\rho \leq \frac{n-1}{N-1}$, the RHS of the above inequality is decreasing with respect to $\rho$.
    \item When $\frac{n-1}{N-1} \leq \rho \leq 1$ and $N \geq 2$, the RHS of the above inequality is increasing with respect to $\rho$.
\end{enumerate}

Thus $\rho = \frac{n-1}{N-1}$ is the minimizer of \eqref{eq:discuss}. Plugging the value of $\rho$ into \eqref{eq:discuss}, we can obtain the upper bound as 
\begin{align*}
|1-\rho| Nb(b-1) LG +  N(N-1) b^2\left|\frac{n-1}{N-1}-\rho \right | + \rho Nb  LG \leq \frac{N-n}{N-1}Nb(b-1)LG + \frac{n-1}{N-1}NbLG
\end{align*}
which leads to
\begin{align*}
\calJ_3 \leq & \eta_s^2 \frac{n}{N} \|  \x_0^s - \x^*\| \left \| \sum_{l=1}^N\sum_{\substack{i \neq j\\i,j \in B_l } } H_i(\x^*) \nabla f_j(\x^*)+  \sum_{l\neq l'} \sum_{ \substack{i,j\\ i \in B_l, j \in B_{l'}}} \frac{n-1}{N-1} H_i(\x^*) \nabla f_j(\x^*) \right \|
\\
\leq & \eta_s^2  \|  \x_0^s - \x^*\|  \frac{N-n}{N-1}nb(b-1)LG + \eta_s^2  \|  \x_0^s - \x^*\|  \frac{n-1}{N-1}nbLG
\\
\leq & \frac{1}{8} \eta_s \mu nb \|  \x_0^s - \x^*\|^2 +  2\eta_s^3 nb \left ( \frac{N-n}{N-1} \right )^2 (b-1)^2L^2G^2 \mu^{-1} \\
&+ \frac{1}{8} \eta_s \mu nb \|  \x_0^s - \x^*\|^2  + 2\eta_s^3 nb \left ( \frac{n-1}{N-1} \right )^2 L^2G^2\mu^{-1}
\\
= & \frac{1}{4} \eta_s \mu nb \|  \x_0^s - \x^*\|^2 +  2\eta_s^3 nb L^2G^2 \mu^{-1} \left[\left ( \frac{N-n}{N-1} \right )^2 (b-1)^2 + \left ( \frac{n-1}{N-1} \right )^2 \right ]
\end{align*}

Therefore, we have
\begin{align*}
\calI_{11} \leq \frac{1}{4} \eta_s \mu nb \|  \x_0^s - \x^*\|^2 + \eta_s^2 L^2 n^2 b^2 \|\x_0^s - \x^*\|^2 +  2\eta_s^3 nb L^2G^2 \mu^{-1} \left[\left ( \frac{N-n}{N-1} \right )^2 (b-1)^2 + \left ( \frac{n-1}{N-1} \right )^2 \right ] 
\end{align*}

\textbf{Bound of $\calI_{12}$}
\begin{align*}
\calI_{12} =& 2 \eta_s^2 \E \left \langle \x_0^s - \x^*, \sum_{k=1}^{bn}   H_{\psi_s(k)}(\x^*)  \left(\sum_{k'=1}^{k-1} \nabla f_{\psi_s(k')} (\x_{k'-1}^s) - \nabla f_{\psi_s(k')}(\x_0^s)  \right) \right \rangle 
\\
\leq & 2 \eta_s^2 \E \left \| \x_0^s - \x^* \right \| \left\| \sum_{k=1}^{bn}   H_{\psi_s(k)}(\x^*)  \left(\sum_{k'=1}^{k-1} \nabla f_{\psi_s(k')} (\x_{k'-1}^s) - \nabla f_{\psi_s(k')}(\x_0^s)  \right) \right \|
\\
\leq & 2 \eta_s^2 \E \left \| \x_0^s - \x^* \right \| \sum_{k=1}^{bn} \sum_{k'=1}^{k-1} \left\|   H_{\psi_s(k)}(\x^*)  \left( \nabla f_{\psi_s(k')} (\x_{k'-1}^s) - \nabla f_{\psi_s(k')}(\x_0^s)  \right) \right \|
\\
\leq & 2 \eta_s^2 \E \left \| \x_0^s - \x^* \right \| \sum_{k=1}^{bn} \sum_{k'=1}^{k-1} L \left\|    \nabla f_{\psi_s(k')} (\x_{k'-1}^s) - \nabla f_{\psi_s(k')}(\x_0^s)   \right \|
\\
\leq & 2 \eta_s^2 \E \left \| \x_0^s - \x^* \right \| \sum_{k=1}^{bn} \sum_{k'=1}^{k-1} L^2 \left\|    \x_{k'-1}^s - \x_0^s  \right \|
\\
= & 2 \eta_s^2 \E \left \| \x_0^s - \x^* \right \| \sum_{k=1}^{bn} \sum_{k'=1}^{k-1} L^2 \left\|   \eta_s \sum_{k'=1}^{k-1}  \nabla f_{\psi_s(k')}(\x^s_{k'-1})  \right \|
\\
\leq & 2 \eta_s^3 L^2  G \left \| \x_0^s - \x^* \right \| \sum_{k=1}^{bn} (k-1)^2 
\\
\leq & \frac{2}{3} (bn)^3 \eta_s^3 L^2  G \left \| \x_0^s - \x^* \right \|  
\\
= & \frac{1}{3}  \eta_s^2 b^2 n^2 L  G \left \| \x_0^s - \x^* \right \|^2 + \frac{1}{3} \eta_s^4 b^4 n^4 L^3  G   
\end{align*}
where the first inequality is due to $\langle\mathbf{a},\mathbf{b} \rangle \leq \|\mathbf{a}\| \|\mathbf{b}\|$, the second inequality is due to $\| \sum_{k=1}^{bn} \mathbf{a}_k\| \leq  \sum_{k=1}^{bn} \|\mathbf{a}\|$, the third and fourth inequalities are because of Lipschitz gradient assumption, the second equality is because of $\x^{s}_{k-1} = \x^s_{0} - \eta_s \sum_{k'=1}^{k-1}  \nabla f_{\psi_s(k')}(\x^s_{k'-1})$, and the last inequality holds since $ab\leq \frac{\lambda}{2} a^2 + \frac{1}{2\lambda}b^2$.

\textbf{Bound of $\calI_{13}$}
\begin{align*}
\calI_{13} =& - 2 \eta_s \E \left \langle \x_0^s - \x^*, \sum_{k=1}^{bn}  \int_{\x^s_0}^{\x^s_{k-1}} \left( H_{\psi_s(k)}(\x) - H_{\psi_s(k)}(\x^*) \right) \mathrm{d}\x \right \rangle
\\
=& - 2 \eta_s \E \left \langle \x_0^s - \x^*, \sum_{k=1}^{bn}  \int_{0}^{\|\x^s_{k-1}-\x^s_0\|} \left( H_{\psi_s(k)}\left(\x_0^s + \frac{\x^s_{k-1}-\x^s_0}{\|\x^s_{k-1}-\x^s_0\|} t \right ) - H_{\psi_s(k)}(\x^*) \right) \frac{\x^s_{k-1}-\x^s_0}{\|\x^s_{k-1}-\x^s_0\|}  \mathrm{d} t \right \rangle
\\
=& - 2 \eta_s \E \sum_{k=1}^{bn}  \int_{0}^{\|\x^s_{k-1}-\x^s_0\|} \left \langle \x_0^s - \x^*, \left( H_{\psi_s(k)}\left(\x_0^s + \frac{\x^s_{k-1}-\x^s_0}{\|\x^s_{k-1}-\x^s_0\|} t \right ) - H_{\psi_s(k)}(\x^*) \right) \frac{\x^s_{k-1}-\x^s_0}{\|\x^s_{k-1}-\x^s_0\|} \right \rangle  \mathrm{d} t 
\\
\leq & 2 \eta_s \E \sum_{k=1}^{bn}  \int_{0}^{\|\x^s_{k-1}-\x^s_0\|} \left \| \x_0^s - \x^* \right \|\cdot  \left \| H_{\psi_s(k)}\left(\x_0^s + \frac{\x^s_{k-1}-\x^s_0}{\|\x^s_{k-1}-\x^s_0\|} t \right ) - H_{\psi_s(k)}(\x^*) \right\| \frac{\| \x^s_{k-1}-\x^s_0 \|}{\|\x^s_{k-1}-\x^s_0\|}  \mathrm{d} t 
\\
\leq & 2 \eta_s \left \| \x_0^s - \x^* \right \|\cdot \E \sum_{k=1}^{bn}  \int_{0}^{\|\x^s_{k-1}-\x^s_0\|} L_H \left \| \x_0^s + \frac{\x^s_{k-1}-\x^s_0}{\|\x^s_{k-1}-\x^s_0\|} t  - \x^* \right\|  \mathrm{d} t 
\\
\leq & 2 \eta_s \left \| \x_0^s - \x^* \right \|\cdot \E \sum_{k=1}^{bn}  \int_{0}^{\|\x^s_{k-1}-\x^s_0\|} L_H  \left( \| \x_0^s  - \x^* \| +  t\right)  \mathrm{d} t 
\\
= & 2 \eta_s \left \| \x_0^s - \x^* \right \| L_H  \cdot \E \sum_{k=1}^{bn}   \left( \|\x^s_{k-1}-\x^s_0\| \| \x_0^s  - \x^* \| +  \frac{1}{2} \|\x^s_{k-1}-\x^s_0\|^2 \right)   
\\
= & 2 \eta_s \left \| \x_0^s - \x^* \right \| L_H  \cdot \E \sum_{k=1}^{bn}   \left( \|\eta_s \sum_{k'=1}^{k-1} \nabla f_{\psi_s(k')}(\x^s_{k'-1})\| \| \x_0^s  - \x^* \| +  \frac{1}{2} \|\eta_s \sum_{k'=1}^{k-1} \nabla f_{\psi_s(k')}(\x^s_{k'-1}) \|^2 \right)  
\\
\leq & 2 \eta_s \left \| \x_0^s - \x^* \right \| L_H  \cdot \sum_{k=1}^{bn}   \left( \eta_s (k-1)  G \| \x_0^s  - \x^* \| +  \frac{1}{2} \eta_s^2 (k-1)^2 G^2 \right)  
\\
\leq & 2 \eta_s \left \| \x_0^s - \x^* \right \| L_H     \left(   \frac{(bn)^2}{2} \eta_s G \| \x_0^s  - \x^* \| +  \frac{(bn)^3}{6} \eta_s^2  G^2 \right)  
\\
= & \eta_s^2 (bn)^2 L_H G \left \| \x_0^s - \x^* \right \|^2  +  \frac{(bn)^3}{3} \eta_s^3  L_H G^2 \left \| \x_0^s - \x^* \right \| 
\\
\leq & \eta_s^2 (bn)^2 L_H G \left \| \x_0^s - \x^* \right \|^2  +  \frac{(bn)^2}{6} \eta_s^2  L_H G \left \| \x_0^s - \x^* \right \|^2 +  \frac{(bn)^4}{6} \eta_s^4  L_H G^3
\\
= & \frac{7}{6}\eta_s^2 b^2 n^2 L_H G \left \| \x_0^s - \x^* \right \|^2 + \frac{1}{6} \eta_s^4 b^4 n^4 L_H G^3
\end{align*}
where the second equality holds since the Fact \ref{fact:grad_hess}, the third inequality is due to $\|\mathbf{a} + \mathbf{b}  \| \leq \|\mathbf{a}\| + \|\mathbf{b}\|$, the fourth inequality is due to the boundedness of the gradient, the fifth equality is due to $\x^{s}_{k-1} = \x^s_{0} - \eta_s \sum_{k'=1}^{k-1}  \nabla f_{\psi_s(k')}(\x^s_{k'-1})$, and the last inequality is because of $ab\leq \frac{\lambda}{2} a^2 + \frac{1}{2\lambda}b^2$.

Based on the upper bounds of $\calI_{11}$, $\calI_{12}$ and $\calI_{13}$, we can obtain the bound of $\calI_1$ as
\begin{align}
\calI_1 =& \calI_{11} + \calI_{12} + \calI_{13}
 \nonumber \\
\leq & \frac{1}{4} \eta_s \mu nb \|  \x_0^s - \x^*\|^2 +  2\eta_s^3 nb L^2G^2 \mu^{-1} \left[\left ( \frac{N-n}{N-1} \right )^2 (b-1)^2 + \left ( \frac{n-1}{N-1} \right )^2 \right ] \label{eq:I1} \\
&+\frac{1}{6}  \eta_s^2 b^2 n^2 (2L  G + 6L^2 + 7L_H G) \left \| \x_0^s - \x^* \right \|^2 + \frac{1}{6} \eta_s^4 b^4 n^4 (2L^3  G + L_H G^3)  \nonumber
\end{align}

\vfill
Now we summarize the above upper bounds of $\calI_1$, $\calI_2$, $\calI_3$, $\calI_4$, $\calI_5$. Plugging   \eqref{eq:I1},\eqref{eq:I2}, \eqref{eq:I3}, \eqref{eq:I4}, \eqref{eq:I5} into \eqref{eq:main}, we can eventually obtain that
\begin{align}
\E\|x_0^{s+1} - x^*\|^2 \leq& (1+\frac{1}{4} \eta_s \mu bn + C_1 \eta_s^2 b^2 n^2 )\|x_0^s-x^* \|^2 - 2 \eta_s bn \left \langle \x_0^s - \x^*, \nabla F(\x_0^s) \right \rangle + 2\eta_s^2 b^2n^2 \left \| \nabla F(\x^s_0) \right \|^2 \nonumber \\
&+ C_2 \eta^2_snb  \frac{N-n}{N-1} h_D \sigma^2  +  C_3\eta_s^3 bn  \left[\left ( \frac{N-n}{N-1} \right )^2 (b-1)^2 + \left ( \frac{n-1}{N-1} \right )^2 \right ]  + C_4 \eta_s^4 b^4 n^4  \label{eq:lastbutone_strongcvx}
\end{align}
where we let
\[
C_1 = \frac{1}{3} L  G + L^2 + \frac{7}{6} L_H G, \quad C_2 = 2, \quad C_3 = 2L^2G^2 \mu^{-1}, \quad C_4 = \frac{2}{3} L^2 G^4 + \frac{1}{3} L^3G + \frac{1}{6} L_H G^3
\]

By the definition of $\mu$-strongly convex function $F(\cdot)$, we have
\begin{align} \label{eq:proofstrongcvx}
F(\x^*) - F(\x_0^s) \geq \left \langle \x^*-  \x_0^s, \nabla F(\x_0^s) \right \rangle + \frac{\mu}{2} \| \x_0^s - \x^* \|^2.
\end{align}

By the definition of $L$-smooth convex function $F(\cdot)$, we have
\begin{align} \label{eq:proofsmooth}
\|\nabla F(\x^s_0)\|^2 \leq 2L(F(\x_0^s) - F(\x^*)).
\end{align}

Plugging \eqref{eq:proofstrongcvx} and \eqref{eq:proofsmooth} into \eqref{eq:lastbutone_strongcvx}, we have
\begin{align*}
& (2 \eta_s b n- 4L \eta_s^2   b^2 n^2 ) (   F(\x_0^s)-F(\x^*)  ) \leq (1-\frac{3}{4} \eta_s bn  \mu  + C_1 \eta_s^2 b^2 n^2 )\|x_0^s-x^* \|^2  -  \E\|x_0^{s+1} - x^*\|^2 \\
&\qquad \qquad \qquad + C_2 \eta^2_snb  \frac{N-n}{N-1} h_D \sigma^2  +  C_3\eta_s^3 bn  \left[\left ( \frac{N-n}{N-1} \right )^2 (b-1)^2 + \left ( \frac{n-1}{N-1} \right )^2 \right ]  + C_4 \eta_s^4 b^4 n^4
\end{align*}

Now assume that $\eta_s \leq \min\{ \frac{1}{4C_1\mu^{-1} b n}, \frac{1}{4Lbn} \}$, we eventually obtain 
\begin{align*}
&\eta_s b  n  (   F(\x_0^s)-F(\x^*)  ) \leq (1-\frac{1}{2} \eta_s b n \mu )\|x_0^s-x^* \|^2  -  \E\|x_0^{s+1} - x^*\|^2  \\
&\qquad \qquad \qquad \qquad + C_2 \eta^2_snb  \frac{N-n}{N-1} h_D \sigma^2  +  C_3\eta_s^3 bn  \left[\left ( \frac{N-n}{N-1} \right )^2 (b-1)^2 + \left ( \frac{n-1}{N-1} \right )^2 \right ]  + C_4 \eta_s^4 b^4 n^4
\end{align*}

We then set 
the learning rate $\eta_s$
to balance different terms on the right hand side to achieve the fastest 
convergence in $O(-)$ sense.
By Lemma \ref{lemma:converge}, letting $S \geq  1$, $\eta_s = \frac{6}{b n \mu(s+a)} $, and $T = b S n$, we can have
\begin{align*}
F\left (\frac{\sum_{s=1}^S w_s \x_0^s}{\sum_{s=1}^S w_s} \right) - F(\x^*) \leq \frac{\sum_{s=1}^S w_s  (F(\x_0^s)-F(\x^*) )}{\sum_{s=1}^S w_s} \lesssim  ( 1-\alpha) \frac{h_D\sigma^2}{T}  + \beta \frac{1}{T^2}  + \gamma \frac{  m^3}{T^3}
\end{align*}
with 
\begin{align*}
\alpha := \frac{n-1}{N-1}, \beta := \alpha^2 + ( 1-\alpha)^2(b-1)^2, \gamma := \frac{n^3}{N^3}.
\end{align*}
and requiring
\begin{align*}
&T \geq bn \\
&a \geq \max \left\{\frac{8LG + 24L^2 + 28L_H G}{\mu^2}, \frac{24L}{\mu}, 1 \right\}.
\end{align*}

\subsection{Proof of Theorem \ref{thm:partialrs_sample_smooth}}
Recall the PL condition for a certain constant $\mu$ is in the following form,
\begin{align*}
2 \mu (f(\x)-f(\x^*)) \leq  \| \nabla f(\x) \|^2
\end{align*}

We begin our proof as follows, by the $L$-smoothness of the objective function $f(\x)$, 
\begin{align}
\E [f(\x_0^{s+1}) - f(\x_0^s)] \leq & \left \langle  \E \x_0^{s+1} - \x_0^s, \nabla f(\x_0^s)  \right \rangle + \frac{L}{2} \E \| \x_0^{s+1} - \x_0^s \|^2 
\nonumber\\
\leq & -\eta_s \E \left \langle  \sum_{k=1}^{bn}  \nabla f_{\psi_s(k)}(\x^s_{k-1}), \nabla f(\x_0^s)  \right \rangle + \frac{L}{2} \eta_s^2 \E \left \| \sum_{k=1}^{bn}  \nabla f_{\psi_s(k)}(\x^s_{k-1}) \right\|^2 
\nonumber\\
\leq & \underbrace{-\eta_s \E \left \langle  \sum_{k=1}^{bn}  [ \nabla f_{\psi_s(k)}(\x^s_{k-1}) - \nabla f_{\psi_s(k)}(\x^s_0) ], \nabla f(\x_0^s)  \right \rangle}_{\calG_1} \nonumber\\
&\underbrace{-\eta_s \E \left \langle  \sum_{k=1}^{bn} \nabla f_{\psi_s(k)}(\x^s_0), \nabla f(\x_0^s)  \right \rangle}_{\calG_2}   + \underbrace{ L \eta_s^2 \E \left \| \sum_{k=1}^{bn}  [ \nabla f_{\psi_s(k)}(\x^s_{k-1}) - \nabla f_{\psi_s(k)}(\x^s_0) ] \right\|^2}_{\calG_3} \label{eq:main_PL} \\
& + \underbrace{L\eta_s^2 \E \left \| \sum_{k=1}^{bn}  \nabla f_{\psi_s(k)}(\x^s_0) - \E \sum_{k=1}^{bn}  \nabla f_{\psi_s(k)}(\x^s_0) \right\|^2}_{\calG_4} + \underbrace{ L \eta_s^2 \E \left \| \E \sum_{k=1}^{bn}  \nabla f_{\psi_s(k)}(\x^s_0) \right\|^2}_{\calG_5} \nonumber 
\end{align}

\textbf{Bounds of $\calG_2$, $\calG_3$, $\calG_4$ and $\calG_5$}

As shown in the proof of Theorem \ref{thm:partialrs_sample}, $\calI_2$, $\calI_3$, $\calI_4$ and $\calI_5$ are the similar terms to $\calG_2$, $\calG_3$, $\calG_4$ and $\calG_5$. We can shown their upper bounds as follows,
\begin{align}
\calG_2 =  - \eta_s \E \left \langle \sum_{k=1}^{bn} \nabla f_{\psi_s(k)}(\x^s_0), \nabla f(\x_0^s) \right \rangle = - \eta_s \frac{n}{N} m  \left\|\nabla F(\x_0^s) \right \|^2 \label{eq:G2}
\end{align}

\begin{align}
\calG_3 = L\eta_s^2 \E \left \| \sum_{k=1}^{bn}  [ \nabla f_{\psi_s(k)}(\x^s_{k-1})- \nabla f_{\psi_s(k)}(\x^s_0)] \right \|^2 \leq  \frac{1}{3}\eta_s^4 L^3 G^2 (bn)^4 \label{eq:G3}
\end{align}

\begin{align}
\calG_4 = L \eta^2_s \E \left \| \sum_{k=1}^{bn} \nabla f_{\psi_s(k)}(\x^s_0) - \E \sum_{k=1}^{bn} \nabla f_{\psi_s(k)}(\x^s_0) \right \|^2 \leq L \eta_s^2
\frac{nb (N-n)}{N-1} h_D\sigma^2 \label{eq:G4}
\end{align}

\begin{align}
\calG_5 = L\eta_s^2 \left \|\E \sum_{k=1}^{bn} \nabla f_{\psi_s(k)}(\x^s_0) \right \|^2 = \eta_s^2 \frac{n^2}{N^2} m^2 L \left \| \nabla F(\x^s_0) \right \|^2 \label{eq:G5}
\end{align}

\textbf{Bound of $\calG_1$}

Next, we will show the upper bound of $\calG_1$.

\begin{align*}
\calG_1 =&  -  \eta_s \E \left \langle \nabla f(\x_0^s), \sum_{k=1}^{bn} [ \nabla f_{\psi_s(k)}(\x^s_{k-1})  - \nabla f_{\psi_s(k)}(\x^s_0) ] \right \rangle 
\\
=&  -  \eta_s \E \left \langle \nabla f(\x_0^s), \sum_{k=1}^{bn}  \int_{\x^s_0}^{\x^s_{k-1}} H_{\psi_s(k)}(\x) \mathrm{d}\x \right \rangle 
\\
= &  -  \eta_s \E \left \langle \nabla f(\x_0^s), \sum_{k=1}^{bn}  \int_{\x^s_0}^{\x^s_{k-1}}  H_{\psi_s(k)}(\x_0^s)  \mathrm{d}\x \right \rangle \\
&  -  \eta_s \E \left \langle \nabla f(\x_0^s), \sum_{k=1}^{bn}  \int_{\x^s_0}^{\x^s_{k-1}} \left( H_{\psi_s(k)}(\x) - H_{\psi_s(k)}(\x_0^s) \right) \mathrm{d}\x \right \rangle
\\
= &  -  \eta_s \E \left \langle \nabla f(\x_0^s), \sum_{k=1}^{bn}   H_{\psi_s(k)}(\x_0^s)  \left(\x^s_{k-1} - \x^s_0 \right) \right \rangle \\
&  -  \eta_s \E \left \langle \nabla f(\x_0^s), \sum_{k=1}^{bn}  \int_{\x^s_0}^{\x^s_{k-1}} \left( H_{\psi_s(k)}(\x) - H_{\psi_s(k)}(\x_0^s) \right) \mathrm{d}\x \right \rangle 
\\
= &   \eta_s^2 \E \left \langle \nabla f(\x_0^s), \sum_{k=1}^{bn}   H_{\psi_s(k)}(\x_0^s)  \left(\sum_{k'=1}^{k-1} \nabla f_{\psi_s(k')}(\x_{k'-1}^s)  \right) \right \rangle \\
&  -  \eta_s \E \left \langle \nabla f(\x_0^s), \sum_{k=1}^{bn}  \int_{\x^s_0}^{\x^s_{k-1}} \left( H_{\psi_s(k)}(\x) - H_{\psi_s(k)}(\x_0^s) \right) \mathrm{d}\x \right \rangle 
\\
= & \underbrace{   \eta_s^2 \E \left \langle \nabla f(\x_0^s), \sum_{k=1}^{bn}   H_{\psi_s(k)}(\x_0^s)  \left(\sum_{k'=1}^{k-1} \nabla f_{\psi_s(k')}(\x_0^s)  \right) \right \rangle }_{\calG_{11}}\\
&  \underbrace{ \eta_s^2 \E \left \langle \nabla f(\x_0^s), \sum_{k=1}^{bn}   H_{\psi_s(k)}(\x_0^s)  \left(\sum_{k'=1}^{k-1} \nabla f_{\psi_s(k')} (\x_{k'-1}^s) - \nabla f_{\psi_s(k')}(\x_0^s)  \right) \right \rangle }_{\calG_{12}} \\
&  \underbrace{-  \eta_s \E \left \langle \nabla f(\x_0^s), \sum_{k=1}^{bn}  \int_{\x^s_0}^{\x^s_{k-1}} \left( H_{\psi_s(k)}(\x) - H_{\psi_s(k)}(\x_0^s) \right) \mathrm{d}\x \right \rangle }_{\calG_{13}}
\end{align*}

\textbf{Bound of $\calG_{11}$}

As shown in the proof of $\calJ_{3}$, we can similarly have
\begin{align}
\calG_{11}  =& \eta_s^2  \left \langle \nabla f(\x_0^s), \sum_{l=1}^N\sum_{\substack{i \neq j\\i,j \in B_l } } \frac{n}{2N} H_i(\x_0^s) \nabla f_j(\x_0^s)+  \sum_{l\neq l'} \sum_{ \substack{i,j\\ i \in B_l, j \in B_{l'}}} \frac{n(n-1)}{2N(N-1)} H_i(\x_0^s) \nabla f_j(\x_0^s) \right \rangle
\nonumber\\
= & \eta_s^2 \frac{n}{2N} \left \langle  \nabla f(\x_0^s), \sum_{l=1}^N\sum_{\substack{i \neq j\\i,j \in B_l } } H_i(\x_0^s) \nabla f_j(\x_0^s)+  \sum_{l\neq l'} \sum_{ \substack{i,j\\ i \in B_l, j \in B_{l'}}} \frac{n-1}{N-1} H_i(\x_0^s) \nabla f_j(\x_0^s) \right \rangle
\nonumber\\
= & \eta_s^2 \frac{n}{2N} \left \langle  \nabla f(\x_0^s), \sum_{l=1}^N\sum_{\substack{i \neq j\\i,j \in B_l } } H_i(\x_0^s) \nabla f_j(\x_0^s)+  \sum_{l\neq l'} \sum_{ \substack{i,j\\ i \in B_l, j \in B_{l'}}} \frac{n-1}{N-1} H_i(\x_0^s) \nabla f_j(\x_0^s) \right \rangle \nonumber\\
&- \eta_s^2 \frac{n}{2N} \left \langle  \nabla f(\x_0^s),  \sum_{l=1}^N \sum_{l' = 1}^N \sum_{ \substack{i,j\\ i \in B_l, j \in B_{l'}}} \frac{n-1}{N-1} H_i(\x_0^s) \nabla f_j(\x_0^s) \right \rangle \nonumber\\
&+ \eta_s^2 \frac{n}{2N} \left \langle  \nabla f(\x_0^s),  \sum_{l=1}^N \sum_{l' = 1}^N \sum_{ \substack{i,j\\ i \in B_l, j \in B_{l'}}} \frac{n-1}{N-1} H_i(\x_0^s) \nabla f_j(\x_0^s) \right \rangle
\nonumber\\
= & \eta_s^2 \frac{n}{2N} \left \langle  \nabla f(\x_0^s), \sum_{l=1}^N\sum_{\substack{i \neq j\\i,j \in B_l } } H_i(\x_0^s) \nabla f_j(\x_0^s)+  \sum_{l=1}^N \sum_{ i,j \in B_l} \frac{n-1}{N-1} H_i(\x_0^s) \nabla f_j(\x_0^s) \right \rangle \nonumber\\
&+ \eta_s^2 \frac{n}{2N} \left \langle  \nabla f(\x_0^s),  \frac{n-1}{N-1} N^2b^2 H(\x_0^s)\nabla f(\x_0^s)  \right \rangle
\nonumber\\
\leq & \eta_s^2 \frac{n}{2N} \left \| \nabla f(\x_0^s)\right\| \left \| \sum_{l=1}^N\sum_{\substack{i \neq j\\i,j \in B_l } } H_i(\x_0^s) \nabla f_j(\x_0^s)+ \sum_{l=1}^N \sum_{ i,j \in B_l} \frac{n-1}{N-1} H_i(\x_0^s) \nabla f_j(\x_0^s)  \right \| \nonumber\\
&+ \eta_s^2 \frac{n(n-1)Nb^2}{2(N-1)}L \| \nabla f(\x_0^s) \|^2
\nonumber\\
\leq & \eta_s^2  \| \nabla f(\x_0^s)\|  \frac{N-n}{N-1}nb(b-1)LG + \eta_s^2  \|  \nabla f(\x_0^s)\|  \frac{n-1}{N-1}nbLG + \eta_s^2 \frac{n(n-1)Nb^2}{2(N-1)}L \| \nabla f(\x_0^s) \|^2
\nonumber\\
\leq & \frac{1}{4} \eta_s  nb \|\nabla f(\x_0^s)\|^2 +  \eta_s^3 nb \left ( \frac{N-n}{N-1} \right )^2 (b-1)^2L^2G^2 
\nonumber \\
&+ \frac{1}{4} \eta_s  nb \|\nabla f(\x_0^s)\|^2  + \eta_s^3 nb \left ( \frac{n-1}{N-1} \right )^2 L^2G^2 + \eta_s^2 \frac{n(n-1)Nb^2}{2(N-1)}L \| \nabla f(\x_0^s) \|^2
\nonumber\\
= & \frac{1}{2} \eta_s  nb \| \nabla f(\x_0^s)\|^2 +  \eta_s^3 nb L^2G^2  \left[\left ( \frac{N-n}{N-1} \right )^2 (b-1)^2 + \left ( \frac{n-1}{N-1} \right )^2 \right ] \label{eq:G11}\\
&+ \eta_s^2 \frac{n(n-1)Nb^2}{2(N-1)}L \| \nabla f(\x_0^s) \|^2 \nonumber
\end{align}
where the fourth equality is due to 
\begin{align*}
&\sum_{l=1}^N \sum_{l' = 1}^N \sum_{ \substack{i,j\\ i \in B_l, j \in B_{l'}}} H_i(\x_0^s) \nabla f_j(\x_0^s)  = N^2b^2 H(\x_0^s)\nabla f(\x_0^s),\\
&\sum_{l\neq l'} \sum_{ \substack{i,j\\ i \in B_l, j \in B_{l'}}}H_i(\x_0^s) \nabla f_j(\x_0^s) - \sum_{l=1}^N \sum_{l' = 1}^N \sum_{ \substack{i,j\\ i \in B_l, j \in B_{l'}}} H_i(\x_0^s) \nabla f_j(\x_0^s) = \sum_{l=1}^N  \sum_{i,j\in B_l} H_i(\x_0^s) \nabla f_j(\x_0^s) 
\end{align*}
with $H(\x_0^s)$ being the Hessian of $f(\x)$ at the point $\x_0^s$.

\textbf{Bound of $\calG_{12}$}

\begin{align}
\calG_{12} =&  \eta_s^2 \E \left \langle \nabla f(\x_0^s), \sum_{k=1}^{bn}   H_{\psi_s(k)}(\x_0^s)  \left(\sum_{k'=1}^{k-1} \nabla f_{\psi_s(k')} (\x_{k'-1}^s) - \nabla f_{\psi_s(k')}(\x_0^s)  \right) \right \rangle 
\nonumber\\
\leq &  \eta_s^2 \E \left \| \nabla f(\x_0^s) \right \| \left\| \sum_{k=1}^{bn}   H_{\psi_s(k)}(\x_0^s)  \left(\sum_{k'=1}^{k-1} \nabla f_{\psi_s(k')} (\x_{k'-1}^s) - \nabla f_{\psi_s(k')}(\x_0^s)  \right) \right \|
\nonumber \\
\leq & \eta_s^2 \E \left \| \nabla f(\x_0^s) \right \| \sum_{k=1}^{bn} \sum_{k'=1}^{k-1} \left\|   H_{\psi_s(k)}(\x_0^s)  \left( \nabla f_{\psi_s(k')} (\x_{k'-1}^s) - \nabla f_{\psi_s(k')}(\x_0^s)  \right) \right \|
\nonumber \\
\leq & \eta_s^2 \E \left \| \nabla f(\x_0^s) \right \| \sum_{k=1}^{bn} \sum_{k'=1}^{k-1} L \left\|    \nabla f_{\psi_s(k')} (\x_{k'-1}^s) - \nabla f_{\psi_s(k')}(\x_0^s)   \right \|
\nonumber \\
\leq & \eta_s^2 \E \left \| \nabla f(\x_0^s) \right \| \sum_{k=1}^{bn} \sum_{k'=1}^{k-1} L^2 \left\|    \x_{k'-1}^s - \x_0^s  \right \|
\nonumber \\
= & \eta_s^2 \E \left \| \nabla f(\x_0^s) \right \| \sum_{k=1}^{bn} \sum_{k'=1}^{k-1} L^2 \left\|   \eta_s \sum_{k'=1}^{k-1}  \nabla f_{\psi_s(k')}(\x^s_{k'-1})  \right \|
\nonumber \\
\leq &  \eta_s^3 L^2  G \left \| \nabla f(\x_0^s) \right \| \sum_{k=1}^{bn} (k-1)^2 
\nonumber \\
\leq & \frac{1}{3} (bn)^3 \eta_s^3 L^2  G \left \| \nabla f(\x_0^s) \right \|  
\nonumber \\
\leq & \frac{1}{6}  \eta_s^2 (bn)^2 L  \left \| \nabla f(\x_0^s) \right \|^2 + \frac{1}{6} \eta_s^4 (bn)^4 L^3  G^2    \label{eq:G12}
\end{align}

\textbf{Bound of $\calG_{13}$}

\begin{align}
\calG_{13} =& -\eta_s \E \left \langle\nabla f(\x_0^s), \sum_{k=1}^{bn}  \int_{\x^s_0}^{\x^s_{k-1}} \left( H_{\psi_s(k)}(\x) - H_{\psi_s(k)}(\x_0^s) \right) \mathrm{d}\x \right \rangle
\nonumber\\
=& -  \eta_s \E \left \langle\nabla f(\x_0^s), \sum_{k=1}^{bn}  \int_{0}^{\|\x^s_{k-1}-\x^s_0\|} \left( H_{\psi_s(k)}\left(\x_0^s + \frac{\x^s_{k-1}-\x^s_0}{\|\x^s_{k-1}-\x^s_0\|} t \right ) - H_{\psi_s(k)}(\x_0^s) \right) \frac{\x^s_{k-1}-\x^s_0}{\|\x^s_{k-1}-\x^s_0\|}  \mathrm{d} t \right \rangle
\nonumber\\
=& -  \eta_s \E \sum_{k=1}^{bn}  \int_{0}^{\|\x^s_{k-1}-\x^s_0\|} \left \langle\nabla f(\x_0^s), \left( H_{\psi_s(k)}\left(\x_0^s + \frac{\x^s_{k-1}-\x^s_0}{\|\x^s_{k-1}-\x^s_0\|} t \right ) - H_{\psi_s(k)}(\x_0^s) \right) \frac{\x^s_{k-1}-\x^s_0}{\|\x^s_{k-1}-\x^s_0\|} \right \rangle  \mathrm{d} t 
\nonumber\\
\leq &  \eta_s \E \sum_{k=1}^{bn}  \int_{0}^{\|\x^s_{k-1}-\x^s_0\|} \left \|\nabla f(\x_0^s) \right \|\cdot  \left \| H_{\psi_s(k)}\left(\x_0^s + \frac{\x^s_{k-1}-\x^s_0}{\|\x^s_{k-1}-\x^s_0\|} t \right ) - H_{\psi_s(k)}(\x_0^s) \right\| \frac{\| \x^s_{k-1}-\x^s_0 \|}{\|\x^s_{k-1}-\x^s_0\|}  \mathrm{d} t 
\nonumber\\
\leq &  \eta_s \left \|\nabla f(\x_0^s) \right \|\cdot \E \sum_{k=1}^{bn}  \int_{0}^{\|\x^s_{k-1}-\x^s_0\|} L_H \left \| \x_0^s + \frac{\x^s_{k-1}-\x^s_0}{\|\x^s_{k-1}-\x^s_0\|} t  - \x_0^s \right\|  \mathrm{d} t 
\nonumber\\
\leq &  \eta_s \left \|\nabla f(\x_0^s) \right \|\cdot \E \sum_{k=1}^{bn}  \int_{0}^{\|\x^s_{k-1}-\x^s_0\|} L_H   t  \mathrm{d} t 
\nonumber\\
= &  \eta_s \left \|\nabla f(\x_0^s) \right \| L_H  \cdot \E \sum_{k=1}^{bn}   \frac{1}{2} \|\x^s_{k-1}-\x^s_0\|^2 
\nonumber\\
= &  \eta_s \left \|\nabla f(\x_0^s) \right \| L_H  \cdot \E \sum_{k=1}^{bn}  \frac{1}{2} \eta_s \| \sum_{k'=1}^{k-1} \nabla f_{\psi_s(k')}(\x^s_{k'-1}) \|^2   
\nonumber\\
\leq &  \eta_s \left \|\nabla f(\x_0^s) \right \| L_H  \cdot \sum_{k=1}^{bn}   \frac{1}{2} \eta_s^2 (k-1)^2 G^2 
\nonumber\\
\leq &  \eta_s \left \|\nabla f(\x_0^s) \right \| L_H  \frac{(bn)^3}{6} \eta_s^2  G^2  
\nonumber\\
\leq &  \frac{(bn)^2}{12} \eta_s^2  L_H \left \|\nabla f(\x_0^s) \right \|^2 +  \frac{(bn)^4}{12} \eta_s^4  L_H G^4 \label{eq:G13}
\end{align}

Based on (\ref{eq:G11}), (\ref{eq:G12}), (\ref{eq:G13}), we can have
\begin{align}
\calG_1 =&  \calG_{11} + \calG_{12} + \calG_{13} 
\nonumber\\
\leq &  \frac{1}{2} \eta_s nb \| \nabla f(\x_0^s)\|^2 +  \eta_s^3 nb L^2G^2 \left[\left ( \frac{N-n}{N-1} \right )^2 (b-1)^2 + \left ( \frac{n-1}{N-1} \right )^2 \right ] 
\nonumber\\
&+ \eta_s^2 \frac{n(n-1)Nb^2}{2(N-1)}L \| \nabla f(\x_0^s) \|^2 
+ \frac{1}{6}  \eta_s^2 b^2 n^2 L  \left \| \nabla f(\x_0^s) \right \|^2 + \frac{1}{6} \eta_s^4 b^4 n^4 L^3  G^2  \label{eq:G1} \\
& + \frac{(bn)^2}{12} \eta_s^2  L_H \left \|\nabla f(\x_0^s) \right \|^2 +  \frac{(bn)^4}{12} \eta_s^4  L_H G^4 \nonumber 
\end{align}

Plugging (\ref{eq:G1}) (\ref{eq:G2}) (\ref{eq:G3}) (\ref{eq:G4}) into (\ref{eq:main_PL}), we have

\begin{align} \label{eq:nonconvex_general}
\E f(\x_0^{s+1}) - f(\x_0^s) \leq & - \frac{1}{2} \eta_s bn  \left\|\nabla F(\x_0^s) \right \|^2 + C_1 \eta_s^2 b^2n^2   \left\|\nabla F(\x_0^s) \right \|^2  + C_2 \eta_s^2 bn \frac{N-n}{N-1} h_D  \sigma^2
\nonumber\\
&+  C_3  \eta_s^3 bn \left[\left ( \frac{N-n}{N-1} \right )^2 (b-1)^2 + \left ( \frac{n-1}{N-1} \right )^2 \right ]   + C_4 b^4n^4 \eta_s^4
\end{align}
where we let
\[
C_1 =  \frac{13}{6} L +   \frac{1}{12} L_H , \quad C_2 = L , \quad C_3 = L^2 G^2, \quad  C_4 = \frac{1}{2} L^3 G^2 + \frac{1}{12} L_H G^4
\]

By the definition of PL condition, we have have
\begin{align*}
2\mu (F(\x_0^s) - F(\x^*)) \leq \|\nabla F(\x_0^s) \|^2
\end{align*}

Plugging this into the above formulation, we have
\begin{align*}
&\left( \frac{1}{4} \eta_s bn -  C_1 \eta_s^2 b^2 n^2\right) \left\|\nabla F(\x_0^s) \right \|^2   \\
\leq &  F(\x_0^s)-\E F(\x_0^{s+1})   - \frac{1}{4}  \eta_s b n  \left\|\nabla F(\x_0^s) \right \|^2  + C_2 \eta_s^2 bn \frac{N-n}{N-1} h_D  \sigma^2
\\
&+  C_3  \eta_s^3 bn \left[\left ( \frac{N-n}{N-1} \right )^2 (b-1)^2 + \left ( \frac{n-1}{N-1} \right )^2 \right ]   + C_4 b^4n^4 \eta_s^4
\\
\leq &  F(\x_0^s) -\E F(\x_0^{s+1})   - \frac{1}{2} \eta_s b n \mu (F(\x_0^s) - F(\x^*) )+ C_1 \eta_s^2 b^2 n^2   \left\|\nabla F(\x_0^s) \right \|^2 +  C_2 \eta_s^2 bn \frac{N-n}{N-1} h_D  \sigma^2
\\
&+  C_3  \eta_s^3 bn \left[\left ( \frac{N-n}{N-1} \right )^2 (b-1)^2 + \left ( \frac{n-1}{N-1} \right )^2 \right ]   + C_4 b^4n^4 \eta_s^4
\\
= &  (1 - \frac{1}{2} \eta_s b n \mu )(F(\x_0^s) - F(\x^*) )  -( \E  F(\x_0^{s+1}) - F(\x^*))  +  C_2 \eta_s^3 b n    + C_3  \eta_s^4 b^4 n^4 
\end{align*}

Now assume that $ \eta_s \leq  \frac{1}{8C_1 b n}$, we eventually obtain 
\begin{align*}
\frac{1}{8} \eta_s b n  \|  \nabla F(\x_0^s)\|^2 \leq& (1-\frac{1}{2} \eta_s b n \mu )(F(\x_0^s) - F(\x^*) )  -  (\E F(\x_0^{s+1}) - F(\x^*) ) + C_2 \eta_s^2 bn \frac{N-n}{N-1} h_D  \sigma^2
\nonumber\\
&+  C_3  \eta_s^3 bn \left[\left ( \frac{N-n}{N-1} \right )^2 (b-1)^2 + \left ( \frac{n-1}{N-1} \right )^2 \right ]   + C_4 b^4n^4 \eta_s^4
\end{align*}

By  Lemma \ref{lemma:converge}, letting $S \geq  1$,  $\eta_s = \frac{6}{b n \mu(s+a)} $, and $T = b S n$, we can have
\begin{align*}
\sum_{s=1}^S  \frac{w_s  \|\nabla F(\x_0^s)\|^2 }{\sum_{s=1}^S w_s}   \lesssim ( 1-\alpha) \frac{h_D\sigma^2}{T}  + \beta \frac{1}{T^2}  + \gamma \frac{  m^3}{T^3}
\end{align*}
with 
\begin{align*}
\alpha := \frac{n-1}{N-1}, \beta := \alpha^2 + ( 1-\alpha)^2(b-1)^2, \gamma := \frac{n^3}{N^3}.
\end{align*}
and requiring
\begin{align*}
&T \geq bn \\
&a \geq \max \left\{\frac{108L + 4L_H}{\mu}, 1 \right\}.
\end{align*}  

The above result further lead to 
\begin{align*}
F\left (\frac{\sum_{s=1}^S w_s \x_0^s}{\sum_{s=1}^S w_s} \right) - F(\x^*) \lesssim ( 1-\alpha) \frac{h_D\sigma^2}{T}  + \beta \frac{1}{T^2}  + \gamma \frac{  m^3}{T^3}
\end{align*}
by PL condition.

For the proof of non-convex objectives without PL condition, we can directly use the formulation \eqref{eq:nonconvex_general}, 
\begin{align}
\E f(\x_0^{s+1}) - f(\x_0^s) \leq & - \frac{1}{2} \eta_s bn  \left\|\nabla F(\x_0^s) \right \|^2 + C_1 \eta_s^2 b^2n^2   \left\|\nabla F(\x_0^s) \right \|^2  + C_2 \eta_s^2 bn \frac{N-n}{N-1} h_D  \sigma^2
\nonumber\\
&+  C_3  \eta_s^3 bn \left[\left ( \frac{N-n}{N-1} \right )^2 (b-1)^2 + \left ( \frac{n-1}{N-1} \right )^2 \right ]   + C_4 b^4n^4 \eta_s^4
\end{align}
where we let
\[
C_1 =  \frac{13}{6} L +   \frac{1}{12} L_H , \quad C_2 = L , \quad C_3 = L^2 G^2, \quad  C_4 = \frac{1}{2} L^3 G^2 + \frac{1}{12} L_H G^4
\]

Now assuming $\eta_s \leq \frac{1}{4 C_1 b n}$
\begin{align*}
 \frac{1}{4} \eta_s bn  \left\|\nabla F(\x_0^s) \right \|^2  \leq &    f(\x_0^s) - \E f(\x_0^{s+1})  + C_2 \eta_s^2 bn \frac{N-n}{N-1} h_D  \sigma^2
\\
&+  C_3  \eta_s^3 bn \left[\left ( \frac{N-n}{N-1} \right )^2 (b-1)^2 + \left ( \frac{n-1}{N-1} \right )^2 \right ]   + C_4 b^4n^4 \eta_s^4
\end{align*}
Taking summation from $s = 1$ to $S$ and dividing both side by $S$, and then we set the step size as follows to obtain different  convergence rate.
Again, 
we set 
the learning rate $\eta_s$
to balance different terms on the right hand side to achieve the fastest 
convergence in $O(-)$ sense.
\begin{enumerate}
\item When $\alpha \leq \frac{N-2}{N-1}$,  choosing $\eta_s = \frac{1}{\sqrt{bn(1-\alpha)h_D\sigma^2 S} }$ and assuming  $S \geq \frac{16bn(\frac{13}{6}L + \frac{1}{12}L_H)^2}{\sigma^2 (1-\alpha) h_D}$, we have,
\begin{align*}
\frac{1}{S} \sum_{s=1}^S \E\|\nabla F(\x_0^s)\|^2 \lesssim &  ( 1-\alpha)^{1/2} \frac{\sqrt{h_D \sigma^2} }{\sqrt{T}}  + \beta \frac{1}{T}  +\gamma \frac{  m^3}{T^{\frac{3}{2}}},
\end{align*}
where the factors are defined as follows
\begin{align*}
\alpha := \frac{n-1}{N-1}, \beta := \frac{\alpha^2}{1-\alpha} \frac{1}{h_D \sigma^2} + ( 1-\alpha)\frac{(b-1)^2}{h_D \sigma^2}, \gamma := \frac{n^3}{(1-\alpha)N^3}.
\end{align*}

\item When $\alpha =1$, choosing $\eta_s = \frac{1}{(m S)^{\frac{1}{3}} }$ and assuming $S \geq 64 (\frac{13}{6}L + \frac{1}{12}L_H)^3 b^2 n^3 / N$, we have,
\begin{align*}
 \frac{1}{S} \sum_{s=1}^S \E\|\nabla F(\x_0^s)\|^2\lesssim  \frac{1}{T^{\frac{2}{3}}}  +\gamma' \frac{  m^3}{T},
\end{align*}
where we define
\[
\gamma' := \frac{n^3}{N^3}.
\]
\end{enumerate}